\pdfoutput=1
\documentclass{article}

\usepackage[preprint]{neurips_2026}

\usepackage{amsmath,amsfonts,amssymb,amsthm}
\usepackage{graphicx}
\usepackage{url}
\usepackage{booktabs}
\usepackage{enumitem}
\usepackage{algorithm}
\usepackage{algorithmic}
\usepackage{xcolor}
\usepackage{natbib}

\usepackage{hyperref}

\newtheorem{theorem}{Theorem}
\newtheorem{corollary}[theorem]{Corollary}
\newtheorem{definition}[theorem]{Definition}

\newtheorem{proposition}[theorem]{Proposition}
\theoremstyle{remark}
\newtheorem{remark}[theorem]{Remark}
\theoremstyle{plain}

\newcommand{\X}{\mathcal{X}}
\newcommand{\Y}{\mathcal{Y}}
\newcommand{\D}{\mathcal{D}}
\newcommand{\Hyp}{\mathcal{H}}

\newcommand{\statPBK}{55}
\newcommand{\statPBSeeds}{5}
\newcommand{\statPBRounds}{4}
\newcommand{\statPBTrajectories}{1{,}650}

\newcommand{\statPBModelF}{53.46}
\newcommand{\statPBModelDf}{(1, 54)}
\newcommand{\statPBModelP}{{<}0.0001}
\newcommand{\statPBModelEta}{0.497}
\newcommand{\statPBModelEtaRunningBest}{0.514}
\newcommand{\statPBModelCohenF}{0.99}
\newcommand{\statPBModelCohenFCI}{$[0.738, \, 1.319]$}
\newcommand{\statPBModeF}{0.50}
\newcommand{\statPBModeDf}{(2, 108)}
\newcommand{\statPBModeP}{0.61}
\newcommand{\statPBModeEta}{0.009}
\newcommand{\statPBModeCohenF}{0.10}
\newcommand{\statPBInterF}{0.81}
\newcommand{\statPBInterDf}{(2, 108)}
\newcommand{\statPBInterP}{0.45}
\newcommand{\statPBInterEta}{0.015}
\newcommand{\statPBInterCohenF}{0.12}

\newcommand{\statPBHaikuIndep}{24/55 (44\%) $[31,57]$}
\newcommand{\statPBHaikuBlind}{26/55 (47\%) $[35,60]$}
\newcommand{\statPBHaikuDiag}{27/55 (49\%) $[36,62]$}
\newcommand{\statPBSonnetIndep}{47/55 (85\%) $[74,92]$}
\newcommand{\statPBSonnetBlind}{47/55 (85\%) $[74,92]$}
\newcommand{\statPBSonnetDiag}{49/55 (89\%) $[78,95]$}
\newcommand{\statPBHaikuSolveRange}{44--49\%}
\newcommand{\statPBSonnetSolveRange}{85--89\%}

\newcommand{\statPBModePower}{11\%}
\newcommand{\statPBInterPower}{17\%}
\newcommand{\statPBModeReqK}{467}
\newcommand{\statPBInterReqK}{364}
\newcommand{\statPBPowerMultiple}{7--9$\times$}

\newcommand{\statPBSeedSD}{0.096}
\newcommand{\statPBSignalOverNoise}{1.7$\times$}
\newcommand{\statPBBetweenModeSignal}{0.057}

\newcommand{\statPBEdgeN}{275}
\newcommand{\statPBHaikuEdge}{(0, -0.39, -0.46, -0.45)}
\newcommand{\statPBSonnetEdge}{(0, -0.33, -0.41, -0.46)}
\newcommand{\statPBHaikuEta}{$0.097 \to 0.036 \to 0.033$}
\newcommand{\statPBSonnetEta}{$0.156 \to 0.068 \to 0.029$}

\newcommand{\statPBEdgeCIReps}{9{,}999}
\newcommand{\statPBHaikuEdgeCI}{$[{-}0.449, \, {-}0.322]$, $[{-}0.485, \, {-}0.431]$, $[{-}0.478, \, {-}0.416]$}
\newcommand{\statPBSonnetEdgeCI}{$[{-}0.373, \, {-}0.285]$, $[{-}0.449, \, {-}0.365]$, $[{-}0.482, \, {-}0.435]$}

\newcommand{\statPBTieRange}{80--95\%}
\newcommand{\statPBRegressMax}{2.9\%}
\newcommand{\statPBHaikuFlat}{80.7\%}
\newcommand{\statPBSonnetFlat}{69.1\%}
\newcommand{\statPBHaikuPadShare}{27.6\%}
\newcommand{\statPBSonnetPadShare}{53.9\%}
\newcommand{\statPBHaikuPadTieRange}{33--45\%}
\newcommand{\statPBSonnetPadTieRange}{74--86\%}
\newcommand{\statPBHaikuGenuineTieRange}{52--59\%}
\newcommand{\statPBSonnetGenuineTieRange}{13--21\%}
\newcommand{\statPBHaikuAbstainZ}{0.81}
\newcommand{\statPBSonnetAbstainZ}{0.67}
\newcommand{\statPBHaikuMoversEdge}{{+}0.233 to {+}0.500}
\newcommand{\statPBSonnetMoversEdge}{{+}0.346 to {+}0.393}
\newcommand{\statPBHaikuEdgeHeadroom}{({-}0.345, {-}0.435, {-}0.411)}
\newcommand{\statPBSonnetEdgeHeadroom}{({-}0.084, {-}0.143, {-}0.276)}
\newcommand{\statPBHaikuEdgeHeadroomCI}{$[{-}0.426, \, {-}0.247]$, $[{-}0.474, \, {-}0.383]$, $[{-}0.460, \, {-}0.348]$}
\newcommand{\statPBSonnetEdgeHeadroomCI}{$[{-}0.173, \, {+}0.021]$, $[{-}0.286, \, {+}0.028]$, $[{-}0.389, \, {-}0.125]$}
\newcommand{\statPBHaikuHeadroomN}{194$\to$158 of 275}
\newcommand{\statPBSonnetHeadroomN}{113$\to$49 of 275}

\newcommand{\statPBSonnetHeadroomNLast}{49}
\newcommand{\statPBHaikuTieSeries}{(0.887, 0.945, 0.949)}
\newcommand{\statPBSonnetTieSeries}{(0.800, 0.898, 0.953)}
\newcommand{\statPBHaikuRegressSeries}{(0.004, 0.015, 0.000)}
\newcommand{\statPBSonnetRegressSeries}{(0.029, 0.011, 0.007)}
\newcommand{\statPBHaikuRegressRows}{5}
\newcommand{\statPBSonnetRegressRows}{12}

\newcommand{\statPBOverlapK}{30}
\newcommand{\statPBOverlapHalfK}{15}
\newcommand{\statPBOverlapMStar}{30}
\newcommand{\statPBOverlapViolate}{23/55 (42\%)}
\newcommand{\statPBOverlapVoteErr}{0.418}
\newcommand{\statPBOverlapVoteErrCI}{$[0.291, \, 0.545]$}
\newcommand{\statPBOverlapEpsBar}{0.356}
\newcommand{\statPBOverlapEpsBarCI}{$[0.277, \, 0.436]$}
\newcommand{\statPBOverlapMarkov}{0.712}
\newcommand{\statPBOverlapGap}{{+}0.062}
\newcommand{\statPBOverlapGapCI}{$[{-}0.018, \, {+}0.142]$}
\newcommand{\statPBOverlapGapPosPct}{94\%}
\newcommand{\statPBOverlapGapLastExclRound}{2}
\newcommand{\statPBOverlapGapLastExcl}{{+}0.089}
\newcommand{\statPBOverlapGapLastExclCI}{$[{+}0.011, \, {+}0.167]$}
\newcommand{\statPBOverlapArms}{6}
\newcommand{\statPBOverlapArmMStar}{6}
\newcommand{\statPBOverlapArmViolate}{28/55}

\newcommand{\statPBOverlapTierK}{15}
\newcommand{\statPBOverlapSonnetMStar}{15}
\newcommand{\statPBOverlapSonnetViolate}{8/55}
\newcommand{\statPBOverlapSonnetVoteErr}{0.145}
\newcommand{\statPBOverlapSonnetEpsBar}{0.147}
\newcommand{\statPBOverlapSonnetGap}{{-}0.001}
\newcommand{\statPBOverlapSonnetGapCI}{$[{-}0.039, \, {+}0.038]$}
\newcommand{\statPBOverlapHaikuMStar}{15}
\newcommand{\statPBOverlapHaikuViolate}{30/55}
\newcommand{\statPBOverlapHaikuVoteErr}{0.545}
\newcommand{\statPBOverlapHaikuEpsBar}{0.565}
\newcommand{\statPBOverlapHaikuGap}{{-}0.019}
\newcommand{\statPBOverlapHaikuGapCI}{$[{-}0.057, \, {+}0.017]$}

\newcommand{\statPBTauResolved}{1.0}
\newcommand{\statPBHaikuResolvedRange}{29--38\%}
\newcommand{\statPBSonnetResolvedRange}{71--75\%}
\newcommand{\statPBHaikuResolvedIndep}{21/55 (38\%)}
\newcommand{\statPBHaikuResolvedDiag}{16/55 (29\%)}
\newcommand{\statPBSonnetResolvedDiag}{41/55 (75\%)}
\newcommand{\statPBOverlapMStarResolved}{30}
\newcommand{\statPBOverlapViolateResolved}{23/55}

\newcommand{\statPBHaikuIndepQ}{0.432}
\newcommand{\statPBHaikuBlindQ}{0.449}
\newcommand{\statPBHaikuDiagQ}{0.466}
\newcommand{\statPBSonnetIndepQ}{0.875}
\newcommand{\statPBSonnetBlindQ}{0.856}
\newcommand{\statPBSonnetDiagQ}{0.871}
\newcommand{\statPBContrastFRange}{0.56--2.16}
\newcommand{\statPBContrastPRange}{0.15--0.46}
\newcommand{\statPBInteractionPRange}{0.13--0.33}

\newcommand{\statPBAuditFailN}{23/50 (46\%)}
\newcommand{\statPBAuditRootCauses}{5}
\newcommand{\statPBTaskPoolRepos}{10}
\newcommand{\statPBTaskPoolRetained}{40}
\newcommand{\statPBTaskPoolBackfilled}{15}

\newcommand{\statPCSessions}{401}
\newcommand{\statPCCommits}{1{,}211}
\newcommand{\statPCRepos}{23}
\newcommand{\statPCDevelopers}{20}

\newcommand{\statPCOpusVersions}{five}
\newcommand{\statPCSonnetVersions}{three}
\newcommand{\statPCOpusTrailers}{eight}
\newcommand{\statPCSonnetTrailers}{four}
\newcommand{\statPCVersionSpan}{generations~4.5 through~5 in both families}

\newcommand{\statPCDecayPct}{70.0\%}
\newcommand{\statPCMedianRatio}{0.49}
\newcommand{\statPCPermCI}{$[0.400, \, 0.793]$}
\newcommand{\statPCPermP}{{<}0.0001}
\newcommand{\statPCPermN}{70}

\newcommand{\statPCTauOpus}{0.55}
\newcommand{\statPCTauSonnet}{0.55}
\newcommand{\statPCNOpus}{158}
\newcommand{\statPCNSonnet}{177}
\newcommand{\statPCGap}{{-}0.003}
\newcommand{\statPCGapCI}{$[{-}0.075, \, {+}0.069]$}
\newcommand{\statPCGapPermP}{0.93}
\newcommand{\statPCGapMWUP}{0.91}
\newcommand{\statPCEvaluePoint}{1.10}
\newcommand{\statPCEvalueCI}{1.00}
\newcommand{\statPCStratGap}{{+}0.002}
\newcommand{\statPCStratP}{0.96}
\newcommand{\statPCStrata}{4}
\newcommand{\statPCChurnRhoOpus}{0.80}
\newcommand{\statPCChurnRhoSonnet}{0.75}
\newcommand{\statPCChurnGap}{0.043}
\newcommand{\statPCChurnGapP}{0.54}
\newcommand{\statPCChurnGapN}{123}
\newcommand{\statPCMultilevelCoef}{{+}0.009}
\newcommand{\statPCMultilevelCI}{$[{-}0.070, \, {+}0.088]$}
\newcommand{\statPCMultilevelP}{0.82}
\newcommand{\statPCVarRepo}{0.132}
\newcommand{\statPCVarDev}{0.062}

\newcommand{\statPCRho}{0.77}
\newcommand{\statPCRhoCI}{$[0.65, \, 0.92]$}
\newcommand{\statPCSigma}{2.39}
\newcommand{\statPCSigmaCI}{$[2.24, \, 2.54]$}
\newcommand{\statPCPiPure}{0.071}
\newcommand{\statPCTrainN}{137}
\newcommand{\statPCTestN}{60}
\newcommand{\statPCKSHeldout}{0.18}
\newcommand{\statPCKSHeldoutP}{0.045}
\newcommand{\statPCKSHeldoutSeedRange}{$[0.0003, \, 0.44]$, rejecting in 6/8 seeds}
\newcommand{\statPCKSInsample}{0.21}
\newcommand{\statPCKSInsampleP}{{<}0.0001}

\newcommand{\statPCMedianLen}{1}
\newcommand{\statPCMeanLen}{3.0}
\newcommand{\statPCMaxLen}{45}
\newcommand{\statPCSingleCommitPct}{50.9\%}

\newcommand{\statPCCtrlSince}{2021-11-08}
\newcommand{\statPCCtrlUntil}{2022-09-19}
\newcommand{\statPCCtrlRepos}{122}
\newcommand{\statPCCtrlDevelopers}{62}
\newcommand{\statPCCtrlSessions}{9{,}395}
\newcommand{\statPCCtrlCommits}{17{,}694}
\newcommand{\statPCCtrlTau}{0.39}
\newcommand{\statPCCtrlRho}{0.86}
\newcommand{\statPCCtrlRhoCI}{$[0.82, \, 0.91]$}
\newcommand{\statPCCtrlSingleCommitPct}{63.9\%}
\newcommand{\statPCCtrlTauMulti}{0.435}
\newcommand{\statPCTauMulti}{0.593}
\newcommand{\statPCCtrlGapOpus}{{-}0.160}
\newcommand{\statPCCtrlGapOpusP}{{<}0.0002}
\newcommand{\statPCCtrlGapSonnet}{{-}0.163}
\newcommand{\statPCCtrlGapSonnetP}{{<}0.0002}
\newcommand{\statPCCtrlMLOpusCoef}{{+}0.145}
\newcommand{\statPCCtrlMLOpusP}{0.00082}
\newcommand{\statPCCtrlMLSonnetCoef}{{+}0.154}
\newcommand{\statPCCtrlMLSonnetP}{0.00026}

\newcommand{\statPCReleasedDevelopers}{71}
\newcommand{\statPCReleasedDatasets}{four}

\newcommand{\statPCHaikuSessions}{9}
\newcommand{\statPCUnresolvedSessions}{5}
\newcommand{\statPCTierExcluded}{14}
\newcommand{\statPCPermReps}{9{,}999}

\newcommand{\statPCWithinSince}{2025-09-18}
\newcommand{\statPCWithinUntil}{2026-12-31}
\newcommand{\statPCWithinAISessions}{487}
\newcommand{\statPCWithinAICommits}{1{,}493}
\newcommand{\statPCWithinNonAISessions}{1{,}874}
\newcommand{\statPCWithinNonAICommits}{4{,}914}
\newcommand{\statPCWithinDevelopers}{22}
\newcommand{\statPCWithinTauAI}{0.540}
\newcommand{\statPCWithinTauNonAI}{0.419}
\newcommand{\statPCWithinFloorB}{15}
\newcommand{\statPCWithinFloorC}{30}
\newcommand{\statPCWithinFloorD}{50}
\newcommand{\statPCWithinFZeroPairs}{21}
\newcommand{\statPCWithinFZeroHigher}{14}
\newcommand{\statPCWithinFZeroSignP}{0.19}
\newcommand{\statPCWithinFZeroWilcoxP}{0.022}
\newcommand{\statPCWithinFZeroMedian}{{+}0.127}
\newcommand{\statPCWithinFBPairs}{14}
\newcommand{\statPCWithinFBHigher}{10}
\newcommand{\statPCWithinFBSignP}{0.18}
\newcommand{\statPCWithinFBWilcoxP}{0.011}
\newcommand{\statPCWithinFBMedian}{{+}0.094}
\newcommand{\statPCWithinFCPairs}{9}
\newcommand{\statPCWithinFCHigher}{6}
\newcommand{\statPCWithinFCSignP}{0.51}
\newcommand{\statPCWithinFCWilcoxP}{0.098}
\newcommand{\statPCWithinFCMedian}{{+}0.061}
\newcommand{\statPCWithinFDPairs}{7}
\newcommand{\statPCWithinFDHigher}{5}
\newcommand{\statPCWithinFDSignP}{0.45}
\newcommand{\statPCWithinFDWilcoxP}{0.11}
\newcommand{\statPCWithinFDMedian}{{+}0.127}
\newcommand{\statPCWithinFBDeltaMin}{{-}0.078}
\newcommand{\statPCWithinFBDeltaMax}{{+}0.332}
\newcommand{\statPCWithinTopVolNonAI}{1{,}340}
\newcommand{\statPCWithinTopVolAI}{444}
\newcommand{\statPCWithinTopVolDelta}{{+}0.127}

\newcommand{\statPCPairedPairs}{12}
\newcommand{\statPCPairedHigher}{9}
\newcommand{\statPCPairedSignP}{0.15}
\newcommand{\statPCPairedWilcoxP}{0.034}
\newcommand{\statPCPairedMedian}{{+}0.117}
\newcommand{\statPCPairedFloorD}{50}
\newcommand{\statPCPairedFDPairs}{4}
\newcommand{\statPCPairedFDHigher}{2}
\newcommand{\statPCPairedFDSignP}{1}
\newcommand{\statPCPairedFDMedian}{{+}0.017}
\newcommand{\statPCPairedTopVolCtrl}{322}
\newcommand{\statPCPairedTopVolAI}{257}
\newcommand{\statPCPairedTopVolDelta}{{+}0.320}
\newcommand{\statPCPairedMaxRatio}{33}

\newcommand{\statPCCapSessions}{343}
\newcommand{\statPCCapCells}{8}
\newcommand{\statPCCapSteps}{7}
\newcommand{\statPCCapCellsOpus}{four}
\newcommand{\statPCCapCellsSonnet}{three}
\newcommand{\statPCCapCellsHaiku}{one}
\newcommand{\statPCCapAbsentOpusVersion}{5}
\newcommand{\statPCCapFamilyRho}{{+}0.012}
\newcommand{\statPCCapFamilyP}{0.82}
\newcommand{\statPCCapFamilyCoef}{{+}0.003}
\newcommand{\statPCCapFamilyCoefP}{0.75}
\newcommand{\statPCCapFamilyCI}{$[{-}0.016, \, {+}0.022]$}
\newcommand{\statPCCapVersionRho}{{-}0.008}
\newcommand{\statPCCapVersionP}{0.88}
\newcommand{\statPCCapVersionCoef}{{-}0.011}
\newcommand{\statPCCapVersionCoefP}{0.38}
\newcommand{\statPCCapVersionCI}{$[{-}0.036, \, {+}0.014]$}
\newcommand{\statPCCapSwing}{0.155}
\newcommand{\statPCCapTauSonnetMid}{0.569}
\newcommand{\statPCCapTauSonnetNew}{0.446}
\newcommand{\statPCCapTauHaiku}{0.352}
\newcommand{\statPCCapNHaiku}{8}

\title{Audit the Scaffold, Not the Checkpoint:\\
A Stationarity Dichotomy for Recursive Self-Improvement in Agentic Coding}

\author{
  Sebastian Bobadilla-Suarez\thanks{Corresponding author.} \\
  OnCorps \\
  \texttt{sebastian.suarez@oncorps.io} \\
  \And
  Bob Suh \\
  OnCorps \\
  \texttt{bob@oncorps.io} \\
  \And
  Ryan Fortin \\
  OnCorps \\
  \texttt{ryan@oncorps.io} \\
}

\begin{document}
\maketitle

\begin{abstract}
Rules for AI self-improvement often key on the model itself: has it been retrained, or are its weights (the trained model) frozen? We argue that, for AI coding agents, this check is incomplete, and we propose, but do not run, the experiment that would test the argument. The criterion is simple. An agent that keeps revising its own work runs into diminishing returns as long as the set of changes it can reach stays fixed, and can escape them only if that set grows. Our hypothesis is that the set can grow with the model untouched, because the software around the model, its \emph{scaffold} of tools, tests and task breakdown, can be rewritten. If so, a frozen model guarantees a ceiling eventually but not steady behavior on the way there, and audits should track changes to the scaffold, not only to the model. The same criterion separates three things often lumped together: more search within a fixed setup, retraining the model during use (which can raise the ceiling), and rewriting the scaffold, which sits between them.

Adding more agents does not lift the ceiling either. Picking the best of several agents' answers reaches the best single agent's ceiling exactly, and only gets there faster. A majority vote could do better, but only if the agents fail on different tasks. In our tests, \statPBOverlapK\ agents from one model family failed together: some task defeated all of them, and a majority failed \statPBOverlapViolate\ of tasks.

Our evidence is descriptive. On a public coding benchmark, the gain from each extra round of revision shrinks toward zero, though the unsolved tasks left to measure are also fewer and harder. Across \statPCSessions\ real development sessions, the amount of code changed falls off steadily from commit to commit, but programmers working before AI tools show the same pattern, so the data show that current practice levels off without showing why. We reach the criterion by treating repeated revision as a known machine-learning method, boosting, and report where real agents depart from it. Those departures are design choices, and together they describe an agent setup worth building.
\end{abstract}

\section{Introduction}
\label{sec:intro}

AI systems now write much of their own code. Anthropic reports that more than 80\% of the code merged into its production codebase was authored by Claude as of May 2026, up from low single digits before its coding-agent preview launched in February 2025, with engineers increasingly supplying the goal rather than the method \citep{anthropic2026aibuildsitself}. In \emph{agentic coding}, one or more AI agents draft code, test it, and revise it in a loop. That loop is a small, observable version of the scenario behind worries about recursive self-improvement (RSI), in which a system improves itself faster and faster \citep{good1965speculations, bostrom2014superintelligence}. This paper asks two questions: \emph{when does a loop of self-revision run out of room, and is a frozen model evidence that it has?}

\paragraph{The argument in brief.} Picture a craftsman whose hands are fixed but whose workshop is not. A sharper tool lets the same hands reach further, yet no rearrangement of the workshop makes the hands more dexterous. An AI coding agent is like this: the trained model (its \emph{weights}, or \emph{checkpoint}) is the hands, and the \emph{scaffold}, the tools, tests, prompts and task breakdown the model works inside, is the workshop. Our criterion (\S\ref{sec:theory}) says that self-revision levels off whenever the set of changes the agent can reach stays fixed, and can keep accelerating only if that set keeps growing. Our hypothesis is that rewriting the scaffold can grow the set while the model stays frozen. If it can, a frozen model guarantees only an eventual ceiling, not steady behavior on the way to it, and an auditor who checks only whether the model changed is checking too little. Current safety frameworks already test models with good scaffolding, but they schedule re-assessment around training events, so a scaffold rewritten between assessments goes unexamined (\S\ref{sec:discussion}).

\paragraph{What this paper does and does not show.} We reach the criterion by viewing repeated revision through \emph{boosting}, a classic machine-learning method in which weak predictors are combined, each correcting the errors left by the last. Each agent edits an existing draft rather than writing from scratch, much as each boosting step corrects the remaining error. The lens is a tool, not the contribution: the criterion itself does not depend on boosting and is elementary to prove. Its value is in naming what to track. Comparing real agent loops against the lens shows two places where they depart from it, and both are design choices. The empirical part is \emph{descriptive}. Current agentic coding does level off, on a benchmark and in production, but human programmers working before AI tools level off in the same way, so the data cannot show that a fixed reachable set is the \emph{cause} (\S\ref{sec:partC}). We also cannot measure the reachable set directly. \S\ref{sec:discussion} proposes a practical stand-in and the experiment that would test the hypothesis: hold the model fixed and vary only the scaffold. We do not run it. What follows is a criterion and a hypothesis, not a demonstrated mechanism.

\paragraph{Contributions.}
\begin{itemize}[leftmargin=*,itemsep=1pt]
    \item \textbf{A criterion for auditors.} Repeated revision with a bounded quality score must slow to zero gains (Thm.~\ref{thm:refinement}); it slows whenever the reachable set of edits is fixed and can escape only if that set grows (Prop.~\ref{prop:dichotomy}). If scaffold rewrites can grow it with the model frozen, a frozen model guarantees an eventual ceiling but not steady behavior along the way (Cor.~\ref{cor:frozen-ceiling}). Growth is necessary but not sufficient for escape, so the check errs toward false alarms. Current practice levels off, on SWE-bench and across \statPCSessions\ production sessions, but so does a \statPCCtrlSessions-session pre-AI baseline, so the cause is left open (\S\ref{sec:partB}, \S\ref{sec:partC}).
    \item \textbf{More agents do not raise the ceiling.} Picking the best of $k$ agents reaches the best agent's ceiling exactly and no higher (Prop.~\ref{prop:orchestration}). A majority vote could do better only if agents fail on different tasks (Prop.~\ref{prop:diversity}); ours, from one model family, fail together, and a majority fails \statPBOverlapViolate\ of tasks. Asking the same pool again each round also has a limit that can be computed in advance (Prop.~\ref{prop:pool-dichotomy}).
    \item \textbf{Where real agent loops depart from the lens.} They overwrite each draft instead of keeping and weighing every attempt, and their failures coincide. Both are design choices, so together they describe an agent setup worth building, which nobody, ourselves included, has built yet (\S\ref{sec:framework}).
    \item \textbf{A proposed decisive experiment.} Hold the model fixed, add scaffold capabilities one at a time, and measure both the per-round gains and how many tasks the agent can solve at all within many tries (\S\ref{sec:discussion}). We specify it; we do not run it.
\end{itemize}

\section{Related Work}
\label{sec:related}

\paragraph{Self-improvement and scaffolds.} The ``intelligence explosion'' \citep{good1965speculations, bostrom2014superintelligence} and the G\"odel machine \citep{schmidhuber2003goedel} motivate the RSI reading. Systems that improve code with the model held fixed are the case our ceiling addresses: STOP \citep{zelikman2023stop}, FunSearch \citep{romeraparedes2024funsearch}, and the Darwin G\"odel Machine \citep{zhang2025darwin}, which rewrites its own scaffold with frozen weights. The scaffold, or \emph{harness}, is increasingly treated as a layer to optimize in its own right \citep{he2026harness}. \citet{xue2026rethinking} let a frontier model design its own improvement process as it runs, close to pure scaffold expansion, and find it beats staged and reflective-evolutionary pipelines \citep{agrawal2025gepa} with a strong model but \emph{loses} with a medium one. Cor.~\ref{cor:frozen-ceiling} predicts that shape: scaffold expansion pays only as far as the frozen model allows. A recent survey splits RSI into convergent \emph{bounded self-refinement} and \emph{open-ended} RSI \citep{chen2026recursive}; our criterion (Prop.~\ref{prop:dichotomy}) gives that split a structural cause. Empirical multi-agent and self-refinement systems, and the test-time-compute literature whose diminishing returns our model predicts, are surveyed in Appendix~\ref{app:related}.

\paragraph{The boosting lens.} AdaBoost \citep{freund1997decision} and its margin theory \citep{schapire1998boosting, schapire2013explaining} supply our vocabulary. Reading boosting as gradient descent on the remaining error \citep{mason2000boosting, friedman2001greedy} is what lets a code patch play the role of a correction step, and mapping a specification to code is an instance of structured prediction \citep{tsochantaridis2005large}.

\paragraph{Coverage is not a vote.} Our overlap result could be mistaken for a known one. Repeated sampling measures \emph{coverage}, whether \emph{any} of $k$ attempts solves a task, and coverage keeps growing with $k$ \citep{brown2024monkeys}. That assumes something can recognize the good attempt afterwards. Without such a checker an orchestrator falls back on a vote, and Prop.~\ref{prop:diversity} governs the vote's error: the tasks a vote cannot recover are exactly those most workers fail. Growing coverage and a failing vote are therefore consistent. Cross-vendor aggregators \citep{wang2024mixtureofagents, jiang2023llmblender} report end-to-end gains but not the overlap statistic that decides whether a vote can work (open problem~2).

\paragraph{Three cases usually merged.} Extra inference compute searches within what the model and scaffold already allow. Test-time \emph{training} updates the model during use \citep{sun2019testtime, hardt2023nearest, sun2024learning, akyurek2024testtime, hubotter2024efficiently}, as does letting the model's own generated edits drive an update \citep{zweiger2025seal}, and so can raise the ceiling itself. Between them sits this paper's case: rewriting the scaffold with the model frozen. The separation matters for oversight: an auditor who checks only whether the model changed catches test-time training and misses scaffold expansion entirely (Appendix~\ref{app:rsi}).

\section{A Boosting Lens on Agentic Coding}
\label{sec:framework}

In boosting, a sequence of \emph{weak learners}, each only slightly better than chance, is combined into a strong one: each new learner concentrates on the examples the previous ones still get wrong, and the final answer is a weighted vote of all of them. The analogy to agentic coding is direct. A \emph{weak coding agent} does not write the program from scratch; it proposes a patch (a git diff) that corrects the remaining error between the current draft and the target. Even a small edge over chance is demanding here, since a random edit almost always breaks a working program (Remark~\ref{rem:capability-threshold}). Appendix~\ref{app:mapping} (Table~\ref{tab:mapping}) maps the two settings term by term.

\paragraph{Where real loops depart from boosting.} Two departures matter. First, \emph{real loops overwrite; boosting keeps}. A coding loop applies each patch on top of the last, so the newest draft replaces its predecessors ($R_T\circ\cdots\circ R_1$). Boosting keeps every learner and lets them vote, and it tracks which examples are still failing so it can focus on them. Real loops do neither (two further gaps, and whether task decomposition might supply them implicitly, in Appendix~\ref{app:decomposition}). Second, \emph{real agents fail together}. A vote helps only if workers make different mistakes, and on our pool they do not: every round has some task that defeats every worker (\S\ref{sec:theory}).

\paragraph{Both are design choices, so they describe a better harness.} A loop that kept every attempt, focused on the specifications still failing, and used workers diverse enough that fewer than half fail any task would satisfy boosting's guarantees by construction (Algorithm~\ref{alg:boost}; Prop.~\ref{prop:diversity} makes the diversity condition exact). We do not build that system and know of no one who has (open problem~1).

\paragraph{What the lens buys.} It supplies the vocabulary that makes ``will this run away?'' a structural question rather than a metaphor; it names the quantity to track, the set of reachable edits and its ceiling; and measuring where real loops depart from it produced the harness specification above and the orchestration results of \S\ref{sec:theory}.

\section{Theoretical Analysis}
\label{sec:theory}

This section states the criterion formally. The proofs are short, and that is deliberate: the claim is not that saturation was hard to prove but that the \emph{reachable set of edits} is the right thing to track. The algebra is machine-checked in Lean~4 (Mathlib, no \texttt{axiom}, no \texttt{sorry}), and the check earned its keep by \emph{refuting} one conjecture of ours (Prop.~\ref{prop:orchestration}). It cannot check what matters most: the modeling assumptions (M1), (M2) and the boosting mapping itself (Remark~\ref{rem:lean-honesty}).\footnote{The formalization is a quality-control layer against AI-hallucinated derivations rather than a contribution: this manuscript was drafted with AI assistance (Responsible Use Statement).} Full proofs are in Appendix~\ref{app:proofs}.

\paragraph{Setup.} A specification $x\in\X$ (a ticket, a failing test) is to be turned into code; drafts and the target $y$ live in $\Y$, with a distribution $\D$ over $\X\times\Y$ and a loss $L:\Y\times\Y\to\mathbb{R}_{\ge0}$ such as an edit distance (Table~\ref{tab:mapping}).

\begin{definition}[Weak Coding Agent]
\label{def:weak-agent}
A \emph{weak coding agent} proposes a patch via a \emph{refinement operator} $R_t:\Y\times\mathcal{A}\to\Y$, taking a draft $z$ and an edit $a\in\mathcal{A}$ to $z'=R_t(z,a)$; let $h_t(x)\in\{-1,+1\}$ indicate success on specification $x$, $h_t(x)=+1$ iff the draft returned matches the target, $L(y,z')=0$. Every label is $+1$ here, the target being solved by definition, so $\mathrm{err}_{\D}(h_t)=\Pr[h_t(x)={-}1]\le\tfrac12-\gamma_t$ and the appendix's generic $y_i$ reads as $+1$ throughout; $\Hyp$ is instantiated only in Theorem~\ref{thm:generalization}. This $\gamma_t$ is the edge on \emph{solving}, the one Thm.~\ref{thm:training-error} and Cor.~\ref{cor:convergence} require; Part~B measures the edge on \emph{improving} the previous draft, written $\tilde\gamma_t=\tfrac1N\sum_i\mathbf1[q_{i,t}>q_{i,t-1}]-\tfrac12$ on the pass fraction $q_{i,t}$ of task $i$ at round $t$, which scores a tie at ${-}\tfrac12$ under the strict $>$ (\S\ref{sec:partB}). The two need not agree, and which is reported shapes the headline.
\end{definition}

The boosting reading, taken literally (Def.~\ref{def:protocol}), would carry over three classical guarantees (Appendix~\ref{app:classical-statements}). We rely on none of them, since each needs a kept-and-reweighted ensemble that no real loop maintains. The results below make no use of boosting.

\paragraph{The criterion.} Suppose every round of revision raises a quality score that cannot exceed 1. The gains then have to add up to at most the room left, so they must shrink toward zero:

\begin{theorem}[Monotone Improvement under Refinement]
\label{thm:refinement}
Write $z_t=R_t(z_{t-1},a_t^*)$ for the draft after round $t$ at that round's best action $a_t^*$, and suppose each $R_t$ satisfies the improvement property $\mathbb{E}_{y\sim\D(\cdot\mid x)}[V(R_t(z,a_t^*))]\ge V(z)+\eta_t$ for quality function $V:\Y\to[0,1]$ and $\eta_t>0$. Then $\mathbb{E}[V(z_T)]\ge V(z_0)+\sum_{t=1}^T\eta_t$, and since $V\le1$, the total improvement saturates: $\sum_{t=1}^\infty\eta_t\le1-V(z_0)$.
\end{theorem}

That already implies diminishing returns, but only if the ceiling stays put. The next result says what happens when it does not, and it is the paper's central criterion:

\begin{proposition}[Stationarity Dichotomy]
\label{prop:dichotomy}
Consider the refinement game of Thm.~\ref{thm:refinement}. (i)~\textbf{Fixed ceiling $\Rightarrow$ saturation.} If $V(z)\le B$ for all $z$, then $\sum_{t=1}^\infty\eta_t\le B-\mathbb{E}[V(z_0)]$, forcing $\eta_t\to0$: strict diminishing returns. (ii)~\textbf{Sustained edge $\Rightarrow$ divergence.} If the effective hypothesis class expands enough to sustain a uniform edge $\eta_t\ge c>0$, then $\mathbb{E}[V(z_T)]$ grows without bound and no fixed ceiling can hold.
\end{proposition}

Both parts are verified in Lean (\texttt{Proposition8.lean}) and follow almost immediately from Thm.~\ref{thm:refinement}. Part~(i) is the bounded, saturating regime; part~(ii) is what an unbounded ``intelligence explosion'' would require. The real difficulty sits in the assumptions, above all that quality has a fixed maximum, which a genuinely open-ended objective would deny (Appendix~\ref{app:bounded-metrics}). The safety question is then whether the reachable set can grow while the model stays the same:

\begin{corollary}[Frozen-Weight Ceiling]
\label{cor:frozen-ceiling}
Model a self-improving system as $\mathcal{S}=(W,C)$, frozen weights $W$ and mutable scaffold $C$. Assume (M1)~\emph{class expandability}: the effective hypothesis class $\Hyp_t(C)$ can expand under a scaffold rewrite without altering $W$; and (M2)~\emph{ultimate ceiling}: $V(z)\le V^*(W)$ for all $z$ reachable by any scaffold the model can construct. Then (i)~the system saturates at a value no greater than $V^*(W)$ (Prop.~\ref{prop:dichotomy}(i) with $B=V^*(W)$); (ii)~escaping this ceiling requires raising $V^*(W)$ itself---modifying weights or architecture---not only rewriting the scaffold (contrapositive of Prop.~\ref{prop:dichotomy}(ii)).
\end{corollary}

So freezing the model does \emph{not} by itself place a system in the safe, saturating regime: under (M1) the reachable set can grow while $W$ is fixed, and the frozen model guarantees only the eventual ceiling of part~(i). Neither (M1) nor (M2) is a theorem; both are explicit assumptions (Remark~\ref{rem:lean-honesty}). The criterion also runs one way only. Escape needs a gain that stays above a fixed floor, so a scaffold that expands while its gains shrink still levels off: growth is necessary for escape but not sufficient. Appendix~\ref{app:rsi} works one example, an agent that writes a test-isolation script and reuses it in later rounds, which unlocks patches it could not previously check, with the model untouched.

\paragraph{More agents: the ceiling binds sideways too.} An orchestrator that runs $k$ frozen workers, each under its own scaffold $C_i$ (different prompts, tools or roles, possibly the same model), faces the same ceiling. Write $V^*(W,C_i):=\sup\{V(z):z\text{ reachable under }C_i\}$ for worker $i$'s ceiling.

\begin{proposition}[Orchestration Selects a Ceiling; It Does Not Raise One]
\label{prop:orchestration}
Best-of-$k$ selection attains $V^*_{\mathrm{orch}}=\max_i V^*(W,C_i)$ \emph{exactly}---an unconditional equality, holding even when no single worker's reachable class contains the union of all of them. Consulting all $k$ workers every round of the refinement game gives $\mathbb{E}[V(z_T)]\ge V(z_0)+\sum_{t=1}^T\max_i\eta_t^{(i)}$, dominating every worker's depth-only guarantee, yet since $V\le1$ the total improvement still saturates at $1-V(z_0)$: width buys rate, not budget. (The width half needs one hypothesis beyond Thm.~\ref{thm:refinement}'s---monotonicity of the inner action-expectation, since selection compares two integrands under the same expectation; Prop.~\ref{prop:width-refinement} states it and the formalization flags it.)
\end{proposition}

Picking the best answer therefore chooses which existing ceiling to reach, and reaches it sooner; it does not create a higher one. That makes ``upgrade the model before adding orchestration'' a consequence rather than a rule of thumb. We had conjectured that selection could strictly beat the best worker, and Lean found the counterexample we had missed: two workers that each reach one different answer of equal value.

\paragraph{A vote needs workers who fail on different tasks.} Voting is the one arrangement that can beat the best single worker, because voters answer the same question side by side (Prop.~\ref{prop:soft-aggregation}). What it needs is diversity, and the requirement reduces to counting:

\begin{proposition}[Bounded Overlap Characterizes a Correct Vote]
\label{prop:diversity}
Let $k$ workers with $\pm1$ outputs be combined by the \emph{unweighted} majority vote, and let $m_i:=\#\{t<k:\,y_ih_t(x_i)={-}1\}$ count the workers that fail point $i$, with $m^\star:=\max_i m_i$ the \emph{tight overlap}. Then the vote has zero training error \emph{iff} $2m_i<k$ for every $i$---equivalently, iff $m^\star{<}k/2$. The condition is sharp in two senses: a single point at $2m_i{=}k$ already makes the error nonzero, and if \emph{every} point sits there the error is $1$.
\end{proposition}

The proposition is bookkeeping; the measurement carries the content, because it turns ``is this pool diverse enough to vote?'' into one countable number. Part~B's pool fails at the extreme: in \emph{every} round $m^\star{=}k$, so some task defeats every worker, and a majority fails \statPBOverlapViolate\ of tasks, which \emph{is} the vote's error (\statPBOverlapVoteErr, 95\%~CI \statPBOverlapVoteErrCI). The vote's error also exceeds the mean single-worker error \statPBOverlapEpsBar, but that gap's interval includes zero (\statPBOverlapGap, 95\%~CI \statPBOverlapGapCI) and vanishes within one capability tier, so we do not read it as a fact about voting (Appendix~\ref{app:orchestration}).

This is close to the weakest diversity test possible: one vendor, two tiers, three prompt variants at temperature $0.25$, with the prompt contrast itself null (\S\ref{sec:partB}), so the pool's effective diversity is nearer two configurations than \statPBOverlapArms. It shows that resampling one model family inherits that family's blind spots. Whether workers from genuinely different vendors would clear the bar is open, and this data cannot reach it.

\begin{figure}[t]
\centering
\includegraphics[width=0.54\textwidth]{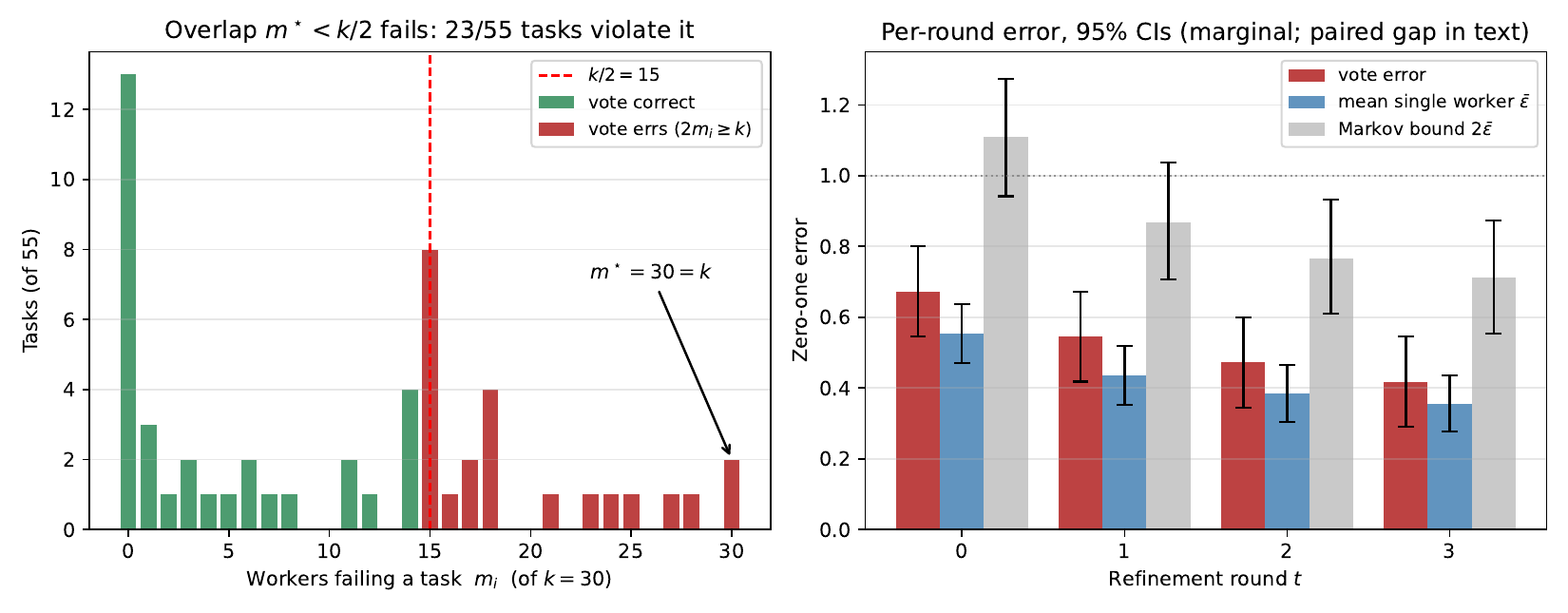}
\caption{Prop.~\ref{prop:diversity}'s overlap condition measured on $k{=}\statPBOverlapK$ workers (\statPBOverlapArms\ arms $\times$ \statPBSeeds\ seeds), shared \statPBK-task pool. \textbf{Left:} workers failing each task; the condition needs every task strictly left of the dashed $k/2$ line, so the red mass is the vote's error and the spike at $m_i{=}k$ is total overlap. \textbf{Right:} per round, vote error against mean per-worker error $\bar\epsilon$ and the Markov bound $2\bar\epsilon$.}
\label{fig:llm-overlap}
\end{figure}

\paragraph{Asking the same pool again has a limit set in advance.} An orchestrator usually consults its workers again each round. How long that can pay is decided by one property of the pool: whether some weighted combination of its workers can already get every task right.

\begin{proposition}[Reused-Pool Dichotomy]
\label{prop:pool-dichotomy}
Fix a pool of $k$ workers and run Algorithm~\ref{alg:boost} over it under any schedule at the optimal confidences, with every scheduled worker having edge $\gamma>0$ against the distribution current when it is used. Let $\epsilon_0$ be the smallest zero-one training error attained by any nonnegative combination of the pool. (i)~If $\epsilon_0>0$, then $\epsilon_0\le e^{-2\gamma^2T}$, hence $T\le\log(1/\epsilon_0)/2\gamma^2$---the edge hypothesis provably fails by a finite, computable round, uniformly over the schedule and the confidences, and depending on $k$ only through $\epsilon_0$. (ii)~If instead the pool interpolates ($\epsilon_0{=}0$) with margin $\theta$ at coefficient sum $S$, then for \emph{every} distribution some member has edge $\gamma{=}\theta/2S>0$, fixed \emph{before} the distribution is chosen, so the hypothesis survives every round of any run. The cases are exhaustive and mutually exclusive.
\end{proposition}

A fixed pool is therefore not a renewable resource: unless it can already fit the work, consulting it again has a limit set by $\epsilon_0$ and $\gamma$, not by how much compute an orchestrator is willing to spend. And which side a pool falls on is, again, a property of the scaffold rather than of any weight. Both halves are Lean-checked (Props.~\ref{prop:pool-horizon} and~\ref{prop:interpolating-pool}).

\section{Experiments}
\label{sec:experiments}

Three parts. \textbf{Part A} only checks that the boosting machinery is implemented correctly, on synthetic labels, and claims nothing about real agents (Appendix~\ref{app:parta}). \textbf{Part B} runs a real coding agent on a public benchmark. \textbf{Part C} looks at \statPCSessions\ real development sessions. All code, processed data, figure inputs and the raw model response cache are in the supplementary repository.

\subsection{Part B: An Agent Revising Its Own Work on SWE-bench Lite}
\label{sec:partB}

\paragraph{Design.} The agent is \emph{mini-swe-agent} \citep{yang2024sweagent}, given $T{=}\statPBRounds$ rounds to fix each of $K{=}\statPBK$ real GitHub issues from SWE-bench Lite \citep{jimenez2023swe}, with \statPBSeeds\ independent runs per task at temperature $0.25$. Pilot runs showed the improvement rate collapsing by round~1, which set $T$. We cross two factors ($\statPBTrajectories$ runs in all). \emph{Model tier}: Claude Sonnet (stronger) vs.\ Claude Haiku~4.5 (weaker). \emph{Feedback}: what each later round is told about the previous attempt, either nothing (\textbf{independent}), the previous patch only (\textbf{blind}), or the previous patch plus which target tests still fail (\textbf{diagnostic}). No arm ever sees the correct patch.

\paragraph{Scoring, and a bug we found in it.} Each attempt is scored by actually running the repository's tests in a pinned Docker container: real pass/fail, not an AI judge. The official SWE-bench harness scores only once at the end, and we needed a score after every round, so we wrote our own scorer. Before trusting it we ran it on the \emph{known-correct} patches, and \statPBAuditFailN\ of the originally sampled $K{=}50$ tasks still failed, across \statPBAuditRootCauses\ root causes in \emph{our} scorer, each of which would otherwise have been counted against the model. We fixed the two shared causes, re-verified against the correct patches, and rebuilt the pool to $K{=}\statPBK$. The resulting pool favors repositories our scorer handles, so it is plausibly easier than SWE-bench Lite as a whole (Appendix~\ref{app:partb}).

\paragraph{Model strength matters a lot; feedback type shows no detectable effect.} Model tier explains roughly half the variance ($F\statPBModelDf{=}\statPBModelF$, $p\statPBModelP$, partial $\eta^2{=}\statPBModelEta$; full ANOVA in Appendix~\ref{app:partb}). Haiku solves \statPBHaikuSolveRange\ of tasks across arms and Sonnet \statPBSonnetSolveRange\ (a task counts as solved when $\ge\tau{=}0.6$ of its tests pass; 95\%~Wilson CIs). Neither feedback type ($p{=}\statPBModeP$) nor the interaction ($p{=}\statPBInterP$) is significant, but at this sample size the test has only $\statPBModePower$/$\statPBInterPower$ power, and 80\% power would need \statPBPowerMultiple\ the tasks. \emph{The feedback null is underpowered, not evidence of no effect.}

These are not SWE-bench resolve rates, since our $\tau$ counts the \emph{fraction} of tests passing. Under SWE-bench's own all-tests criterion ($\tau{=}\statPBTauResolved$) the same analysis gives Haiku \statPBHaikuResolvedRange\ and Sonnet \statPBSonnetResolvedRange: the tier gap survives, and the weaker tier's feedback ordering flips, one more reason to read the feedback contrast as noise. The tier \emph{contrast} is the result that should transfer, not the absolute rates.

\paragraph{Later rounds mostly change nothing.} Most rounds after the first leave the code as it was: ties make up \statPBTieRange\ of round-to-round transitions overall, and \statPBSonnetGenuineTieRange\ (Sonnet) and \statPBHaikuGenuineTieRange\ (Haiku) after excluding rounds that were never run (below), while at most \statPBRegressMax\ of rounds make things worse. The loop is wasteful rather than harmful: it acts less and less, not wrongly. On tasks that could still improve, the improvement \emph{edge} (the share of rounds that improve, minus one half, so that zero is boosting's weak-learning threshold) is Haiku \statPBHaikuEdgeHeadroom\ on \statPBHaikuHeadroomN\ rows and Sonnet \statPBSonnetEdgeHeadroom\ on \statPBSonnetHeadroomN. Sonnet's first two intervals include zero, so we do not claim the strong model falls below the threshold (Appendix~\ref{app:edge-decomposition}, Fig.~\ref{fig:llm-edge}).

\paragraph{One measurement is distorted by the harness, and we set it aside.} The loop stops calling the model once a task passes and fills the remaining rounds with a perfect score. A simple ``did this round improve?'' rate then counts every filled-in round as a failure to improve, so pooled over all rounds it tracks how many tasks were solved, and backwards: the stronger model solves more, fills in more, and looks worse. We report that pooled series only for disclosure (Appendix~\ref{app:edge-decomposition}). Nothing in \S\ref{sec:theory} depends on it.

\paragraph{The headline number, and its two confounds.} The quantity that bears on the criterion is the mean gain per round on each task's best score so far. It falls \statPBHaikuEta\ (Haiku) and \statPBSonnetEta\ (Sonnet): gains continue but shrink toward zero, the shape Thm.~\ref{thm:refinement} predicts. Filled-in rounds do not bias this number the way they bias the improvement rate, since they add exactly zero gain. But two confounds sit beside it, and either could produce part of the decline without any ceiling. First, the mean runs over all tasks, so solved tasks add zeros and the set of tasks still able to improve shrinks and gets harder with each round. Second, because the loop stops on success, no round after a first pass was ever run, so this is not an unconditional per-round curve. A fixed set of tasks with every round executed would remove both (\S\ref{sec:limitations}). This is also one vendor's two tiers, one harness, one benchmark and $T{=}\statPBRounds$ rounds.

\subsection{Part C: Revision in Real Development Work}
\label{sec:partC}

\paragraph{Data.} \statPCSessions\ real sessions of AI-assisted software development at an enterprise software company. Unlike benchmark tasks, production work is not chosen to be nearly solvable, so it leaves the room for improvement the theory needs. We extract structural git metrics only, keeping no code, file paths, commit messages or author identities. A \emph{session} is a run of consecutive AI-co-authored commits less than 4~hours apart: \statPCSessions\ sessions, \statPCCommits\ commits, \statPCRepos\ repositories. Construction, caveats on how model tiers are labeled, and all five findings are in Appendix~\ref{app:partc}.

\paragraph{Finding 1: the amount of change falls off over a session.} \statPCDecayPct\ of multi-commit sessions show \emph{churn} (lines added plus lines deleted) decreasing from commit to commit, with a median late-to-early ratio of \statPCMedianRatio\ (95\%~CI \statPCPermCI; Appendix Fig.~\ref{fig:emp}, left). The order is genuinely temporal, not an artifact of which commits happen to be large: shuffling commit order destroys it (permutation test, $p\statPCPermP$, on the $n{=}\statPCPermN$ sessions with ${\ge}5$ commits; half the cohort is a single commit). Descriptively, then, current frozen-model agentic coding looks like the bounded regime~(i): it behaves like an editor that converges, not a process that compounds.

\paragraph{Findings 2--3: what this does not show.} It does not show the mechanism, for two reasons. First, the control that could have pointed to something AI-specific does not: \statPCCtrlSessions\ sessions by human programmers before AI tools show the same \emph{shape} of decay (geometric decay factor $\rho_{\mathrm{human}}{=}\statPCCtrlRho$ against $\rho{\approx}\statPCRho$ for AI sessions). Fast early change followed by polishing is how editing works in general, and a fixed reachable set is not needed to predict it. Second, the one test specific to our theory comes back null. If a more capable model can reach more, its sessions should keep more of what they write; but Opus sessions show no higher \emph{survivor ratio} $\psi$ (net lines over total churn) than Sonnet sessions ($\psi_{\mathrm{Opus}}{\approx}\statPCTauOpus$, $n{=}\statPCNOpus$ vs.\ $\psi_{\mathrm{Sonnet}}{\approx}\statPCTauSonnet$, $n{=}\statPCNSonnet$; gap${=}\statPCGap$, 95\%~CI \statPCGapCI, permutation $p{=}\statPCGapPermP$; Appendix Fig.~\ref{fig:emp}, right). The null survives four checks for confounds and a re-test by model version. Since Part~B finds a large capability effect, we read this as a limit of the proxy rather than of the models (\S\ref{sec:discussion}), and as a \emph{bounded} null, not a demonstrated zero.

\section{Discussion}
\label{sec:discussion}

\paragraph{For oversight: what to monitor.} Frontier safety frameworks already test models with strong scaffolding, so the gap is not in how capability is elicited but in \emph{when} re-assessment happens: it is largely triggered by training-side events such as a finished post-training run or a jump in effective compute \citep{anthropic2026rsp, deepmind2026fsf, openai2025preparedness}, and the frameworks themselves treat scaffold-driven gains as an open problem (Appendix~\ref{app:rsi}). Under (M1), a scaffold rewritten \emph{between} assessments can grow what a system reaches with no training event to prompt a new look. The quantity to monitor is therefore whether the scaffold is expanding itself, not only whether the weights changed. A loop that only re-samples is bounded by construction; one that rewrites its own tools or task breakdown is not (Cor.~\ref{cor:frozen-ceiling}). Self-expansion is also a reward-hacking risk, since a system allowed to rewrite its own tests can game them rather than improve (Appendix~\ref{app:rsi}--\ref{app:bounded-metrics}).

\paragraph{For practitioners.} Two lessons follow, with five in Appendix~\ref{app:guidance}. \textbf{Spend on the first attempt, not the revision loop}: later rounds mostly leave the code unchanged (Appendix Fig.~\ref{fig:llm-edge}). \textbf{Upgrade the model before adding orchestration}: orchestration cannot exceed the best worker's ceiling (Prop.~\ref{prop:orchestration}), and model tier outweighed every other factor we varied (partial $\eta^2{=}\statPBModelEta$).

\paragraph{Why model strength matters in the lab but not in production.} Part~B's tier contrast explains roughly half the variance; Part~C's finds no gap between tiers. We do not resolve this, but the two contrasts differ in ways that could each explain it. In production, developers choose the model rather than having it assigned at random, and Opus is often used for planning while Sonnet does the execution, so a capability gap can be offset by harder or different work going to the stronger model. Part~C's adjustments for task type, repository and developer remove the coarsest form of that selection, not selection within a task type. The outcomes differ too: Part~B scores whether the tests pass on a fixed task, Part~C a structural proxy on whatever work a session took on. And the models differ: Part~B compares Sonnet with Haiku, Part~C two upper tiers. Random model assignment on matched tasks, or a functional outcome in place of the proxy, would separate these explanations; this data does neither.

\paragraph{The decisive experiment, proposed and not run.} The confounds above, and Part~C's human baseline, all point to one design this paper does not execute. Hold the model fixed and climb a ladder of scaffolds: (S0)~no tools; (S1)~the agent can run tests; (S2)~it can also search or retrieve code; (S3)~it can write tools during a run and keep them for later rounds. Run every rung on a fixed set of tasks with every round executed, including after a first pass, so neither confound of \S\ref{sec:partB} applies. Record two curves per rung: the gain per round $\bar\eta_t$, and pass@$k$ \emph{coverage}, the share of tasks on which \emph{any} of $k$ attempts passes the tests \citep{brown2024monkeys}. Coverage is our \emph{proposed} practical stand-in for the reachable set, which we otherwise cannot measure: comparing rungs at the same $k$ asks whether a rung reaches tasks the previous one could not, rather than the same tasks sooner. It is a lower bound under a finite budget and depends on the tests, so it approximates the set rather than measuring it. The hypothesis predicts that within each fixed rung (S0--S2) gains shrink and coverage plateaus; that rungs adding new reach raise the coverage plateau with the model unchanged, which is (M1); and that only (S3), whose reach can grow mid-run, can show a plateau that rises within a run. It is falsified if coverage is flat across rungs, so these scaffolds do not expand reach and a model-level audit misses nothing on them, or if every rung decays like the human baseline whatever its reach, so the ordinary economics of editing explain the data.

\paragraph{Open problems.} Each is falsifiable, and blocked by something we did not do rather than something we found. \textbf{(1)~Build the harness} of \S\ref{sec:framework}, keeping every attempt, focusing on failing specifications, and enforcing $m^\star{<}k/2$: does it improve at the rate boosting promises, or does a new gap appear only in practice? \textbf{(2)~Run the vote across model vendors}: do workers from genuinely different families satisfy Prop.~\ref{prop:diversity}? The measurement is cheap and has not been made. If they fail it, voting orchestration has no viable regime on real code. \textbf{(3)~Keep revising after the first pass}: one re-run that does not stop on success would give the unconditional per-round curve this data cannot. \textbf{(4)~Hold the model fixed and vary the scaffold}, the design above, and the only one that separates a fixed reachable set from the ordinary economics of editing.

\subsection{Limitations}
\label{sec:limitations}

The four below could each overturn a claim we make, or are why we do not make one; seven more that bound scope without threatening a specific result are in Appendix~\ref{app:limitations}.

\begin{itemize}[leftmargin=*,itemsep=1pt]
    \item \textbf{The slowdown is not specific to AI, which is why the evidence is descriptive.} Human programmers before AI tools show the same shape of decay (\S\ref{sec:partC}). The data therefore place current practice in the bounded regime, but a fixed reachable set is not needed to predict that shape, so they cannot attribute it to one. Only a fixed-model design such as the one in \S\ref{sec:discussion} could.
    \item \textbf{The per-round measurements are conditional.} Because the harness stops calling the model once a task passes, the later rounds that would give an unconditional per-round curve were never run (\S\ref{sec:partB}). What we report conditions on a task still being unsolved, so the population shrinks and gets harder with each round (Sonnet \statPBSonnetHeadroomN), and part of any decline is a change of sample rather than of behavior. That applies to both the improvement edge and the gain per round; a fixed set of tasks with every round executed would remove it (open problem~3). Reading the strong model's zero-crossing intervals as proof that revision \emph{does} keep paying would over-read them as much as the pooled series over-reads the opposite. Part~C does not measure the edge at all, so Cor.~\ref{cor:convergence} is conditional throughout.
    \item \textbf{Rounds may restart rather than revise (an unchecked assumption).} The model treats each round as an edit of the previous draft, $z_t=R_t(z_{t-1},a_t)$, but an agent can abandon its approach mid-task and start over. Thm.~\ref{thm:refinement}'s bound on the best score so far still holds for such a path; what fails is reading the shrinking gains as revision saturating, and reading patches as building on each other. Pass/fail scoring cannot tell a restart from a revision, and we did not check, for instance by comparing patches across rounds. Part~B's independent arm restarts by design, and the underpowered feedback null cannot separate it from the others.
    \item \textbf{Observational data, model choice, and version pooling.} In Part~C the model tier is \emph{not} randomly assigned: Opus is often used for planning and Sonnet for execution, so the tier contrast is partly a task-type contrast. The survivor ratio measures structural convergence, not whether the code works. Finding~2's checks adjust for task type, repository and developer but not model version, which is why it is re-run by version. That the null survives rules out version pooling as the explanation without making it a demonstrated zero, and reading it as evidence \emph{against} a capability effect would over-read it. Finding~3's human comparison is likewise a comparison across eras.
\end{itemize}
Orchestrator-worker topologies are treated formally in Appendix~\ref{app:orchestration}.

\section{Conclusion}
\label{sec:conclusion}

Viewing agentic coding through boosting yields a simple criterion: repeated self-revision levels off whenever the set of reachable edits is fixed and quality has a ceiling, and can escape only if that set keeps growing (Prop.~\ref{prop:dichotomy}). If scaffold rewrites can grow it with the model frozen (M1), a frozen model does not by itself guarantee a safe, levelling-off trajectory: audit the scaffold, not only the checkpoint. Adding agents does not change this. Picking the best answer reaches the best worker's ceiling exactly (Prop.~\ref{prop:orchestration}), a vote would need workers who fail on different tasks and ours do not (a majority fails \statPBOverlapViolate\ of tasks), and asking the same pool again has a limit computable in advance unless the pool can already fit the work (Prop.~\ref{prop:pool-dichotomy}). Real loops depart from boosting by overwriting drafts and by failing together; both are design choices, and together they describe a harness worth building. What we observe is levelling off, on SWE-bench and across \statPCSessions\ production sessions, matched in shape by human programmers before AI tools. What remains open is why: a genuinely fixed reachable set, or the ordinary economics of editing. Holding the model fixed while varying the scaffold would tell them apart. This paper specifies that experiment; it does not run it.

\section*{Responsible Use Statement}
\label{sec:responsible-use}

Recursive self-improvement in coding agents is dual-use adjacent, but this paper's finding runs the other way: a \emph{saturation and limits} result, not an acceleration recipe. The criterion, whether the reachable hypothesis class is stationary, lets evaluators \emph{recognize} when a self-modifying system could escape diminishing returns instead of reading frozen weights as assurance, and we release no technique or scaffold that expands an agent's reachable class. Part~C releases \statPCReleasedDatasets\ session datasets covering \statPCReleasedDevelopers\ people's commit timing, structural only: developer identity is a salted hash whose roster is withheld, and no code content, file paths or commit messages survive (sanitization audit in Appendix~\ref{app:partc}). This is IP-safe but not anonymous in the strong sense, since commit timing is in principle linkable by anyone holding the original repositories. We also disclose substantial AI assistance in drafting this manuscript.

\clearpage
\typeout{MAINBODY-ENDPAGE:\thepage}

\begin{ack}
The authors thank Tom Edwards and colleagues at OnCorps for feedback and
discussions that improved this work. We gratefully acknowledge the engineers at OnCorps,
past and present, whose production code forms the Part~C dataset.
We thank Brett Roads for detailed comments on an earlier draft that
substantially improved its clarity and framing. We also thank Peter Riefer,
Aristeidis Panos, Bradley Love and Simon Busch-Moreno for their feedback, and
the Data Science Team at OnCorps for comments on the manuscript.
\end{ack}

\bibliographystyle{plainnat}
\bibliography{references}

\begin{thebibliography}{59}
\providecommand{\natexlab}[1]{#1}
\providecommand{\url}[1]{\texttt{#1}}
\expandafter\ifx\csname urlstyle\endcsname\relax
  \providecommand{\doi}[1]{doi: #1}\else
  \providecommand{\doi}{doi: \begingroup \urlstyle{rm}\Url}\fi

\bibitem[Agrawal et~al.(2026)Agrawal, Tan, Soylu, Ziems, Khare, Opsahl-Ong,
  Singhvi, Shandilya, Ryan, Jiang, Potts, Sen, Dimakis, Stoica, Klein, Zaharia,
  and Khattab]{agrawal2025gepa}
Lakshya~A Agrawal, Shangyin Tan, Dilara Soylu, Noah Ziems, Rishi Khare, Krista
  Opsahl-Ong, Arnav Singhvi, Herumb Shandilya, Michael~J Ryan, Meng Jiang,
  Christopher Potts, Koushik Sen, Alexandros~G. Dimakis, Ion Stoica, Dan Klein,
  Matei Zaharia, and Omar Khattab.
\newblock {GEPA}: Reflective prompt evolution can outperform reinforcement
  learning.
\newblock In \emph{International Conference on Learning Representations
  (ICLR)}, 2026.

\bibitem[{AI Security Institute, UK}(2024)]{aisi2024inspect}
{AI Security Institute, UK}.
\newblock Inspect: An open-source framework for large language model
  evaluations.
\newblock \url{https://inspect.aisi.org.uk/}, 2024.
\newblock Released 2024-05-10 as a platform of the UK AI Safety Institute,
  since renamed the AI Security Institute; developed with Meridian Labs.
  Framework source at github.com/UKGovernmentBEIS/inspect\_ai, companion
  evaluation collection at github.com/UKGovernmentBEIS/inspect\_evals. Accessed
  2026-08-21.

\bibitem[Aky{\"u}rek et~al.(2025)Aky{\"u}rek, Damani, Zweiger, Qiu, Guo, Pari,
  Kim, and Andreas]{akyurek2024testtime}
Ekin Aky{\"u}rek, Mehul Damani, Adam Zweiger, Linlu Qiu, Han Guo, Jyothish
  Pari, Yoon Kim, and Jacob Andreas.
\newblock The surprising effectiveness of test-time training for few-shot
  learning.
\newblock In \emph{International Conference on Machine Learning (ICML)}, 2025.

\bibitem[Amodei et~al.(2016)Amodei, Olah, Steinhardt, Christiano, Schulman, and
  Man{\'e}]{amodei2016concrete}
Dario Amodei, Chris Olah, Jacob Steinhardt, Paul Christiano, John Schulman, and
  Dan Man{\'e}.
\newblock Concrete problems in {AI} safety, 2016.

\bibitem[{Anthropic}(2026)]{anthropic2026rsp}
{Anthropic}.
\newblock Anthropic's responsible scaling policy, version 3.4.
\newblock \url{https://www.anthropic.com/rsp}, 2026.
\newblock Effective 2026-07-08. Accessed 2026-08-18.

\bibitem[Beygelzimer et~al.(2015)Beygelzimer, Kale, and
  Luo]{beygelzimer2015optimal}
Alina Beygelzimer, Satyen Kale, and Haipeng Luo.
\newblock Optimal and adaptive algorithms for online boosting.
\newblock In \emph{Proceedings of the 32nd International Conference on Machine
  Learning}, pages 2323--2331, 2015.

\bibitem[Bostrom(2014)]{bostrom2014superintelligence}
Nick Bostrom.
\newblock \emph{Superintelligence: Paths, dangers, strategies}.
\newblock Oxford University Press, 2014.

\bibitem[Brown et~al.(2024)Brown, Juravsky, Ehrlich, Clark, Le, R{\'e}, and
  Mirhoseini]{brown2024monkeys}
Bradley Brown, Jordan Juravsky, Ryan Ehrlich, Ronald Clark, Quoc~V Le,
  Christopher R{\'e}, and Azalia Mirhoseini.
\newblock Large language monkeys: Scaling inference compute with repeated
  sampling, 2024.

\bibitem[Chen et~al.(2026)Chen, Wang, and Qu]{chen2026recursive}
Mingguang Chen, Licheng Wang, and Bo~Qu.
\newblock Recursive self-improvement in ai: From bounded self-refinement to
  autonomous research loops, 2026.

\bibitem[{DeepSeek-AI}(2025)]{deepseek2025r1}
{DeepSeek-AI}.
\newblock {DeepSeek-R1}: Incentivizing reasoning capability in {LLMs} via
  reinforcement learning, 2025.

\bibitem[Favaro and Clark(2026)]{anthropic2026aibuildsitself}
Marina Favaro and Jack Clark.
\newblock When ai builds itself: Our progress toward recursive
  self-improvement, and its implications.
\newblock The Anthropic Institute,
  \url{https://www.anthropic.com/institute/recursive-self-improvement}, 2026.
\newblock Accessed 2026-08-03.

\bibitem[Freund and Schapire(1997)]{freund1997decision}
Yoav Freund and Robert~E Schapire.
\newblock A decision-theoretic generalization of on-line learning and an
  application to boosting.
\newblock \emph{Journal of Computer and System Sciences}, 55\penalty0
  (1):\penalty0 119--139, 1997.

\bibitem[Friedman et~al.(2000)Friedman, Hastie, and
  Tibshirani]{friedman2000additive}
Jerome Friedman, Trevor Hastie, and Robert Tibshirani.
\newblock Additive logistic regression: A statistical view of boosting.
\newblock \emph{Annals of Statistics}, 28\penalty0 (2):\penalty0 337--407,
  2000.

\bibitem[Friedman(2001)]{friedman2001greedy}
Jerome~H Friedman.
\newblock Greedy function approximation: A gradient boosting machine.
\newblock \emph{Annals of Statistics}, 29\penalty0 (5):\penalty0 1189--1232,
  2001.

\bibitem[Good(1965)]{good1965speculations}
Irving~John Good.
\newblock Speculations concerning the first ultraintelligent machine.
\newblock \emph{Advances in computers}, 6:\penalty0 31--88, 1965.

\bibitem[{Google DeepMind}(2026)]{deepmind2026fsf}
{Google DeepMind}.
\newblock Frontier safety framework, version 3.1.
\newblock
  \url{https://deepmind.google/discover/blog/strengthening-our-frontier-safety-framework/},
  2026.
\newblock Effective 2026-04-17. Accessed 2026-08-18.

\bibitem[Hardt and Sun(2024)]{hardt2023nearest}
Moritz Hardt and Yu~Sun.
\newblock Test-time training on nearest neighbors for large language models.
\newblock In \emph{International Conference on Learning Representations
  (ICLR)}, 2024.

\bibitem[He et~al.(2026)He, Zhou, Wang, Xu, Liu, and Miao]{he2026harness}
Chaoyue He, Xin Zhou, Di~Wang, Hong Xu, Wei Liu, and Chunyan Miao.
\newblock Harness engineering for language agents: The harness layer as
  control, agency, and runtime.
\newblock Preprints.org, 2026.

\bibitem[Hernandez and Brown(2020)]{hernandez2019measuring}
Danny Hernandez and Tom~B Brown.
\newblock Measuring the algorithmic efficiency of neural networks, 2020.

\bibitem[Hong et~al.(2023)Hong, Zhuge, Chen, Zheng, Cheng, Zhang, Wang, Wang,
  Yau, Lin, et~al.]{hong2023metagpt}
Sirui Hong, Mingchen Zhuge, Jiaqi Chen, Xiawu Zheng, Yuheng Cheng, Ceyao Zhang,
  Jinlin Wang, Zili Wang, Steven Ka~Shing Yau, Zijuan Lin, et~al.
\newblock Metagpt: Meta programming for a multi-agent collaborative framework,
  2023.

\bibitem[Huang et~al.(2023)Huang, Zhang, Luck, Bu, Qing, and
  Cui]{huang2023agentcoder}
Dong Huang, Jie~M. Zhang, Michael Luck, Qingwen Bu, Yuhao Qing, and Heming Cui.
\newblock Agentcoder: Multi-agent-based code generation with iterative testing
  and optimisation, 2023.

\bibitem[H{\"u}botter et~al.(2025)H{\"u}botter, Bongni, Hakimi, and
  Krause]{hubotter2024efficiently}
Jonas H{\"u}botter, Sascha Bongni, Ido Hakimi, and Andreas Krause.
\newblock Efficiently learning at test-time: Active fine-tuning of {LLMs}.
\newblock In \emph{International Conference on Learning Representations
  (ICLR)}, 2025.

\bibitem[Jiang et~al.(2023)Jiang, Ren, and Lin]{jiang2023llmblender}
Dongfu Jiang, Xiang Ren, and Bill~Yuchen Lin.
\newblock {LLM}-blender: Ensembling large language models with pairwise ranking
  and generative fusion, 2023.

\bibitem[Jimenez et~al.(2024)Jimenez, Yang, Wettig, Yao, Pei, Press, and
  Narasimhan]{jimenez2023swe}
Carlos~E Jimenez, John Yang, Alexander Wettig, Shunyu Yao, Kexin Pei, Ofir
  Press, and Karthik Narasimhan.
\newblock Swe-bench: Can language models resolve real-world github issues?
\newblock In \emph{The Twelfth International Conference on Learning
  Representations}, 2024.

\bibitem[Kambhampati(2024)]{kambhampati2024can}
Subbarao Kambhampati.
\newblock Can large language models reason and plan?, 2024.

\bibitem[Khan(2026)]{khan2026clinersi}
Ara Khan.
\newblock Recursive self-improvement for coding agents.
\newblock Cline Blog,
  \url{https://cline.bot/blog/recursive-self-improvement-for-coding-agents},
  2026.
\newblock Accessed 2026-08-03.

\bibitem[Lightman et~al.(2023)Lightman, Kosaraju, Burda, Edwards, Baker, Lee,
  Leike, Schulman, Sutskever, and Cobbe]{lightman2023verify}
Hunter Lightman, Vineet Kosaraju, Yura Burda, Harri Edwards, Bowen Baker, Teddy
  Lee, Jan Leike, John Schulman, Ilya Sutskever, and Karl Cobbe.
\newblock Let's verify step by step, 2023.

\bibitem[Madaan et~al.(2023)Madaan, Tandon, Gupta, Hallinan, Gao, Wiegreffe,
  Alon, Dziri, Prabhumoye, Yang, et~al.]{madaan2023self}
Aman Madaan, Niket Tandon, Prakhar Gupta, Skyler Hallinan, Luyu Gao, Sarah
  Wiegreffe, Uri Alon, Nouha Dziri, Shrimai Prabhumoye, Yiming Yang, et~al.
\newblock Self-refine: Iterative refinement with self-feedback, 2023.

\bibitem[Mason et~al.(1999)Mason, Baxter, Bartlett, and
  Frean]{mason2000boosting}
Llew Mason, Jonathan Baxter, Peter Bartlett, and Marcus Frean.
\newblock Boosting algorithms as gradient descent.
\newblock In \emph{Advances in Neural Information Processing Systems
  (NeurIPS)}, volume~12. MIT Press, 1999.

\bibitem[Mu{\~n}oz and Yuan(2025)]{munoz2025rttc}
J.~Pablo Mu{\~n}oz and Jinjie Yuan.
\newblock {RTTC}: Reward-guided collaborative test-time compute, 2025.

\bibitem[Nagappan and Ball(2005)]{nagappan2005churn}
Nachiappan Nagappan and Thomas Ball.
\newblock Use of relative code churn measures to predict system defect density.
\newblock In \emph{Proceedings of the 27th International Conference on Software
  Engineering (ICSE)}, pages 284--292, 2005.

\bibitem[{OpenAI}(2025)]{openai2025preparedness}
{OpenAI}.
\newblock Preparedness framework, version 2.
\newblock
  \url{https://cdn.openai.com/pdf/18a02b5d-6b67-4cec-ab64-68cdfbddebcd/preparedness-framework-v2.pdf},
  2025.
\newblock Last updated 2025-04-15. Accessed 2026-08-18.

\bibitem[Pansuriya et~al.(2026)Pansuriya, Ghorbani, Singh, and
  AlOmar]{pansuriya2026predicting}
Kartik~Ghanshyambhai Pansuriya, Ehsan Ghorbani, Deepak Singh, and Eman~Abdullah
  AlOmar.
\newblock Predicting acceptance and review effort in human and agent pull
  requests, 2026.

\bibitem[Papamarkou et~al.(2026)Papamarkou, Smirnov, Mazanov, Vazhentsev,
  Nakov, Baldwin, and Shelmanov]{papamarkou2026bayesian}
Theodore Papamarkou, Vladislav Smirnov, Viktor Mazanov, Artem Vazhentsev,
  Preslav Nakov, Timothy Baldwin, and Artem Shelmanov.
\newblock Bayesian control for coding agents, 2026.

\bibitem[Qian et~al.(2024)Qian, Liu, Liu, Chen, Dang, Li, Yang, Chen, Su, Cong,
  Xu, Li, Liu, and Sun]{qian2023chatdev}
Chen Qian, Wei Liu, Hongzhang Liu, Nuo Chen, Yufan Dang, Jiahao Li, Cheng Yang,
  Weize Chen, Yusheng Su, Xin Cong, Juyuan Xu, Dahai Li, Zhiyuan Liu, and
  Maosong Sun.
\newblock Chatdev: Communicative agents for software development.
\newblock In \emph{ACL 2024}, 2024.

\bibitem[Romera-Paredes et~al.(2024)Romera-Paredes, Barekatain, Novikov, Balog,
  Kumar, Dupont, Ruiz, Ellenberg, et~al.]{romeraparedes2024funsearch}
Bernardino Romera-Paredes, Mohammadamin Barekatain, Alexander Novikov, Matej
  Balog, M.~Pawan Kumar, Emilien Dupont, Francisco J.~R. Ruiz, Jordan~S.
  Ellenberg, et~al.
\newblock Mathematical discoveries from program search with large language
  models.
\newblock \emph{Nature}, 625:\penalty0 468--475, 2024.
\newblock \doi{10.1038/s41586-023-06924-6}.

\bibitem[Schapire(2013)]{schapire2013explaining}
Robert~E Schapire.
\newblock Explaining adaboost.
\newblock \emph{Empirical Inference}, pages 37--52, 2013.

\bibitem[Schapire and Singer(1999)]{schapire1999improved}
Robert~E Schapire and Yoram Singer.
\newblock Improved boosting algorithms using confidence-rated predictions.
\newblock \emph{Machine Learning}, 37\penalty0 (3):\penalty0 297--336, 1999.
\newblock \doi{10.1023/A:1007614523901}.

\bibitem[Schapire et~al.(1998)Schapire, Freund, Bartlett, and
  Lee]{schapire1998boosting}
Robert~E Schapire, Yoav Freund, Peter Bartlett, and Wee~Sun Lee.
\newblock Boosting the margin: A new explanation for the effectiveness of
  voting methods.
\newblock \emph{Annals of Statistics}, 26\penalty0 (5):\penalty0 1651--1686,
  1998.

\bibitem[Schmidhuber(2003)]{schmidhuber2003goedel}
J{\"u}rgen Schmidhuber.
\newblock {G\"odel} machines: Self-referential universal problem solvers making
  provably optimal self-improvements, 2003.

\bibitem[Shinn et~al.(2023)Shinn, Cassano, Berman, Gopinath, Narasimhan, and
  Yao]{shinn2023reflexion}
Noah Shinn, Federico Cassano, Edward Berman, Ashwin Gopinath, Karthik
  Narasimhan, and Shunyu Yao.
\newblock Reflexion: Language agents with verbal reinforcement learning, 2023.

\bibitem[Skalse et~al.(2022)Skalse, Howe, Krasheninnikov, and
  Krueger]{skalse2022defining}
Joar Skalse, Nikolaus H.~R. Howe, Dmitrii Krasheninnikov, and David Krueger.
\newblock Defining and characterizing reward hacking.
\newblock In \emph{Advances in Neural Information Processing Systems
  (NeurIPS)}, 2022.

\bibitem[Snell et~al.(2024)Snell, Lee, Xu, and Kumar]{snell2024scaling}
Charlie Snell, Jaehoon Lee, Kelvin Xu, and Aviral Kumar.
\newblock Scaling {LLM} test-time compute optimally can be more effective than
  scaling model parameters, 2024.

\bibitem[Sun et~al.(2020)Sun, Wang, Liu, Miller, Efros, and
  Hardt]{sun2019testtime}
Yu~Sun, Xiaolong Wang, Zhuang Liu, John Miller, Alexei~A. Efros, and Moritz
  Hardt.
\newblock Test-time training with self-supervision for generalization under
  distribution shifts.
\newblock In \emph{International Conference on Machine Learning (ICML)}, 2020.

\bibitem[Sun et~al.(2025)Sun, Li, Dalal, Xu, Vikram, Zhang, Dubois, Chen, Wang,
  Koyejo, Hashimoto, and Guestrin]{sun2024learning}
Yu~Sun, Xinhao Li, Karan Dalal, Jiarui Xu, Arjun Vikram, Genghan Zhang, Yann
  Dubois, Xinlei Chen, Xiaolong Wang, Sanmi Koyejo, Tatsunori Hashimoto, and
  Carlos Guestrin.
\newblock Learning to (learn at test time): {RNNs} with expressive hidden
  states.
\newblock In \emph{International Conference on Machine Learning (ICML)}, 2025.

\bibitem[Tsochantaridis et~al.(2005)Tsochantaridis, Joachims, Hofmann, and
  Altun]{tsochantaridis2005large}
Ioannis Tsochantaridis, Thorsten Joachims, Thomas Hofmann, and Yasemin Altun.
\newblock Large margin methods for structured and interdependent output
  variables.
\newblock \emph{Journal of Machine Learning Research}, 6:\penalty0 1453--1484,
  2005.

\bibitem[VanderWeele and Ding(2017)]{vanderweele2017evalue}
Tyler~J. VanderWeele and Peng Ding.
\newblock Sensitivity analysis in observational research: Introducing the
  e-value.
\newblock \emph{Annals of Internal Medicine}, 167\penalty0 (4):\penalty0
  268--274, 2017.
\newblock \doi{10.7326/M16-2607}.

\bibitem[Wang et~al.(2023)Wang, Xie, Jiang, Mandlekar, Xiao, Zhu, Fan, and
  Anandkumar]{wang2023voyager}
Guanzhi Wang, Yuqi Xie, Yunfan Jiang, Ajay Mandlekar, Chaowei Xiao, Yuke Zhu,
  Linxi Fan, and Anima Anandkumar.
\newblock Voyager: An open-ended embodied agent with large language models,
  2023.

\bibitem[Wang et~al.(2024)Wang, Wang, Athiwaratkun, Zhang, and
  Zou]{wang2024mixtureofagents}
Junlin Wang, Jue Wang, Ben Athiwaratkun, Ce~Zhang, and James Zou.
\newblock Mixture-of-agents enhances large language model capabilities, 2024.

\bibitem[Wang et~al.(2022)Wang, Wei, Schuurmans, Le, Chi, Narang, Chowdhery,
  and Zhou]{wang2023selfconsistency}
Xuezhi Wang, Jason Wei, Dale Schuurmans, Quoc Le, Ed~Chi, Sharan Narang,
  Aakanksha Chowdhery, and Denny Zhou.
\newblock Self-consistency improves chain of thought reasoning in language
  models, 2022.

\bibitem[Welleck et~al.(2023)Welleck, Lu, West, Brahman, Shen, Khashabi, and
  Choi]{welleck2023generating}
Sean Welleck, Ximing Lu, Peter West, Faeze Brahman, Tianxiao Shen, Daniel
  Khashabi, and Yejin Choi.
\newblock Generating sequences by learning to self-correct, 2023.

\bibitem[Wu et~al.(2024)Wu, Sun, Li, Welleck, and Yang]{wu2024inference}
Yangzhen Wu, Zhiqing Sun, Shanda Li, Sean Welleck, and Yiming Yang.
\newblock Inference scaling laws: An empirical analysis of compute-optimal
  inference for problem-solving with language models, 2024.

\bibitem[Xue and Yang(2026)]{xue2026rethinking}
Hui Xue and Fan Yang.
\newblock Rethinking self-evolving agents: Do we still need prescribed
  optimization pipelines?, 2026.

\bibitem[Yang et~al.(2024)Yang, Jimenez, Wettig, Lieret, Yao, Narasimhan, and
  Press]{yang2024sweagent}
John Yang, Carlos~E Jimenez, Alexander Wettig, Kilian Lieret, Shunyu Yao,
  Karthik Narasimhan, and Ofir Press.
\newblock {SWE-agent}: Agent-computer interfaces enable automated software
  engineering.
\newblock In \emph{Advances in Neural Information Processing Systems
  (NeurIPS)}, 2024.

\bibitem[Zelikman et~al.(2023)Zelikman, Lorch, Mackey, and
  Kalai]{zelikman2023stop}
Eric Zelikman, Eliana Lorch, Lester Mackey, and Adam~Tauman Kalai.
\newblock Self-taught optimizer ({STOP}): Recursively self-improving code
  generation, 2023.

\bibitem[Zhang et~al.(2025{\natexlab{a}})Zhang, Hu, Lu, Lange, and
  Clune]{zhang2025darwin}
Jenny Zhang, Shengran Hu, Cong Lu, Robert Lange, and Jeff Clune.
\newblock Darwin {G\"odel} machine: Open-ended evolution of self-improving
  agents, 2025{\natexlab{a}}.

\bibitem[Zhang et~al.(2025{\natexlab{b}})Zhang, Lyu, Sun, Wang, Zhang, Hua, Wu,
  Guo, Wang, Muennighoff, King, Liu, and Ma]{zhang2025testtime}
Qiyuan Zhang, Fuyuan Lyu, Zexu Sun, Lei Wang, Weixu Zhang, Wenyue Hua, Haolun
  Wu, Zhihan Guo, Yufei Wang, Niklas Muennighoff, Irwin King, Xue Liu, and Chen
  Ma.
\newblock A survey on test-time scaling in large language models: What, how,
  where, and how well?, 2025{\natexlab{b}}.

\bibitem[Zhu et~al.(2009)Zhu, Zou, Rosset, and Hastie]{zhu2009multi}
Ji~Zhu, Hui Zou, Saharon Rosset, and Trevor Hastie.
\newblock Multi-class adaboost.
\newblock \emph{Statistics and Its Interface}, 2\penalty0 (3):\penalty0
  349--360, 2009.

\bibitem[Zweiger et~al.(2025)Zweiger, Pari, Guo, Aky{\"u}rek, Kim, and
  Agrawal]{zweiger2025seal}
Adam Zweiger, Jyothish Pari, Han Guo, Ekin Aky{\"u}rek, Yoon Kim, and Pulkit
  Agrawal.
\newblock Self-adapting language models.
\newblock In \emph{Advances in Neural Information Processing Systems
  (NeurIPS)}, 2025.

\end{thebibliography}

\clearpage

\appendix

\section{The Boosting/Coding Mapping, and the Protocol}
\label{app:mapping}

Table~\ref{tab:mapping} gives the term-by-term mapping that \S\ref{sec:framework} states in prose; Algorithm~\ref{alg:boost} below is the Protocol whose exact form \S\ref{sec:framework} reads as a build specification rather than a description of any deployed loop.

\begin{table}[h]
\centering
\small
\begin{tabular}{lll}
\toprule
\textbf{Boosting concept} & \textbf{Agentic coding analog} & \textbf{Concrete example} \\
\midrule
Input $X$ & Specification & Prompts, GitHub tickets \\
Target $Y$ & Optimal Code & The optimal ``best'' code for $X$ \\
Current State $F_{t-1}(X)$ & Current Draft & The code generated by previous agents \\
Residual $(Y - F_{t-1}(X))$ & Code Edit & A git diff or patch \\
Weak learner $h_t$ & Agent Patch Generator & LLM call proposing an edit \\
Loss Function $L(Y, \hat{Y})$ & Structured Distance & Token edit distance, AST distance \\
Distribution $D_t$ & Spec Weighting & Focus on failing specifications \\
\bottomrule
\end{tabular}
\caption{Mapping between gradient boosting and agentic coding.}
\label{tab:mapping}
\end{table}

\subsection{Algorithm: Agentic Boosting Protocol}
\label{app:algorithm}

\begin{algorithm}[h]
\caption{Agentic Boosting Protocol for Specifications}
\label{alg:boost}
\begin{algorithmic}[1]
\REQUIRE Spec-Code pairs $\{(x_i, y_i)\}_{i=1}^N$, number of rounds $T$
\STATE Initialize $w_i \leftarrow 1/N$ for all $i$
\FOR{$t = 1$ to $T$}
    \STATE Select weak coding agent $h_t$ minimizing $\epsilon_t = \sum_i w_i \mathbf{1}[h_t(x_i) \neq y_i]$
    \IF{$\epsilon_t \geq 1/2$}
        \STATE Stop: no weak learner available
    \ENDIF
    \STATE Set $\alpha_t \leftarrow \frac{1}{2}\ln\frac{1 - \epsilon_t}{\epsilon_t}$
    \STATE Update $w_i \leftarrow w_i \cdot \exp(-\alpha_t \, y_i \, h_t(x_i))$ for all $i$
    \STATE Normalize $w_i \leftarrow w_i / \sum_j w_j$
\ENDFOR
\RETURN artifact $H(x) = (R_T \circ \dots \circ R_1)(z_0)$ (sequential patch composition), with classification confidence given by the additive margin $f(x) = \sum_t \alpha_t h_t(x)$
\end{algorithmic}
\end{algorithm}

\section{Full Proofs of the Theoretical Results}
\label{app:proofs}

\emph{Note on formalization scope and motivation:} the algebraic steps in this section have been machine-checked in Lean~4 (Mathlib) as a quality-control layer against AI-hallucinated derivations, since this manuscript was substantially drafted with AI assistance. The result-to-declaration map is in Appendix~\ref{app:lean}; what that index cannot do is certify the modelling assumptions, which Remark~\ref{rem:lean-honesty} states plainly.

\subsection{Statements of the Transferred Classical Results}
\label{app:classical-statements}

Summarized in \S\ref{sec:theory} and stated formally here, next to their proofs; the main text relies on none of them. What they require is not a positive edge---\S\ref{sec:framework} withdraws that reading---but a \emph{retained, reweighted ensemble}: a distribution $D_t$ maintained over specifications, and all $T$ hypotheses kept and combined by weighted vote rather than overwritten. No deployed refinement loop does either. Three transfer verbatim under this reading and none is new: the training-error product bound (Thm.~\ref{thm:training-error}) and its $e^{-2\gamma^2T}$ corollary under a uniform edge (Cor.~\ref{cor:convergence}); equivalence to forward stagewise additive modeling under exponential loss (Thm.~\ref{thm:fsam}); and \citet{schapire1998boosting}'s margin bound, making generalization depend on weak-learner complexity rather than round count (Thm.~\ref{thm:generalization})---though for an LLM patch generator the VC-dimension is unknown, examples need not be i.i.d., and a self-expanding scaffold breaks the fixed-$\Hyp$ premise outright.

\begin{definition}[Strong Coding System]
\label{def:strong-agent}
A \emph{strong coding system} is, at the \emph{artifact level}, the sequential composition $F_T(x)=(R_T\circ\cdots\circ R_1)(F_0(x))$ of patches (order-dependent, non-linear); at the \emph{score level}, the additive margin $f(x)=\sum_{t=1}^T\alpha_t h_t(x)$, $\alpha_t\ge0$. It is this margin, not the code bytes, that inherits the AdaBoost machinery---which is exactly what a composing loop does not build.
\end{definition}

\begin{definition}[Agentic Boosting Protocol]
\label{def:protocol}
The \emph{Agentic Boosting Protocol} (Algorithm~\ref{alg:boost}) maintains a distribution $D_t$ over $\{(x_i,y_i)\}$; at round $t$ it selects $h_t$ minimizing $\epsilon_t=\Pr_{D_t}[h_t(x)\ne y]$, sets $\alpha_t=\tfrac12\ln\tfrac{1-\epsilon_t}{\epsilon_t}$, and updates $D_{t+1}(i)\propto D_t(i)\exp(-\alpha_t y_i h_t(x_i))$.
\end{definition}

\begin{theorem}[Training Error of Agentic Boosting]
\label{thm:training-error}
After $T$ rounds of the Agentic Boosting Protocol, the training error of the strong coding system $H$ is bounded by $\frac{1}{N}\sum_i\mathbf{1}[H(x_i)\ne y_i]\le\prod_{t=1}^T Z_t$, where $Z_t=2\sqrt{\epsilon_t(1-\epsilon_t)}$.
\end{theorem}

\begin{corollary}[Exponential Decay of Training Error]
\label{cor:convergence}
If each weak coding agent has edge $\gamma_t\ge\gamma>0$ for all rounds $t$, then $\mathrm{err}_{\mathrm{train}}(H)\le\exp(-2\gamma^2 T)$.
\end{corollary}

\begin{theorem}[Generalization Error of Agentic Boosting]
\label{thm:generalization}
Let $\Hyp$ be the hypothesis class of weak coding agents with VC-dimension $d$, and let the margin be measured on the \emph{normalized} combination $f/\sum_t\alpha_t$, so $yf(x)\in[-1,1]$ and the threshold is $\theta\in(0,1]$. Running the Protocol for $T$ rounds on $N$ i.i.d.\ samples, with probability at least $1-\delta$: $\Pr_{\D}[H(x)\ne y]\le\Pr_{\mathrm{train}}[yf(x)\le\theta]+O\!\left(\sqrt{\tfrac{1}{N}\left(\tfrac{d\ln^2(N/d)}{\theta^2}+\ln\tfrac1\delta\right)}\right)$, where the second term is \emph{independent of $T$} \citep{schapire1998boosting}.
\end{theorem}

For a real LLM patch generator, $d$ is typically unknown, examples need not be i.i.d., and a self-expanding scaffold can violate the fixed-$\Hyp$ assumption (\S\ref{app:limitations}).

\begin{theorem}[Boosting as Greedy Coordinate Descent]
\label{thm:fsam}
The Agentic Boosting Protocol is equivalent to forward stagewise additive modeling with exponential loss $L(y,f(x))=e^{-yf(x)}$: at each round $t$ it greedily selects the weak agent and weight that most reduce $J(\alpha,h)=\sum_i\exp(-y_i(f_{t-1}(x_i)+\alpha h(x_i)))$. This is \citet{friedman2000additive}'s statistical view of boosting, read over specifications.
\end{theorem}

\subsection{Proofs}
\label{app:proof-bodies}

Grouped by family: first the four transferred classical results, which we rely on nowhere; then the refinement game and the dichotomy, which carry the paper; then the frozen-weight ceiling the dichotomy instantiates.

\paragraph{Transferred classical results.}

\begin{proof}[Proof of Theorem~\ref{thm:training-error}]
Let $f(x_i) = \sum_{t=1}^{T} \alpha_t h_t(x_i)$ be the unnormalized weighted sum. The training error can be bounded via the exponential loss:
\[
\frac{1}{N}\sum_{i=1}^N \mathbf{1}[H(x_i) \neq y_i] \leq \frac{1}{N}\sum_{i=1}^N \exp(-y_i f(x_i)).
\]
We prove this by induction on $T$. For $T = 0$, $f \equiv 0$, so the bound gives $1$, which is trivially satisfied.

\textbf{Inductive step.} Assume the bound holds after $T - 1$ rounds with weights $D_T(i) = \frac{\exp(-y_i \sum_{t=1}^{T-1} \alpha_t h_t(x_i))}{N \prod_{s=1}^{T-1} Z_s}$. After round $T$, we have:
\begin{align*}
\frac{1}{N}\sum_{i=1}^N \exp\!\left(-y_i \sum_{t=1}^{T} \alpha_t h_t(x_i)\right)
&= \sum_{i=1}^N D_T(i) \cdot \prod_{t=1}^{T-1} Z_t \cdot \exp(-y_i \alpha_T h_T(x_i)) \\
&= \prod_{t=1}^{T-1} Z_t \sum_{i=1}^N D_T(i) \exp(-y_i \alpha_T h_T(x_i)).
\end{align*}
We split the sum into correctly and incorrectly classified samples:
\begin{align*}
\sum_{i=1}^N D_T(i) \exp(-y_i \alpha_T h_T(x_i))
&= \sum_{i: h_T(x_i) = y_i} D_T(i) \, e^{-\alpha_T} + \sum_{i: h_T(x_i) \neq y_i} D_T(i) \, e^{\alpha_T} \\
&= (1 - \epsilon_T) e^{-\alpha_T} + \epsilon_T \, e^{\alpha_T}.
\end{align*}
Substituting $\alpha_T = \frac{1}{2}\ln\frac{1-\epsilon_T}{\epsilon_T}$:
\begin{align*}
(1 - \epsilon_T) e^{-\alpha_T} + \epsilon_T \, e^{\alpha_T}
&= (1 - \epsilon_T) \sqrt{\frac{\epsilon_T}{1 - \epsilon_T}} + \epsilon_T \sqrt{\frac{1 - \epsilon_T}{\epsilon_T}} \\
&= \sqrt{\epsilon_T(1-\epsilon_T)} + \sqrt{\epsilon_T(1-\epsilon_T)} \\
&= 2\sqrt{\epsilon_T(1 - \epsilon_T)} = Z_T.
\end{align*}
Combining:
\[
\frac{1}{N}\sum_{i=1}^N \exp(-y_i f(x_i)) = \prod_{t=1}^{T-1} Z_t \cdot Z_T = \prod_{t=1}^T Z_t.
\]
Since $\mathbf{1}[H(x_i) \neq y_i] \leq \exp(-y_i f(x_i))$ (because $\mathrm{sign}(f(x_i)) \neq y_i \implies y_i f(x_i) \leq 0 \implies \exp(-y_i f(x_i)) \geq 1$), the training error is bounded by $\prod_{t=1}^T Z_t$.
\end{proof}

\begin{proof}[Proof of Corollary~\ref{cor:convergence}]
From Theorem~\ref{thm:training-error}, $\mathrm{err}_{\mathrm{train}} \leq \prod_{t=1}^T Z_t$ where $Z_t = 2\sqrt{\epsilon_t(1-\epsilon_t)}$. Since $\epsilon_t \leq \frac{1}{2} - \gamma$, we have:
\begin{align*}
Z_t &= 2\sqrt{\epsilon_t(1 - \epsilon_t)} \\
&\leq 2\sqrt{\left(\tfrac{1}{2} - \gamma\right)\left(\tfrac{1}{2} + \gamma\right)} \\
&= 2\sqrt{\tfrac{1}{4} - \gamma^2} \\
&= \sqrt{1 - 4\gamma^2}.
\end{align*}
Using the inequality $\sqrt{1 - u} \leq e^{-u/2}$ for $u \in [0,1)$:
\[
\prod_{t=1}^T Z_t \leq \left(\sqrt{1 - 4\gamma^2}\right)^T \leq \left(e^{-2\gamma^2}\right)^T = e^{-2\gamma^2 T}.
\]
As $T \to \infty$, the training error converges to zero exponentially fast.
\end{proof}

\begin{remark}[The edge condition is a capability threshold, not a weak requirement]
\label{rem:capability-threshold}
The hypothesis $\gamma_t \geq \gamma > 0$ is substantially more demanding in code space than the analogous condition in ordinary binary classification, because the space is overwhelmingly hostile: almost all edits to a working program break it, so an error below $\tfrac{1}{2}$ already presupposes a considerable baseline of program understanding. That baseline is what \citet{kambhampati2024can} argues single-model reasoning does not reliably supply, which is one reason we treat a positive edge as a hypothesis to be measured rather than granted. That is a statement about the \emph{edit space}, not about the agents we measure; read as a claim about them it would predict the wrong thing.

Deployed patches are not destructive: on the diagnostic arm at most \statPBRegressMax\ of rounds lower quality, and \statPBHaikuFlat/\statPBSonnetFlat\ of trajectories never change it at all (Appendix~\ref{app:edge-decomposition}). The measured shortfall is idempotence, not damage, which is why a negative pooled $\gamma$ here cannot be read as these agents being anti-correlated with correctness. Corollary~\ref{cor:convergence} therefore guarantees exponential convergence only \emph{once} the base agents cross this capability threshold. Agents with weighted error $\epsilon_t \geq \tfrac{1}{2}$ (non-positive edge) receive weight $\alpha_t = \tfrac{1}{2}\ln\tfrac{1-\epsilon_t}{\epsilon_t} \leq 0$ and are gated out by the stopping rule of Algorithm~\ref{alg:boost}; the protocol thus performs a form of capability warm-up, admitting an agent to the ensemble only after it becomes a genuine weak learner.
\end{remark}

\begin{proof}[Proof sketch of Theorem~\ref{thm:generalization}]
Normalizing by $\sum_t\alpha_t$ makes $H$ a convex combination of $T$ weak learners from $\mathcal{H}$, which is the form the margin bound requires; $H=\mathrm{sign}(f)$ is unchanged by the rescaling. By the margin bound of \citet{schapire1998boosting}, for any $\theta \in (0,1]$ and with probability $1 - \delta$:
\[
\Pr_{\mathcal{D}}[y \, f(x) \leq 0] \leq \Pr_{\mathrm{train}}[y \, f(x) \leq \theta] + O\!\left(\sqrt{\frac{1}{N}\left(\frac{d \ln^2(N/d)}{\theta^2} + \ln\frac{1}{\delta}\right)}\right).
\]
The right-hand side does not depend on $T$; it is controlled by the VC-dimension $d$ of the weak-learner class and the margin threshold $\theta$, and it degrades as $\theta^{-2}$, so a bound quoted with a weaker $\theta$-dependence would overstate it. As boosting increases margins on training data, the term $\Pr_{\mathrm{train}}[y\,f(x)\leq\theta]$ decreases, giving better generalization. Balancing the two terms in $\theta$ would give an optimal stopping time of order $\frac{1}{\gamma^2}\ln\frac{N}{d}$; we do not derive it, and Cor.~\ref{cor:convergence} alone does not supply it---that corollary bounds the \emph{zero-one} training error, whereas the balance needs the $\theta$-margin training bound, a different inequality \citep{schapire1998boosting}. Nothing in this paper uses the stopping time.
\end{proof}

\begin{proof}[Proof of Theorem~\ref{thm:fsam}]
Fix round $t$ and let $w_i^{(t)} = \exp(-y_i f_{t-1}(x_i))$ be the unnormalized weights. The objective is:
\[
J(\alpha, h) = \sum_{i=1}^N w_i^{(t)} \exp(-y_i \alpha h(x_i)).
\]
Splitting by correctness:
\[
J(\alpha, h) = e^{-\alpha} \sum_{i: h(x_i) = y_i} w_i^{(t)} + e^{\alpha} \sum_{i: h(x_i) \neq y_i} w_i^{(t)}.
\]
In terms of the weighted error $\epsilon = \frac{\sum_{i: h \neq y} w_i}{\sum_i w_i}$:
\[
J(\alpha, h) = \left(\sum_i w_i\right) \left[ (1 - \epsilon) e^{-\alpha} + \epsilon \, e^{\alpha} \right].
\]
The factor $\sum_i w_i$ does not depend on $\alpha$ or $h$. Minimizing over $\alpha$ for fixed $h$ (setting $\frac{\partial J}{\partial \alpha} = 0$):
\[
-(1-\epsilon) e^{-\alpha} + \epsilon \, e^{\alpha} = 0 \implies \alpha^* = \frac{1}{2}\ln\frac{1 - \epsilon}{\epsilon},
\]
which is exactly the AdaBoost weight update rule. Minimizing over $h$ for fixed $\alpha > 0$ amounts to minimizing $\epsilon$, i.e., choosing the agent with the lowest weighted error. This establishes the equivalence between the Agentic Boosting Protocol and greedy coordinate descent on the exponential loss.
\end{proof}

\paragraph{The refinement game and the dichotomy.} These are the results the paper relies on, and they mention AdaBoost nowhere.

\begin{proof}[Proof of Theorem~\ref{thm:refinement}]
By telescoping:
\[
\mathbb{E}[V(z_T)] - V(z_0) = \sum_{t=1}^T \left(\mathbb{E}[V(z_t)] - \mathbb{E}[V(z_{t-1})]\right) \geq \sum_{t=1}^T \eta_t.
\]
Since $V \leq 1$, we have $\mathbb{E}[V(z_T)] \leq 1$, so $\sum_{t=1}^T \eta_t \leq 1 - V(z_0)$ for all $T$, which implies $\sum_{t=1}^\infty \eta_t \leq 1 - V(z_0)$.

\textbf{Connection to boosting.} The quality function $V$ plays the role of the margin $y f(x)$, the quantity boosting drives upward. The refinement property is the analog of the weak learning condition (Definition~\ref{def:weak-agent}): each agent must improve quality by at least $\eta_t$, just as each weak learner must have edge at least $\gamma_t$. The saturation bound is the analog of Corollary~\ref{cor:convergence} only in its \emph{consequence}---the number of useful rounds is finite---and not in its form: Cor.~\ref{cor:convergence} is a rate, whereas $\sum_t\eta_t\le1-V(x_0)$ bounds the total improvement without asserting how it is distributed over rounds. No rate is claimed here (\S\ref{sec:intro}).
\end{proof}

The high-probability form of the refinement game trades the improvement property's expectation for a per-round tail bound. The main text relies on it nowhere.

\begin{corollary}[High-Probability Refinement]
\label{cor:probabilistic-refinement}
In the setting of Thm.~\ref{thm:refinement}, suppose instead $\Pr[V(z_t)-V(z_{t-1})\ge\eta_t]\ge1-p_t$ for each round. Then $\Pr\!\left[V(z_T)\ge V(z_0)+\sum_{t=1}^T\eta_t\right]\ge1-\sum_{t=1}^T p_t$.
\end{corollary}

\begin{proof}[Proof of Corollary~\ref{cor:probabilistic-refinement}]
Let $E_t$ be the event that $V(z_t) - V(z_{t-1}) \geq \eta_t$. By assumption, $\Pr[E_t] \geq 1 - p_t$. Using the union bound, the probability that all events $E_1, \dots, E_T$ occur simultaneously is at least $1 - \sum_{t=1}^T p_t$. When all $E_t$ occur, the telescoping sum yields $V(z_T) - V(z_0) = \sum_{t=1}^T (V(z_t) - V(z_{t-1})) \geq \sum_{t=1}^T \eta_t$. This relaxes the strict monotone requirement to a high-probability guarantee, more accurately reflecting stochastic agentic code edits.
\end{proof}

\begin{proof}[Proof of Proposition~\ref{prop:dichotomy}]
Part (i) generalizes the saturation bound of Theorem~\ref{thm:refinement} from the ceiling $1$ to an arbitrary $B$: by monotonicity, $\mathbb{E}[V(z_T)] \leq B$, and combined with $\mathbb{E}[V(z_T)] \geq \mathbb{E}[V(z_0)] + \sum_{t=1}^T \eta_t$ we obtain $\sum_{t=1}^T \eta_t \leq B - \mathbb{E}[V(z_0)]$ for every $T$. For part (ii), the telescoping lower bound of Theorem~\ref{thm:refinement} uses only per-round improvement and gives $\mathbb{E}[V(z_T)] \geq \mathbb{E}[V(z_0)] + \sum_{t=1}^T \eta_t \geq \mathbb{E}[V(z_0)] + Tc$. By the Archimedean property, for any $B$ we may choose $T > (B - \mathbb{E}[V(z_0)])/c$, giving $\mathbb{E}[V(z_T)] > B$.
\end{proof}

Part (i) is the \emph{stationary} regime: when the class of reachable patches is fixed, a bounded quality metric mathematically enforces a plateau. Part (ii) is the \emph{non-stationary} regime: sustaining a positive edge indefinitely requires the reachable class itself to keep growing, which is precisely what an unbounded ``intelligence explosion'' would demand.

\paragraph{The frozen-weight ceiling.}

\begin{proof}[Proof of Corollary~\ref{cor:frozen-ceiling}]
Part (i) applies Proposition~\ref{prop:dichotomy}(i) with the fixed bound $B = V^*(W)$, whose existence is assumption (M2); monotone convergence of the bounded, non-decreasing sequence $\mathbb{E}[V(x_t)]$ gives the limit $L \leq V^*(W)$. Part (ii) is the contrapositive of Proposition~\ref{prop:dichotomy}(ii): if a fixed ceiling $V^*(W)$ bounds $V$, then no uniform edge $\eta_t \geq c > 0$ can be sustained, so $\mathbb{E}[V(x_T)]$ cannot exceed every bound. Part~(i)'s \emph{saturation bound}---$\sum_{t<T}\eta_t \le V^*(W) - \mathbb{E}[V(x_0)]$ for every $T$---is machine-checked in Lean as \texttt{cor\_frozen\_weight\_ceiling} (a direct instantiation of \texttt{prop8\_saturation\_forces\_decay} with $B := V^*(W)$); the step from there to the limit $L$ is the monotone-convergence argument above and is not formalized, since the Lean statement does not assume $\eta_t\ge0$. Part (ii) is backed by \texttt{prop8\_sustained\_edge\_diverges}.
\end{proof}

What the corollary buys---that frozen weights secure the eventual ceiling but not stationary behaviour on the way to it, since (M1) lets $\Hyp_t(C)$ expand while $W$ is fixed---is the paper's governance claim, and \S\ref{app:rsi} develops it with the mechanisms that move a system between the two regimes.

\begin{remark}[Lean verifies the algebra, not the modeling assumptions]
\label{rem:lean-honesty}
Theorem~\ref{thm:refinement} and Proposition~\ref{prop:dichotomy} are elementary results: part~(i) is a bounded-monotone-sequence argument; part~(ii) is the Archimedean property. Lean verifies the formal algebra of these proofs---it does \emph{not} certify the modeling assumptions (the per-round improvement property, (M1), (M2)) that connect the abstract game to real agentic systems. Those assumptions are stated as explicit hypotheses and remain empirical claims.
\end{remark}

\section{Orchestration, Voting, and the Diversity a Vote Needs}
\label{app:orchestration}

The dismissal of ``elaborate multi-agent orchestration'' in the main text and \S\ref{app:guidance} (item~3) is narrower than it may read: it says orchestration cannot raise any individual worker's ceiling $V^*(W_i)$---a stronger model still wins by clearing tasks no arrangement of weaker models can reach. It is not an argument that hierarchical or orchestrator-worker \emph{topologies} fall outside the framework; the multi-agent code-generation systems surveyed in \S\ref{app:related} already route work across specialized agents, and the paper's own $\mathcal{S}=(W,C)$ framing (\S\ref{app:rsi}) already lists orchestration as part of the mutable scaffold $C$. We make that instance explicit.

\paragraph{Setup.} Replace single-agent sequential refinement with $k$ frozen workers under scaffolds $C_1,\dots,C_k$ (distinct prompts, tools, or specializations, possibly sharing weights $W$), coordinated by an orchestrator that, each round, selects among their outputs. Let $\mathcal{H}_t(C_i)$ be worker $i$'s reachable class (Cor.~\ref{cor:frozen-ceiling}'s framing), and define the orchestrator's reachable class under best-of-$k$ selection as $\mathcal{H}_t^{\mathrm{orch}} := \bigcup_{i=1}^k \mathcal{H}_t(C_i)$.

\subsection{Hard Selection Reaches the Best Worker, Exactly}

\begin{proof}[Proof of Prop.~\ref{prop:orchestration} (stated in \S\ref{sec:theory})]
For any instance $x$, the orchestrator's best-of-$k$ patch attains value $\max_i V(h_i(x))$ over each worker's best reachable patch, so $V^*_{\mathrm{orch}} = \sup_x \max_i V(h_i(x))$. Sup and finite max commute unconditionally: $\sup_x \max_i V(h_i(x)) = \max_i \sup_x V(h_i(x)) = \max_i V^*(W,C_i)$, with no side condition on the $\mathcal{H}_t(C_i)$'s. (One might expect a strict inequality whenever no single worker's class contains the union, on the intuition that ``the winning worker varies with $x$'' should cost something; it does not---two disjoint singleton classes $\mathcal{H}_t(C_1)=\{a\}$, $\mathcal{H}_t(C_2)=\{b\}$ with $V(a)=V(b)$ already give a counterexample, since neither singleton contains $\{a,b\}$ yet the ceilings tie exactly. Equality is the only case; there is no strict regime for hard selection.) Taking the union of $k$ scaffold-indexed classes is itself an instance of assumption~(M1): adding workers or routing among them is a scaffold rewrite (orchestration is already named as part of $C$, \S\ref{app:rsi}) that can expand $\mathcal{H}_t(C)$ without touching any $W_i$. Orchestrator-worker architectures do not escape Cor.~\ref{cor:frozen-ceiling}; best-of-$k$ selection instantiates its escape clause exactly, never more. This equality is machine-checked in Lean as \texttt{prop15\_tight\_ceiling\_orchestration}; a separate, weaker fact (\texttt{prop15\_orchestration\_ceiling}, \texttt{prop15\_orchestration\_dominates}) covers the case where each $V^*(W,C_i)$ is only an upper-bound hypothesis rather than a tight supremum, giving $V^*_{\mathrm{orch}} \ge \max_i V^*(W,C_i)$ in that weaker setting.
\end{proof}

\paragraph{Relation to the main-text guidance.} This sharpens, rather than merely complements, ``reach for capability before orchestration'' (\S\ref{app:guidance}, item~3): that guidance says orchestration cannot raise any $V^*(W,C_i)$ itself, and Prop.~\ref{prop:orchestration} now shows best-of-$k$ orchestration cannot exceed the best individual worker \emph{either}---the equality is exact, not merely an upper bound. Routing among a fixed set of frozen workers is bounded by, never exempt from, the frozen-weight ceiling; the only lever a scaffold rewrite of this kind provides is selecting which worker's already-existing ceiling to realize, an instance of (M1)'s class-expansion mechanism, not a mechanism special to multi-agent systems.

\subsection{Weighted Voting, and the Diversity It Requires}

\paragraph{Weighted aggregation: a partial answer.} Best-of-$k$ hard selection reduces cleanly to a union of reachable classes, and Prop.~\ref{prop:orchestration} shows that union buys \emph{exactly} the best individual worker's ceiling---no more. The natural place to look for genuine synergy is \emph{weighted or soft aggregation}: an orchestrator that votes or blends across workers rather than selecting one is structurally closer to AdaBoost's margin $f=\sum_t\alpha_t h_t$ (\S\ref{sec:framework}). Unlike sequential patch refinement, this mapping is not merely boosting-\emph{inspired}---it is boosting-\emph{exact}: workers responding to the \emph{same} instance stand beside each other to be voted on rather than composing (\S\ref{app:decomposition}'s objection to treating decomposition as an implicit vote does not apply here, since nothing is generated on the fly). Re-reading the index $t$ in Algorithm~\ref{alg:boost} as ``worker'' rather than ``round'' costs nothing---Thm.~\ref{thm:training-error} and Cor.~\ref{cor:convergence} transfer verbatim.

\begin{proposition}[Weighted Aggregation Beats the Weak-Learning Baseline]
\label{prop:soft-aggregation}
Let $k$ frozen workers be combined via the AdaBoost weighted vote rather than best-of-$k$ hard selection, incorporated in sequence $h_0,\dots,h_{k-1}$, each maintaining weighted error $\mathrm{err}_t \le \tfrac12-\gamma$ against the reweighted distribution current when it is incorporated, for some fixed $\gamma \in (0,\tfrac12)$. Then there exists $k_0$ such that for all $k \ge k_0$, the ensemble's zero-one training error is strictly below $\tfrac12-\gamma$---the weak-learning threshold any single worker is only guaranteed to satisfy, not guaranteed to beat.
\end{proposition}

\begin{proof}[Proof sketch]
The zero-one training error is at most $\exp(-2\gamma^2 k)$ (Thm.~\ref{thm:training-error}, Cor.~\ref{cor:convergence}), which $\to 0$ as $k\to\infty$ while $\tfrac12-\gamma$ is a fixed positive constant; take $k_0$ large enough that $\exp(-2\gamma^2 k_0) < \tfrac12-\gamma$, which exists since $\exp(-2\gamma^2 k)$ is eventually smaller than any fixed positive threshold. Machine-checked in Lean as \texttt{cor\_soft\_aggregation\_beats\_weak\_baseline}, built on the pure-analysis fact \texttt{exists\_ensemble\_size\_beating\_weak\_threshold}.
\end{proof}

\paragraph{Width and depth together.} Prop.~\ref{prop:orchestration} and Prop.~\ref{prop:soft-aggregation} both concern a \emph{single} decision: $k$ workers, one selection or one vote. The refinement game (Thm.~\ref{thm:refinement}) concerns the opposite axis: one agent, $T$ sequential rounds. Neither addresses their composition---$k$ workers consulted on \emph{every} one of $T$ rounds---which is what an orchestrated loop actually runs. For hard selection the composition goes through cleanly.

\begin{proposition}[Multi-Round Best-of-$k$ Orchestration]
\label{prop:width-refinement}
Run the refinement game with $k$ frozen workers consulted every round, the orchestrator retaining whichever output has the highest quality, so that $V(R^{(i)}_t(x,a)) \le V(R^{\mathrm{orch}}_t(x,a))$ for every worker $i$ (such an $R^{\mathrm{orch}}$ always exists). If worker $i$ satisfies Thm.~\ref{thm:refinement}'s improvement property with edge $\eta_t^{(i)}$, then
\[
\mathbb{E}[V(z_T)] \;\ge\; V(z_0) + \sum_{t=1}^{T}\max_i \eta_t^{(i)} \;\ge\; V(z_0) + \sum_{t=1}^{T}\eta_t^{(j)} \quad\text{for every worker } j,
\]
and, since $V\le1$, the total improvement still saturates: $\sum_{t=1}^{T}\max_i \eta_t^{(i)} \le 1-V(z_0)$. The main text states the same result compactly (Prop.~\ref{prop:orchestration}).
\end{proposition}

\begin{proof}[Proof sketch]
Fix a round $t$ and let $i^\star$ attain $\max_i \eta_t^{(i)}$. Hard selection dominates worker $i^\star$ pointwise, so the orchestrated round inherits worker $i^\star$'s improvement and the orchestrated per-round edge is $\max_i \eta_t^{(i)}$; Thm.~\ref{thm:refinement}'s telescoping applies verbatim to that edge, and its saturation argument to the same sum. The middle inequality is $\eta_t^{(j)}\le\max_i \eta_t^{(i)}$ summed over rounds. Machine-checked in Lean as \texttt{thm\_width\_refinement\_game}, \texttt{cor\_width\_dominates\_best\_worker}, and \texttt{cor\_width\_saturation}, with \texttt{exists\_width\_orchestrator} discharging the non-vacuity of the domination hypothesis. One hypothesis beyond Thm.~\ref{thm:refinement}'s list is needed and is flagged as such in the formalization: monotonicity of the \emph{inner} action-expectation, not only the outer one, since the selection step compares two integrands under the same expectation.
\end{proof}

\paragraph{Width buys rate, not budget.} The two inequalities pull in opposite directions. The first says orchestration strictly helps per round: the cumulative guarantee tracks the \emph{best} worker available at each round and dominates every individual worker's depth-only guarantee---not merely their average, and without knowing in advance which worker will win. The second says this cannot compound: the total improvement budget $1-V(z_0)$ is fixed by the bounded metric, so consulting $k$ workers every round spends that budget faster rather than enlarging it. Depth$\times$width composition does not escape the frozen-weight ceiling (Cor.~\ref{cor:frozen-ceiling}); it reaches it sooner. This is the multi-round counterpart of Prop.~\ref{prop:orchestration}'s single-decision equality, and it extends ``reach for capability before orchestration'' (\S\ref{app:guidance}, item~3) into the time dimension: adding rounds to an orchestrated loop is subject to the same saturation as adding rounds to a single agent.

\paragraph{Quantifying the diversity a vote needs.} Prop.~\ref{prop:soft-aggregation}'s edge hypothesis is qualitative and, more awkwardly, is stated against the \emph{reweighted} distribution $D_t$, so it cannot be checked from primitive facts about the workers. Prop.~\ref{prop:diversity} (stated in \S\ref{sec:theory}) gives a deterministic route instead, carrying diversity in a single integer. One consequence it leaves implicit is the form a designer would actually use: if any $m$ bounds every $m_i$ and $m{<}k/2$, the vote has zero training error (\texttt{thm\_bounded\_overlap\_zero\_error}), so a pool can be certified from an upper bound on overlap without computing $m^\star$ exactly.

\begin{proof}[Proof sketch of Prop.~\ref{prop:diversity} (stated in \S\ref{sec:theory})]
Splitting the margin at point $i$ over the workers that get $i$ right and those that get it wrong gives the identity $y_i f(x_i) = k - 2m_i$. The zero-one loss is an average of nonnegative indicators, so it vanishes exactly when every indicator does, i.e.\ when $y_i f(x_i)>0$ for every $i$, i.e.\ when $k-2m_i>0$---the claimed equivalence (\texttt{thm\_overlap\_iff\_zero\_error}, with \texttt{cor\_tight\_overlap\_iff} for the $m^\star$ form), from which the sufficient direction above follows in one line. For sharpness, a single point at $2m_i{=}k$ has margin exactly $0$---a tie, which the zero-one loss counts as an error, as it must, since a tied vote carries no information---so $2m<k$ cannot be weakened to $2m\le k$. If \emph{every} point sits at $m_i{=}k/2$, every margin vanishes and the error is $1$; that uniform hypothesis is what \texttt{thm\_overlap\_boundary\_sharp} assumes, and the even split witnesses that the case is non-vacuous (\texttt{exists\_boundary\_overlap\_family}). The tight overlap reaching $k/2$ on its own gives only the first of the two, since points below the max keep positive margins. The unweighted vote is the existing margin $f=\sum_t\alpha_t h_t$ at $\alpha_t\equiv1$, so no separate notion of voting is introduced.
\end{proof}

\paragraph{Overlap is a real constraint, and counting alone buys nothing.} First, at the tight overlap the condition is \emph{exact}, not merely sufficient (Prop.~\ref{prop:diversity}), so it cannot be traded for a weaker one: it is a real constraint on orchestration design rather than a formality. Workers drawn from genuinely different model families, scaffolds, or prompts plausibly satisfy it, whereas the same model resampled does not---its failures concentrate on the same hard points, driving $m$ toward $k$. That second half is measurable on Part~B rather than merely plausible, and we measure it below. Second, dropping the condition and counting alone buys \emph{nothing}: double-counting failures and applying Markov gives only that the vote's zero-one training error is at most $2\bar{\epsilon}$, twice the average per-worker error under the uniform distribution (\texttt{cor\_overlap\_markov\_bound}), and for $\bar{\epsilon}<\tfrac12$ the bound $2\bar{\epsilon}$ is \emph{weaker} than $\bar{\epsilon}$---weaker than using a single worker. Nor could a counting-only bound do better: the adversarial configuration in which every worker fails on the same $\bar{\epsilon}N$ points genuinely has ensemble error $\bar{\epsilon}$. It is the overlap hypothesis, not individual accuracy, that does the work---the precise sense in which \emph{diversity} is what the vote needs.

\paragraph{Measuring the overlap: same-family workers fail together.} Prop.~\ref{prop:diversity} is checkable from primitive facts about the workers, which means it can be checked against data rather than asserted. Part~B supplies \statPBOverlapK\ workers on a shared task pool: the \statPBOverlapArms\ arms (two capability tiers $\times$ three feedback modes) $\times$ \statPBSeeds\ seeds, all evaluated on the same \statPBK\ SWE-bench instances. Reading a worker's $\pm1$ output as whether it solves a task---running-best quality reaching $\tau$, the same criterion used throughout \S\ref{sec:partB}---gives $m_i$ directly. The result is the extreme of the pessimistic case: the tight overlap is $m^\star{=}\statPBOverlapMStar{=}k$, not merely above the $k/2{=}\statPBOverlapHalfK$ threshold but at its ceiling, and it is $m^\star{=}k$ at \emph{every} round, not only the last. Some task is failed by all \statPBOverlapK\ workers; \statPBOverlapViolate\ of tasks have $2m_i\ge k$, which by Prop.~\ref{prop:diversity} \emph{is} the unweighted vote's zero-one error (95\%~CI \statPBOverlapVoteErrCI, cluster bootstrap over tasks as in \S\ref{sec:partB}). Fig.~\ref{fig:llm-overlap} shows the distribution: a mode of easy tasks no worker fails, a second mode straddling $k/2$, and a spike at $m_i{=}k$.

\paragraph{One decisive statistic, and several weaker ones.} Decisive: $m^\star{=}k$. No interval is needed for a statistic sitting at its maximum, and it is the overlap condition, not any error comparison, that Prop.~\ref{prop:diversity} makes exact. Weaker: everything involving $\bar{\epsilon}$. The mean per-worker error is $\bar{\epsilon}{=}\statPBOverlapEpsBar$ (CI \statPBOverlapEpsBarCI), so the Markov bound $2\bar{\epsilon}{=}\statPBOverlapMarkov$ is satisfied but nearly vacuous---indeed weaker than $\bar{\epsilon}$ itself, exactly as predicted for $\bar{\epsilon}<\tfrac12$. The vote's error \statPBOverlapVoteErr\ does sit above $\bar{\epsilon}$ at every round, but we rest nothing on that: the paired gap is \statPBOverlapGap\ with 95\%~CI \statPBOverlapGapCI\ (\statPBOverlapGapPosPct\ of bootstrap replicates positive), an interval including zero, and its sign does not survive check~(iii) below. That interval is the \emph{final} round's, and the qualification is narrower than it may read: the paired gap excludes zero at every round through round~\statPBOverlapGapLastExclRound\ (\statPBOverlapGapLastExcl, 95\%~CI \statPBOverlapGapLastExclCI), and only at the last round does it straddle zero, as the eligible population shrinks and every error falls. We still rest nothing on it, because check~(iii) removes the excess at every round at once: what it is evidence of is tier pooling, and a round-by-round pattern of a pooling artifact is still an artifact. Fig.~\ref{fig:llm-overlap} therefore draws the marginal intervals, which overlap throughout, and leaves the paired comparison to this paragraph.

\emph{One limit, then three robustness checks.} The limit is the pool: these \statPBOverlapK\ workers are one vendor, two tiers, and three prompt variants, so this measures only the pessimistic half of the claim above (that a resampled same-family worker inherits the same hard points) and cannot test the optimistic half, that genuinely different model families satisfy the condition; that remains the open empirical question.

\emph{(i)~Granularity.} At one worker per arm (seed majority within arm, $k{=}\statPBOverlapArms$) the picture is unchanged, $m^\star{=}\statPBOverlapArmMStar{=}k$ with \statPBOverlapArmViolate\ tasks violating, so the result is not an artifact of counting seeds as workers. \emph{(ii)~Threshold.} At SWE-bench's own bar ($\tau{=}\statPBTauResolved$, every test passing; \S\ref{sec:partB}) the statistic is identical, $m^\star{=}\statPBOverlapMStarResolved{=}k$ with \statPBOverlapViolateResolved\ violating. That is expected rather than reassuring, since raising $\tau$ can only turn a solved task into a failed one and $m_i$ is monotone in it, but it means the headline does not rest on the looser threshold.

\emph{(iii)~Tier pooling}, which is where the two halves of the result come apart. Restricting to a single tier leaves the overlap at its ceiling on both sides ($m^\star{=}\statPBOverlapSonnetMStar{=}k$ over Sonnet's \statPBOverlapTierK\ workers, $m^\star{=}\statPBOverlapHaikuMStar{=}k$ over Haiku's), so the headline is not an artifact of mixing a strong tier with a weak one. The error comparison does not survive: within a tier the vote matches the average worker almost exactly (Sonnet \statPBOverlapSonnetVoteErr\ against $\bar{\epsilon}{=}\statPBOverlapSonnetEpsBar$, \statPBOverlapSonnetViolate\ violating, gap \statPBOverlapSonnetGap, CI \statPBOverlapSonnetGapCI; Haiku \statPBOverlapHaikuVoteErr\ against \statPBOverlapHaikuEpsBar, \statPBOverlapHaikuViolate\ violating, gap \statPBOverlapHaikuGap, CI \statPBOverlapHaikuGapCI). Mixing tiers inflates the majority-failure count while $\bar{\epsilon}$ averages the two, so the pooled gap's positive sign is a pooling effect rather than a fact about voting. Finally, $\tau$-thresholded quality is a structural proxy for the functional correctness Prop.~\ref{prop:diversity}'s $\pm1$ label denotes (\S\ref{sec:limitations}); Part~C cannot substitute, since its sessions share no task instances across workers and $m_i$ is undefined there.

Fig.~\ref{fig:llm-overlap} in \S\ref{sec:theory} plots the distribution and, in its right panel, the per-round errors with a cluster-bootstrap interval on every bar: the vote's error, $\bar\epsilon$, and the Markov bound \texttt{cor\_overlap\_markov\_bound}, whose interval is a deterministic $2\times$ of $\bar\epsilon$'s. Those are marginal intervals and they overlap at every round, which is the weaker of the two comparisons available; the paired test is the one reported above, and check~(iii) is what attributes the excess to tier pooling.

\paragraph{A separate route to the reweighted edge hypothesis.} Prop.~\ref{prop:soft-aggregation}'s hypothesis can also be discharged directly, at a stated cost. If AdaBoost's reweighting concentrates mass by at most a factor $\kappa$ versus uniform ($D_t(i)\le\kappa/N$ for all $i$) and worker $t$'s error \emph{under the uniform distribution} is at most $(\tfrac12-\gamma)/\kappa$, then $\mathrm{err}_t\le\tfrac12-\gamma$ as required (\texttt{lem\_bounded\_reweighting\_edge}). The concentration cap $\kappa$ need not be assumed: since $D_t(i)=\exp(-y_if_t(x_i))/\sum_j\exp(-y_jf_t(x_j))$ is a softmax of the margins, any bound $M$ on the margins yields $\kappa\le e^{2M}$ (\texttt{D\_le\_of\_margin\_bound}), and with $\pm1$ outputs $M=\sum_{s<t}|\alpha_s|$ always works, giving
\[
\kappa \;\le\; \exp\Bigl(2\sum_{s<t}|\alpha_s|\Bigr)
\]
as a function of confidences the algorithm already logs (\texttt{cor\_rho\_from\_confidences}), and hence a version of the edge conclusion with no $\kappa$ hypothesis at all (\texttt{cor\_reweighting\_edge\_from\_confidences}). What this buys is precise and limited: $\kappa$ moves from a run-level regularity \emph{assumption} to a \emph{computable} quantity. It does not become small. The bound grows exponentially in the number of rounds, so the uniform-error requirement it implies, $(\tfrac12-\gamma)e^{-2\sum_{s<t}|\alpha_s|}$, tightens exponentially and is close to vacuous after many rounds. Whether $\kappa$ is small on a given run is an empirical question this does not answer.

\subsection{Reusing a Pool: A Horizon, or an Interpolating Cone}

\paragraph{Reuse cannot supply the edge.} One structural question about multi-round weighted aggregation admits a sharp negative answer, and the reweighting recursion is what makes it askable: because $D_t$ is defined in closed form, $D_{t+1}(i)=D_t(i)e^{-\alpha_t y_ih_t(x_i)}/Z_t$ is a derived identity rather than a separate definition (\texttt{D\_succ\_eq}).

\begin{proposition}[Immediate Reuse Zeroes the Edge]
\label{prop:self-reuse}
Run one round with worker $h_t$ at the optimal confidence $\alpha_t=\tfrac12\log\frac{1-\epsilon_t}{\epsilon_t}$, where $\epsilon_t=\mathrm{err}_t\in(0,1)$. Then worker $h_t$'s weighted error against the distribution its own incorporation produced is exactly $\tfrac12$:
\[
\sum_{i:\,y_ih_t(x_i)=-1} D_{t+1}(i) \;=\; \tfrac12 .
\]
Its edge on the next round is therefore exactly zero.
\end{proposition}

\begin{proof}[Proof sketch]
Summing the recursion over the points $h_t$ gets wrong, where the tilt is the constant $e^{\alpha_t}$, gives $\epsilon_t e^{\alpha_t}/Z_t$ (\texttt{reused\_worker\_err\_eq}). At the optimal $\alpha_t$ the two halves of $Z_t$ are equal, $(1-\epsilon_t)e^{-\alpha_t}=\epsilon_te^{\alpha_t}$, since $e^{2\alpha_t}=(1-\epsilon_t)/\epsilon_t$ (\texttt{optimal\_alpha\_balances}); hence $Z_t=2\epsilon_te^{\alpha_t}$ and the ratio is $\tfrac12$. Machine-checked as \texttt{thm\_self\_reuse\_zero\_edge}; \texttt{exists\_zero\_edge\_reuse\_instance} discharges non-vacuity, since round zero starts uniform (\texttt{D\_zero\_eq\_uniform}) and a worker right on one of two points has $\epsilon_0=\tfrac12$. No square roots are needed: the balance identity does all the work.
\end{proof}

\paragraph{The naive multi-round scheme is ruled out.} This makes precise, and strengthens, a claim the discussion below states informally. Correlated workers fail the edge hypothesis via AdaBoost's own stopping rule; Prop.~\ref{prop:self-reuse} is the extremal instance of that rule---a worker is maximally correlated with itself, and the reweighting it induces is exactly calibrated to neutralize it. So the obvious way to extend Prop.~\ref{prop:soft-aggregation} to many rounds---keep consulting the worker you just consulted---cannot sustain the edge hypothesis for even one further round. Prop.~\ref{prop:soft-aggregation} is untouched: its workers $h_0,\dots,h_{k-1}$ are each incorporated once, against the distribution current at that time, so no worker is reused there. And this settles only \emph{self}-reuse; whether some \emph{other} member of a reused pool still has an edge is a different question, settled below (Prop.~\ref{prop:pool-horizon}). The alignment with Part~B's high failure overlap is suggestive rather than confirmatory: Part~B runs sequential self-refinement, not AdaBoost reweighting, so the result explains why reuse is structurally unpromising without predicting that measurement.

\paragraph{The cross-worker case: a capacity ceiling, not a correlation obstruction.} Prop.~\ref{prop:self-reuse} leaves open whether some \emph{other} member of a reused pool retains an edge. Two facts settle the shape of the answer. First, there is no per-round obstruction: against any round's distribution---including the one a reuse produces---a binary worker with error at most $1/N$ exists, hence edge at least $\tfrac12-1/N$ (\texttt{exists\_edge\_worker\_vs\_any\_round}), so Prop.~\ref{prop:self-reuse} does not generalize into a one-round impossibility. Second, what does bind is the pool's \emph{capacity}: collecting the confidences a schedule assigns to each pool member shows the ensemble margin never leaves the pool's nonnegative cone, a set fixed by the pool and independent of both the number of rounds and the schedule (\texttt{lem\_pool\_span\_collapse}). The next two propositions are the two halves of Prop.~\ref{prop:pool-dichotomy}, stated here with the exact hypotheses the main-text summary compresses: this one covers $\epsilon_0>0$, and Prop.~\ref{prop:interpolating-pool} below covers $\epsilon_0{=}0$.

\begin{proposition}[Reused-Pool Horizon]
\label{prop:pool-horizon}
Run Algorithm~\ref{alg:boost} with a fixed pool of $k$ workers under any schedule, at the optimal confidences, every scheduled worker having edge $\gamma>0$ against the distribution current when it is used. If no element of the pool's nonnegative cone achieves zero-one training error below $\epsilon_0>0$, then
\[
\epsilon_0 \;\le\; e^{-2\gamma^2 T}, \qquad\text{hence}\qquad T \;\le\; \frac{\log(1/\epsilon_0)}{2\gamma^2}.
\]
The bound is uniform over the schedule and the confidences, and depends on $k$ only through $\epsilon_0$. Consequently the edge hypothesis cannot hold at \emph{every} round: a pool whose cone cannot fit the training data provably fails it by a finite, computable round.
\end{proposition}

\begin{proof}[Proof sketch]
Grouping the rounds by which pool member they use rewrites the margin $\sum_{t<T}\alpha_t h_{\sigma(t)}$ as $\sum_j A_j H_j$, with $A_j$ the total confidence sent to member $j$ (\texttt{lem\_pool\_span\_collapse}); each $A_j\ge0$ because an edge at the optimal confidence forces $\alpha_t\ge0$ (\texttt{alpha\_opt\_nonneg}), which is AdaBoost's stopping rule read forwards. So the run's zero-one loss \emph{is} the loss of a cone element (\texttt{zero\_one\_loss\_eq\_spanLoss}), which $\epsilon_0$ bounds below, while Thm.~\ref{thm:training-error} and Cor.~\ref{cor:convergence} bound it above by $e^{-2\gamma^2T}$. Since a fixed pool supplies one $\epsilon_0$ for every $T$, a large enough $T$ contradicts the horizon (\texttt{cor\_no\_perpetual\_pool\_edge}). Machine-checked as \texttt{thm\_pool\_reuse\_horizon}; \texttt{exists\_positive\_span\_floor} and \texttt{exists\_pool\_reuse\_horizon\_instance} discharge non-vacuity, the latter exhibiting a pool where a positive cone floor and a surviving edge round hold \emph{jointly}, so the statement is not vacuously true.
\end{proof}

\paragraph{The obstruction is capacity, not correlation.} The mechanism is not that reused workers correlate. It is that a fixed pool has a fixed reachable class, which is Prop.~\ref{prop:dichotomy}'s stationarity dichotomy and its frozen-weight ceiling (Cor.~\ref{cor:frozen-ceiling}) transplanted into the weighted-vote setting. The qualitative version of this is classical. AdaBoost's minimax duality already gives that a pool is $\gamma$-weak-learnable exactly when its convex hull separates every training point with margin at least $2\gamma$ \citep{schapire2013explaining}. What is contributed is the \emph{effective} negative half, machine-checked, with an explicit horizon uniform in schedule and confidences; the duality's converse direction, weak learnability $\Rightarrow$ interpolation with margin, is not formalized. Next, this result itself says nothing when $\epsilon_0=0$: a pool whose cone interpolates the training set is untouched by the horizon argument, and such a pool \emph{does} sustain the edge, which is now machine-checked separately, and without duality, as Prop.~\ref{prop:interpolating-pool}. The two propositions partition the cases by the sign of $\epsilon_0$, so between them the negative and positive content are both accounted for. Finally, the proof itself is not reuse-specific, since any sequence of workers factors as a pool of size $T$ used once each. What does the work is the \emph{fixity} of $\epsilon_0$ across $T$, not the repetition of workers; with fresh workers each round no single $\epsilon_0$ survives.

\paragraph{The interpolating case, and why it is the easy half.} Prop.~\ref{prop:pool-horizon} is silent when $\epsilon_0=0$. That gap closes in the positive direction: a pool whose nonnegative cone fits the training data sustains the edge hypothesis against \emph{every} round's distribution, at an edge fixed once and for all by the fitting combination's margin.

\begin{proposition}[Interpolating Pools Sustain the Edge]
\label{prop:interpolating-pool}
Let $H_1,\dots,H_k$ be binary workers, and suppose some nonnegative combination $A$ interpolates the sample: $y_i\sum_j A_jH_j(x_i)>0$ for every $i$, equivalently $\mathrm{spanLoss}(A)=0$. Put $\theta=\min_i y_i\sum_j A_jH_j(x_i)>0$ and $S=\sum_j A_j>0$. Then for \emph{every} distribution $D$ on the sample there is a pool member $j$ with
\[
\sum_{i:\,y_iH_j(x_i)=-1} D(i) \;\le\; \tfrac12-\gamma, \qquad \gamma \;=\; \frac{\theta}{2S} \;>\; 0 .
\]
The order of quantifiers is the content: $\gamma$ is fixed \emph{before} $D$ is chosen, so the edge holds at every round of any run, whatever the schedule and confidences did beforehand.
\end{proposition}

\begin{proof}[Proof sketch]
Zero cone loss means a \emph{strict} margin at every point, and finitely many positive numbers have a positive minimum $\theta$ (\texttt{lem\_margin\_floor\_of\_spanLoss\_zero}); $S>0$ because vanishing coefficients would make every margin zero (\texttt{lem\_span\_coeff\_sum\_pos}). Averaging the pointwise bound $\theta\le y_i\sum_jA_jH_j(x_i)$ against $D$ and exchanging the two sums gives $\theta\le\sum_jA_je_j$, where $e_j=\sum_iD(i)\,y_iH_j(x_i)$ is worker $j$'s $D$-weighted margin. A nonnegative combination cannot exceed its coefficient sum times its largest term, so $\theta\le S\max_je_j$ and some $e_j\ge\theta/S$. For a binary worker $e_j=1-2\epsilon_j$ (\texttt{lem\_weighted\_margin\_eq}), whence $\epsilon_j\le\tfrac12-\theta/2S$.
\end{proof}

\paragraph{The positive half, proved without duality.} No duality is used. AdaBoost's minimax duality is needed for the \emph{converse} (weak learnability $\Rightarrow$ interpolation with margin), which remains unformalized; the direction proved here is averaging plus a pigeonhole.

First, Prop.~\ref{prop:interpolating-pool} and Prop.~\ref{prop:pool-horizon} are complementary halves of a dichotomy on $\epsilon_0$ rather than overlapping claims: the horizon's floor hypothesis together with $\epsilon_0>0$ is outright incompatible with the existence of an interpolating cone element (\texttt{cor\_pool\_horizon\_floor\_excludes\_interpolation}), so the restriction to $\epsilon_0>0$ is the boundary between the two cases, not an artifact of the proof. Second, and more consequential, this is a \emph{per-round} statement: at every round some member has the edge. It does not exhibit a schedule and confidence sequence realizing a full run, which would require defining the two by simultaneous recursion ($\alpha_t$ is pinned by the optimal-confidence rule to a quantity that reads $D_t$, which in turn reads the schedule and confidences before $t$) and would additionally need to exclude a \emph{perfect} pool member, since Cor.~\ref{cor:convergence} assumes $\epsilon_t>0$ and an interpolating pool is precisely the regime where a member may misclassify nothing. So the multi-round weighted-vote run is still not formalized; what closes is the $\epsilon_0=0$ case.

\subsection{Are the Two Conditions Equivalent?}

\paragraph{Are the two conditions equivalent?} Prop.~\ref{prop:soft-aggregation}'s reweighted edge hypothesis and Prop.~\ref{prop:diversity}'s overlap condition $m^\star{<}k/2$ have stood alongside each other with no relation established between them. They are in fact independent.

\begin{proposition}[The Two Conditions Are Independent]
\label{prop:conditions-independent}
Neither the reweighted edge hypothesis nor the bounded-overlap condition implies the other. Both directions are witnessed on four points with three workers. (i)~Three workers with reweighted errors $\tfrac14,\tfrac16,\tfrac25$---each at most $\tfrac12-\tfrac1{10}$, so the edge hypothesis holds with $\gamma=\tfrac1{10}$ at every round---in which one point is failed by two of the three, so $2m^\star=4\ge k$ and the overlap condition fails. (ii)~Three workers with $m^\star=1<k/2$, hence a correct unweighted vote, in which one worker has error $\tfrac34$, so no $\gamma>0$ makes the edge hypothesis true.
\end{proposition}

\begin{proof}[Proof sketch]
In~(i) the errors are the genuinely reweighted ones, not uniform-distribution stand-ins. At the optimal confidence the reweighting recursion loses its exponentials---a failing point's weight is divided by $2\epsilon_t$ and a passing point's by $2(1-\epsilon_t)$ (\texttt{D\_succ\_fail}, \texttt{D\_succ\_pass}, both from \texttt{D\_succ\_eq} and \texttt{optimal\_alpha\_balances})---so rational errors keep the entire distribution rational and the three rounds are settled by arithmetic. With failure sets $\{0\},\{1\},\{0,2\}$ the successive distributions are $(\tfrac14,\tfrac14,\tfrac14,\tfrac14)$, $(\tfrac12,\tfrac16,\tfrac16,\tfrac16)$, $(\tfrac3{10},\tfrac12,\tfrac1{10},\tfrac1{10})$, and point~$0$ is failed by workers $0$ and $2$ (\texttt{thm\_edge\_not\_imply\_overlap}). In~(ii), one worker failing three of the four points has error $\tfrac34$ against the uniform round-zero distribution while no point is failed by more than one worker (\texttt{thm\_overlap\_not\_imply\_edge}); note $\tfrac34$ lies strictly inside $(0,1)$, so what fails is the edge bound itself rather than a degeneracy at $\epsilon\in\{0,1\}$. Packaged as \texttt{thm\_edge\_overlap\_independent}.
\end{proof}

\subsection{The Probabilistic (Condorcet) Route}

\paragraph{The probabilistic route.} Diversity also has a probabilistic formulation---the Condorcet/Hoeffding bound, where independent worker errors yield exponential decay for the majority vote. It can be had without leaving Prop.~\ref{prop:diversity}'s combinatorial setting.

\begin{proposition}[Condorcet Bound for Independent Errors]
\label{prop:condorcet}
Let $2m$ workers fail independently, worker $i$ with probability $q_i\le\tfrac12-\gamma$ for some $\gamma\in(0,\tfrac12)$. Then the unweighted majority vote errs with probability at most
\[
(1-4\gamma^2)^m \;\le\; e^{-2\gamma^2\cdot 2m}.
\]
\end{proposition}

\begin{proof}[Proof sketch]
Write the weight of a failure pattern $\mathcal{F}$ as the product $\prod_{i\in\mathcal{F}}q_i\prod_{i\notin\mathcal{F}}(1-q_i)$ (\texttt{wt}), and the vote's error as the total weight of the patterns with $2|\mathcal{F}|\ge 2m$ (\texttt{voteErr}). The algebraic identity $\prod_i(f_i+g_i)=\sum_{\mathcal{F}}\prod_{i\in\mathcal{F}}f_i\prod_{i\notin\mathcal{F}}g_i$ does all the work: at $f_i=q_i,\,g_i=1-q_i$ it shows the weights sum to $1$ (\texttt{lem\_wt\_sum\_one}), and at $f_i=q_ir$ it factorizes the tilted sum (\texttt{lem\_wt\_tilt\_sum}). Taking $r=(1-\gamma')/\gamma'$ with $\gamma'=\tfrac12-\gamma$, every counted pattern has $r^{|\mathcal{F}|}\ge r^m$, so the error is at most $r^{-m}\prod_i(q_ir+1-q_i)\le r^{-m}\bigl(2(1-\gamma')\bigr)^{2m}=\bigl(4\gamma'(1-\gamma')\bigr)^m=(1-4\gamma^2)^m$, and $1-4\gamma^2\le e^{-4\gamma^2}$ finishes. Machine-checked as \texttt{thm\_condorcet\_majority\_bound}, with \texttt{exists\_condorcet\_instance} for non-vacuity. No moment-generating function is evaluated and no logarithm appears: $r$ is written down directly.
\end{proof}

\paragraph{The price of assuming independence.} The constant $e^{-2\gamma^2k}$ is exactly Cor.~\ref{cor:convergence}'s, so Prop.~\ref{prop:condorcet} composes with Prop.~\ref{prop:soft-aggregation} with no further analysis---the same $k_0$ works. Independence was initially \emph{encoded} in the product form of the pattern weight rather than derived. It is now derived. Prop.~\ref{prop:condorcet-prob} places the workers on a genuine product of Bernoulli measures over $\{0,1\}^{2m}$, obtains the independence of the per-worker failure indicators from Mathlib's product-measure machinery, and identifies the measure of the majority-failure event with the very quantity Prop.~\ref{prop:condorcet} bounds. So the stronger reading of ``independent'' is earned, and by an identification rather than by re-deriving the bound from Hoeffding's inequality: the combinatorial proof was always a statement about this measure, and that is now a theorem instead of an assertion. What no proof can discharge is the decision to model worker failures by a product measure at all---that is the Condorcet jury setup itself. Nor is the assumption innocent empirically: independent failures across workers is precisely what Part~B's \statPBOverlapViolate\ majority-failure rate fails to exhibit (\S\ref{sec:partB}), and what Prop.~\ref{prop:diversity}'s overlap condition constrains without assuming.

\begin{proposition}[Condorcet Bound with Independence Derived]
\label{prop:condorcet-prob}
Let $\mu=\bigotimes_{i=1}^{2m}\mathrm{Bernoulli}(q_i)$ on $\{0,1\}^{2m}$, coordinate $i$ recording whether worker $i$ fails. The per-worker failure indicators are independent under $\mu$, and for $q_i\le\tfrac12-\gamma$ with $\gamma\in(0,\tfrac12)$,
\[
\mu\bigl(\{\omega:\ \#\{i:\omega_i=1\}\ge m\}\bigr) \;\le\; e^{-2\gamma^2(2m)} .
\]
The event's measure equals the $\mathrm{voteErr}$ of Prop.~\ref{prop:condorcet}, so the two propositions bound the same quantity.
\end{proposition}

\begin{proof}[Proof sketch]
Independence is not assumed: it is Mathlib's \texttt{iIndepFun\_pi} applied to the coordinate projections of a product measure (\texttt{lem\_coord\_iIndepFun}). The measure of a single failure pattern is the product a product measure actually has, which is exactly the pattern weight $\mathrm{wt}$ (\texttt{lem\_pattern\_measure}, via \texttt{Measure.pi\_singleton}); summing over the patterns in which at least half the workers fail, reindexed along the bijection between patterns and failure sets, identifies the event's measure with $\mathrm{voteErr}$ (\texttt{lem\_failure\_event\_eq\_voteErr}). Prop.~\ref{prop:condorcet} then transports verbatim. The bound is not re-proved, and Hoeffding's inequality is not used: the sample space is a finite product of finite spaces, so no centering step or measurability side condition arises.
\end{proof}

\subsection{Status: What Is Settled, and What Is Not}

\paragraph{What remains open.} The comparison above is to the \emph{weak-learning threshold} $\tfrac12-\gamma$, not to whichever single worker happens to be best in a given instantiation: a worker that is accidentally already excellent is not beaten by this result, nor need it be. Best-of-$k$ hard selection (Prop.~\ref{prop:orchestration}) already reaches that worker's ceiling exactly. The regime where weighted aggregation earns its keep is the reverse one: many individually-mediocre workers, none dominating the others, where hard selection is stuck at the least-bad of a bad set while weighted voting escapes that floor, given the diversity condition that each worker keeps beating the \emph{reweighted} distribution which up-weights points previous workers got wrong. Correlated or identical workers fail this hypothesis via AdaBoost's own stopping rule ($\epsilon_t\ge\tfrac12 \Rightarrow \alpha_t\le0$, Algorithm~\ref{alg:boost}), not via an assumption bolted on separately. Diversity is encoded in the edge hypothesis itself, not assumed on top of it.

Three questions about multi-round weighted voting are settled here, and not all of what settles them is ours. \emph{Reuse} is bounded: consulting the same worker twice zeroes its edge outright (Prop.~\ref{prop:self-reuse}), and a reused \emph{pool} either interpolates the sample (in which case some member has an edge against every distribution, Prop.~\ref{prop:interpolating-pool}) or does not, in which case the edge hypothesis fails by round $\log(1/\epsilon_0)/2\gamma^2$ (Prop.~\ref{prop:pool-horizon}). That dichotomy is Prop.~\ref{prop:pool-dichotomy} in the main text; its qualitative half follows from AdaBoost's minimax duality and is not ours, while the effective form with an explicit horizon is. \emph{Diversity} has a probabilistic route as well as a combinatorial one: the Condorcet/Hoeffding majority bound holds at the same constant $e^{-2\gamma^2k}$ as Cor.~\ref{cor:convergence} (Prop.~\ref{prop:condorcet}), with the independence of the per-worker failure indicators derived on a product of Bernoulli measures rather than encoded in the pattern weight (Prop.~\ref{prop:condorcet-prob}). \emph{The two conditions} are logically independent: neither $m^\star{<}k/2$ nor the reweighted edge hypothesis implies the other (Prop.~\ref{prop:conditions-independent}), and the concentration cap $\kappa$ the latter needs is derivable from logged confidences rather than assumed.

What remains outside the formalization is not a gap in a proof. The multi-round weighted-vote \emph{run}---a schedule and confidence sequence exhibited outright, as distinct from the per-round edge feeding it---is not formalized, and additionally requires excluding a perfect pool member since Cor.~\ref{cor:convergence} assumes $\epsilon_t>0$. And choosing a product measure over worker failures \emph{is} the Condorcet jury setup: it is a modelling decision no formalization can discharge, and precisely the assumption Part~B's \statPBOverlapViolate\ majority-failure rate shows real workers violate.

\section{Part A: Additional Figures}
\label{app:parta}

\paragraph{Part A's scope.} Labels throughout Part A are synthetic---a fixed linear separator applied to real SWE-bench TF-IDF features plus noise---so no claim about code, correctness, or LLM behavior is being tested here; the same curves would result from running the identical script on any other feature distribution. Part A's role is an implementation sanity check, confirming the training-error, margin, and generalization machinery is coded correctly before that machinery is applied to Part~B/C data, and it supplies the positive-edge reference trajectory against which Part~B's measured decay in $\tilde\gamma_t$ (\S\ref{sec:partB}) becomes interpretable as a property of real agents rather than of a broken measurement pipeline. Note the estimand shift: Part~A's stumps have a genuine solve-edge $\gamma_t>0$, whereas Part~B's $\tilde\gamma_t$ scores improvement over the previous draft and goes negative largely because the loop stops changing the artifact---so the two trajectories are comparable in shape, not in sign. The claim that boosting mechanics govern real code generation is carried entirely by Parts~B and~C, not by Part~A.

\begin{figure}[h]
\centering
\includegraphics[width=0.47\textwidth]{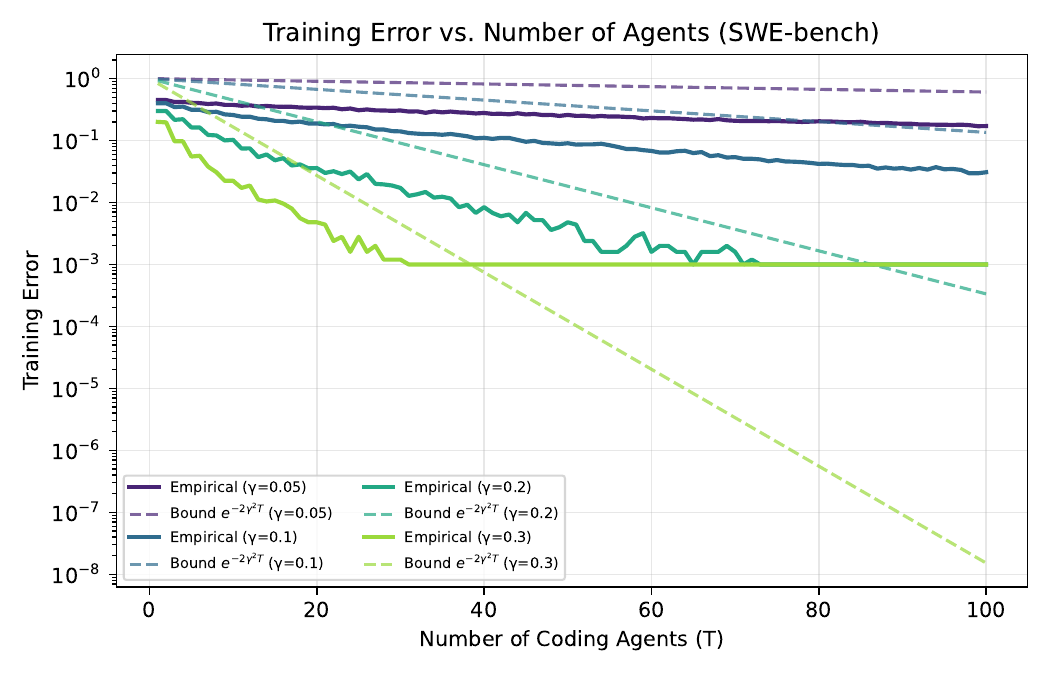}
\includegraphics[width=0.47\textwidth]{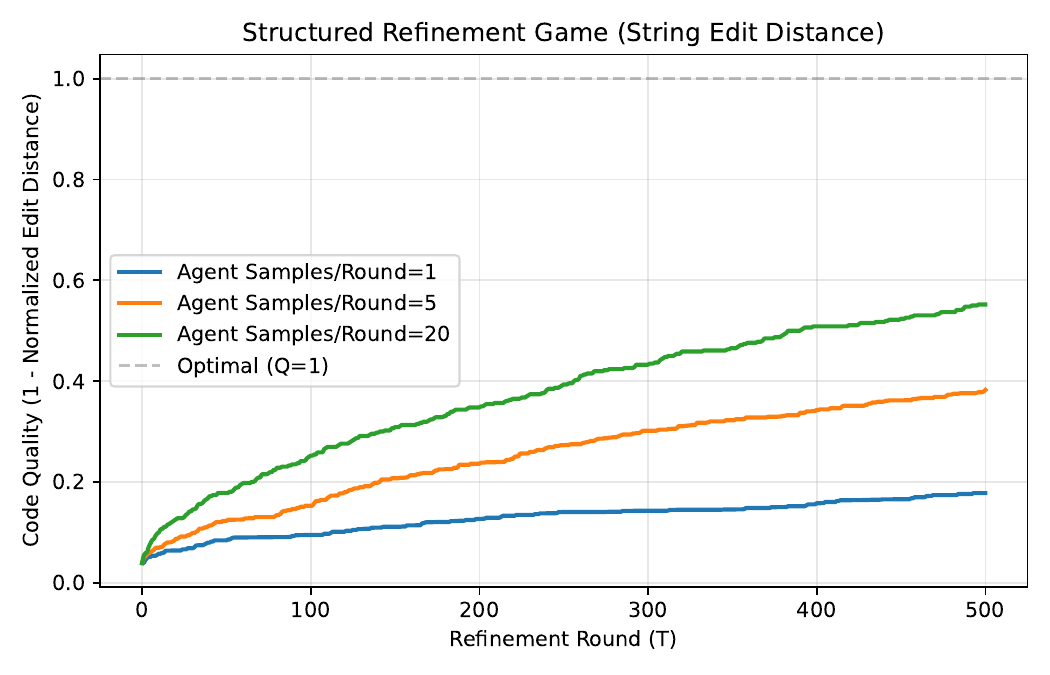}
\caption{\textbf{Left:} training error vs.\ number of weak coding agents for four target edges $\gamma$; solid = empirical (mean of 5 trials), dashed = bound $e^{-2\gamma^2T}$, confirming Cor.~\ref{cor:convergence}. \textbf{Right:} code quality under the refinement game (Levenshtein-distance optimization), confirming Thm.~\ref{thm:refinement}; quality increases monotonically with diminishing returns, remaining bounded below 1.0, with more capable agents converging higher.}
\label{fig:exp1}
\end{figure}

\begin{figure}[h]
\centering
\includegraphics[width=\textwidth]{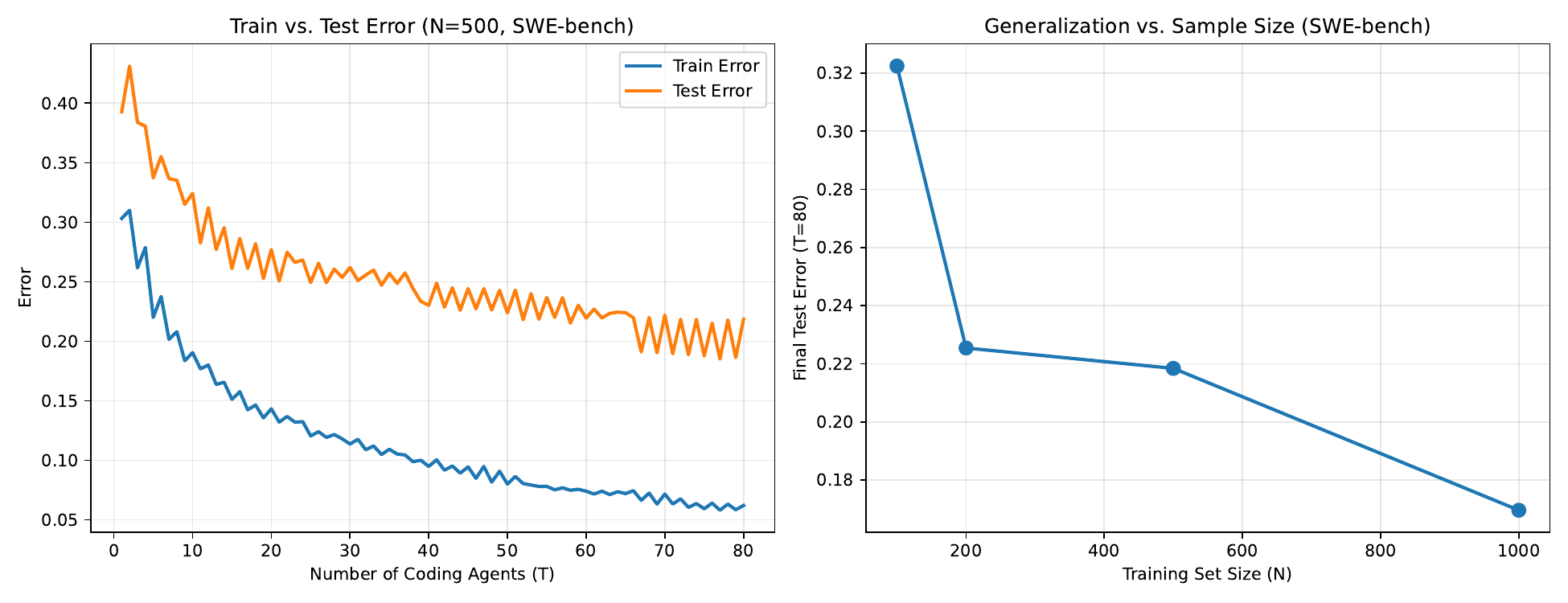}
\caption{Left: Train vs.\ test error for $N=500$ over 80 boosting rounds. Test error decreases then plateaus, consistent with the generalization bound. Right: Final test error decreases monotonically with training set size (averaged over 5 seeds on the full SWE-bench split with a disjoint held-out test set), consistent with the $O(\sqrt{d/N})$ term.}
\label{fig:exp2}
\end{figure}

\begin{figure}[h]
\centering
\includegraphics[width=0.7\textwidth]{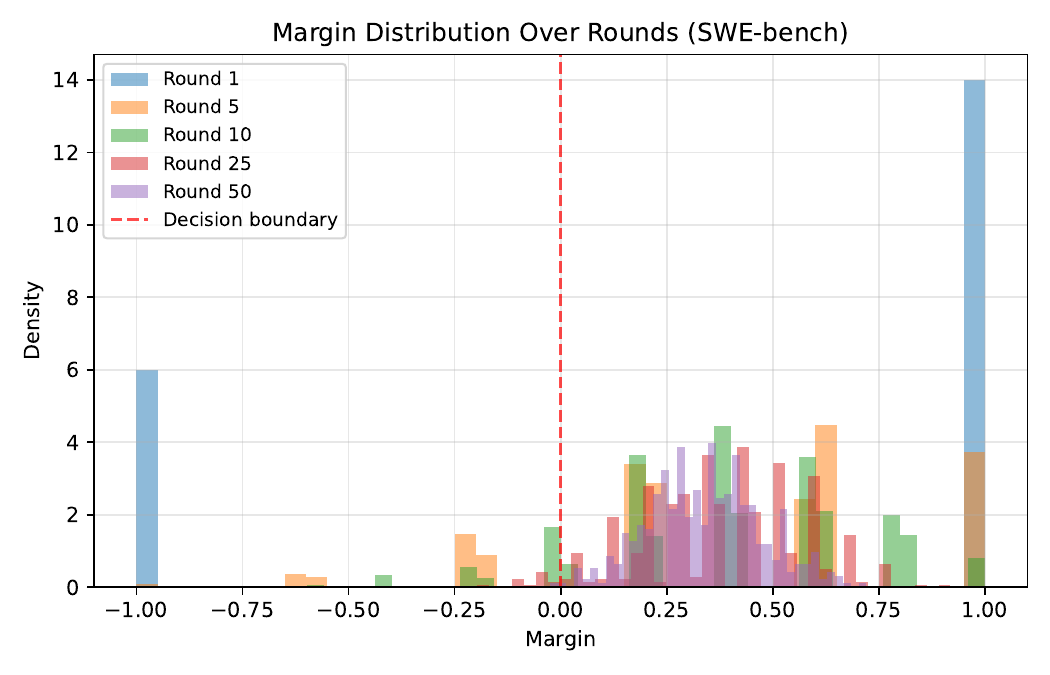}
\caption{Margin distribution over boosting rounds (edge-calibrated weak agents). Round~1 is a single learner with saturated $\pm1$ margins (roughly half misclassified); over successive rounds the negative-margin mass collapses and the distribution concentrates on the positive side of the boundary. This concentration of margin mass away from the decision boundary is the mechanism by which boosting achieves good generalization despite adding more weak learners.}
\label{fig:exp4}
\end{figure}

\section{Part B: Harness Verification, Task Pool, and Power Analysis}
\label{app:partb}

\paragraph{Harness verification and task pool (methods disclosure).} Our evaluator is not SWE-bench's own. A $T$-round refinement experiment needs per-round scoring inside a container that stays alive between rounds; the official harness is a \emph{batch} evaluator---\texttt{run\_evaluation.py} grades one patch per instance from a predictions file---so re-invoking it per round per task across ${\approx}6{,}600$ evaluations was not viable, and we built the loop ourselves, scoring \texttt{--junit-xml} reports directly. We stand by that decision and not by its execution: SWE-bench also ships per-repository test commands (\texttt{MAP\_REPO\_VERSION\_TO\_SPECS}) and per-repository log parsers (\texttt{MAP\_REPO\_TO\_PARSER}), both importable independently of its CLI and its container lifecycle, and we reimplemented that knowledge instead of importing it.

A gold-patch audit---apply the \emph{known-correct} SWE-bench patch and check it scores $1.0$---of the originally-sampled $K{=}50$ pool found \statPBAuditFailN\ tasks failing deterministically, none of them model failures, separable into \statPBAuditRootCauses\ root causes. Three are ours outright. \texttt{pytest} is not the test runner for Django or SymPy at all (9 tasks): their official commands are \texttt{./tests/runtests.py --verbosity 2} and \texttt{bin/test -C --verbose}. An older pinned \texttt{pytest} rejects the \texttt{--no-header} flag we passed universally (3), which the official specs apply only to \texttt{seaborn}, where the pinned version accepts it. And ANSI color codes broke a text-based pass/fail line parser (3), which \texttt{astropy}'s official command forestalls by setting \texttt{console\_output\_style=classic}. One is partly genuine and independently known: parametrize-ID drift between \texttt{pytest}'s auto-generated collection ids and SWE-bench's recorded ids (7). One is mundane: a network-dependent test that hangs inside a sandboxed evaluator with no internet access (1). We record the split because the count is a fact about \emph{our} evaluator rather than about the benchmark: read as a benchmark defect rate it would be simply wrong, and the distinction is the whole reason the breakdown is here rather than a bare failure count.

We fixed the two shared-root-cause bugs (replacing text-parsing of colorized \texttt{pytest -v} output with \texttt{--junit-xml} scoring, and fuzzy-matching parametrize ids via a \texttt{pytest --collect-only} pass) and verified the fix against the gold patch on every previously-flagged task (all now score $\mathrm{fraction\_passing}{=}1.0$); those two are the transferable part of the exercise. We then excluded Django and SymPy entirely, not because their environments are defective but because a \texttt{pytest}-only evaluator cannot run them, plus the one network-dependent task, and backfilled \statPBTaskPoolBackfilled\ gold-patch-verified replacement tasks to reach $K{=}\statPBK$; every active task's Docker image is now digest-pinned. All Part~B results use this harness-fixed, gold-verified $K{=}\statPBK$ pool, which is therefore \texttt{pytest}-native by construction. That is a property of our tooling that shaped the pool, and the reason a follow-up should add per-repository dispatch and restore both repositories rather than inherit this selection.

\emph{One divergence survives inside the retained pool.} SWE-bench runs \texttt{sphinx} tests through \texttt{tox --current-env -epy39 -v --} rather than bare \texttt{pytest}, and \texttt{sphinx-doc} is one of the \statPBTaskPoolRepos\ retained repositories. What bounds the risk is the audit itself: every retained task, \texttt{sphinx} included, scores $\mathrm{fraction\_passing}{=}1.0$ under the gold patch, so bare \texttt{pytest} is collecting and running those tests rather than silently reporting nothing. Agreement with the official grader on the retained pool is established only through that gold-patch check, not by running both harnesses side by side.

\begin{table}[h]
\centering\small
\caption{Part~B: $2\times3$ repeated-measures ANOVA (model $\times$ feedback mode) on final-round quality, $K{=}\statPBK$ tasks $\times$ \statPBSeeds\ seeds. Summarized in \S\ref{sec:partB}. The model effect's Cohen's $f$ carries 95\%~CI \statPBModelCohenFCI, consistent with the near-complete tier separation in solve rates rather than with a computation artifact.}
\label{tab:swebench-anova}
\begin{tabular}{lccccc}
\toprule
Effect & $F$ & df & $p$ & partial $\eta^2$ & Cohen's $f$ \\
\midrule
Model (Sonnet vs.\ Haiku) & \statPBModelF & \statPBModelDf & \statPBModelP & \statPBModelEta & \statPBModelCohenF\ (huge) \\
Feedback mode & \statPBModeF & \statPBModeDf & \statPBModeP & \statPBModeEta & \statPBModeCohenF\ (small) \\
Model $\times$ mode & \statPBInterF & \statPBInterDf & \statPBInterP & \statPBInterEta & \statPBInterCohenF\ (small) \\
\bottomrule
\end{tabular}
\end{table}

\paragraph{Per-arm solve rates.} The ranges quoted in \S\ref{sec:partB} resolve as follows at $\ge\tau{=}0.6$ with 95\%~Wilson intervals. Haiku: independent \statPBHaikuIndep, blind \statPBHaikuBlind, diagnostic \statPBHaikuDiag. Sonnet: independent \statPBSonnetIndep, blind \statPBSonnetBlind, diagnostic \statPBSonnetDiag. The intervals within a tier overlap heavily, which is the same fact the mode term's null records. At SWE-bench's own criterion ($\tau{=}\statPBTauResolved$) the weak tier's mode ordering inverts, which is why \S\ref{sec:partB} reads the mode contrast as noise: Haiku independent \statPBHaikuResolvedIndep\ against diagnostic \statPBHaikuResolvedDiag, while Sonnet's diagnostic arm gives \statPBSonnetResolvedDiag.

\paragraph{The dependent variable, and an asymmetry in it.} The dependent variable is final-round quality averaged over seeds. It is therefore threshold-free, which is why $\tau$ moves the solve-rate tallies of \S\ref{sec:partB} and nothing else: the ANOVA's DV is continuous quality, and $\tilde\gamma_t$ compares $q_{i,t}$ against $q_{i,t-1}$ directly rather than against any threshold. Because the harness back-fills the rounds after a task first passes (Appendix~\ref{app:edge-decomposition}), a solved row's final round is its running best, whereas an unsolved row contributes whatever its last attempt scored---which may sit below its own earlier best, on \statPBHaikuRegressRows\ and \statPBSonnetRegressRows\ rows. The stronger tier solves more rows, so the asymmetry runs in the same direction as the tier effect and cannot be dismissed on inspection. Re-running the identical ANOVA with every row contributing its running best moves the tier effect from partial $\eta^2{=}\statPBModelEta$ to \statPBModelEtaRunningBest, and leaves mode and interaction non-significant: the effect does not rest on the asymmetry.

\paragraph{Task pool.} The active $K{=}\statPBK$ pool spans \statPBTaskPoolRepos\ SWE-bench Lite repositories (\texttt{astropy}, \texttt{flask}, \texttt{matplotlib}, \texttt{pylint}, \texttt{pytest}, \texttt{requests}, \texttt{scikit-learn}, \texttt{seaborn}, \texttt{sphinx}, \texttt{xarray}); Django and SymPy are excluded entirely and one network-dependent task is excluded. Of the originally-sampled 50 tasks, \statPBTaskPoolRetained\ are retained unchanged and \statPBTaskPoolBackfilled\ are gold-patch-verified replacements backfilled to reach $K{=}\statPBK$. Exhaustive per-task, per-arm, per-seed quality trajectories for all $6$ arms $\times$ $\statPBK$ tasks $\times$ $\statPBSeeds$ seeds, together with the committed analysis script that reproduces every Part~B table and figure directly from that data, are in the supplementary repository.

\paragraph{Power analysis (full detail).} \emph{This null is underpowered, not evidence of zero effect.}

A case-resampling bootstrap (30{,}000 reps/$K$, preserving each task's full 6-arm profile) at the observed effect size finds only \statPBModePower\ power (mode) and \statPBInterPower\ power (interaction) at the $K{=}\statPBK$ we ran; reaching 80\% power requires $K{\approx}\statPBModeReqK$ (mode) and $K{\approx}\statPBInterReqK$ (interaction), roughly \statPBPowerMultiple\ the current sample. An independent, quantified reason compounds the power problem regardless of $K$: seed-to-seed noise within a fixed (task, model, mode) cell averages SD${\approx}\statPBSeedSD$, about \statPBSignalOverNoise\ the mean between-mode signal (mean $|\text{diagnostic}-\text{blind}|$ per task ${\approx}\statPBBetweenModeSignal$). The sampling-noise floor exceeds the effect being measured.

This noise is a deliberate, explained design choice, not an oversight: generation temperature $0.25$ is applied uniformly across all three arms, but it is \emph{structurally required} only for \texttt{independent}, whose prompt is identical every round (no prior patch, no critique); at temperature $0$ it would degenerate into $T$ identical greedy draws. \texttt{blind} and \texttt{diagnostic} have no such requirement, since their prompts already change round to round with the carried-forward patch and critique; a single shared temperature was chosen for simplicity over doubling the design into temperature-decoupled arms. A future rerun could decouple this, running temperature $0$ for \texttt{blind}/\texttt{diagnostic} and temperature ${>}0$ only for \texttt{independent}, as a variance-reduction change to the design, not a reanalysis of the data collected here.

\paragraph{Why report a null this underpowered at all?} First, the required-$K$ and noise-to-signal numbers above are themselves informative: they tell a future replication exactly how large a sample the effect needs, which a silently-dropped ablation would not. Second, the capability effect (Table~\ref{tab:swebench-anova}) is properly powered at this same $K$, so the ablation is not a uniformly underpowered design---only the mode and interaction terms are, and the significant capability result should not be read as implying the whole design was adequately powered. Third, the null is the empirical basis for the ``spend on the first attempt'' guidance (\S\ref{sec:discussion}): reporting it with its power caveat attached is more useful to a practitioner than either suppressing it or overclaiming it as a confirmed non-effect.

\paragraph{The point-estimate crossover (not significant; a hypothesis, not a finding).} Haiku's mean quality increases monotonically with more carried-forward context (independent \statPBHaikuIndepQ\ $<$ blind \statPBHaikuBlindQ\ $<$ diagnostic \statPBHaikuDiagQ), while Sonnet's \texttt{blind} arm is its \emph{worst} of the three (independent \statPBSonnetIndepQ, blind \statPBSonnetBlindQ, diagnostic \statPBSonnetDiagQ)---a crossover, not merely noise around a flat line (Fig.~\ref{fig:feedback-ablation}). Every formal test of this ordering is non-significant (linear contrast $F{=}$~\statPBContrastFRange, $p{=}$~\statPBContrastPRange\ across gating choices; model$\times$linear interaction $p{=}$~\statPBInteractionPRange), so we do not claim it, and we explicitly do \emph{not} attribute Sonnet's inversion to ``solved early, then regressed later''---we tested that mechanism directly against the per-round trajectories and it runs the wrong direction to explain the gap (Sonnet's \texttt{blind} gives back \emph{less} to late-round regression than \texttt{independent} does, $0.0091$ vs.\ $0.0105$).

\begin{figure}[h]
\centering
\includegraphics[width=0.6\textwidth]{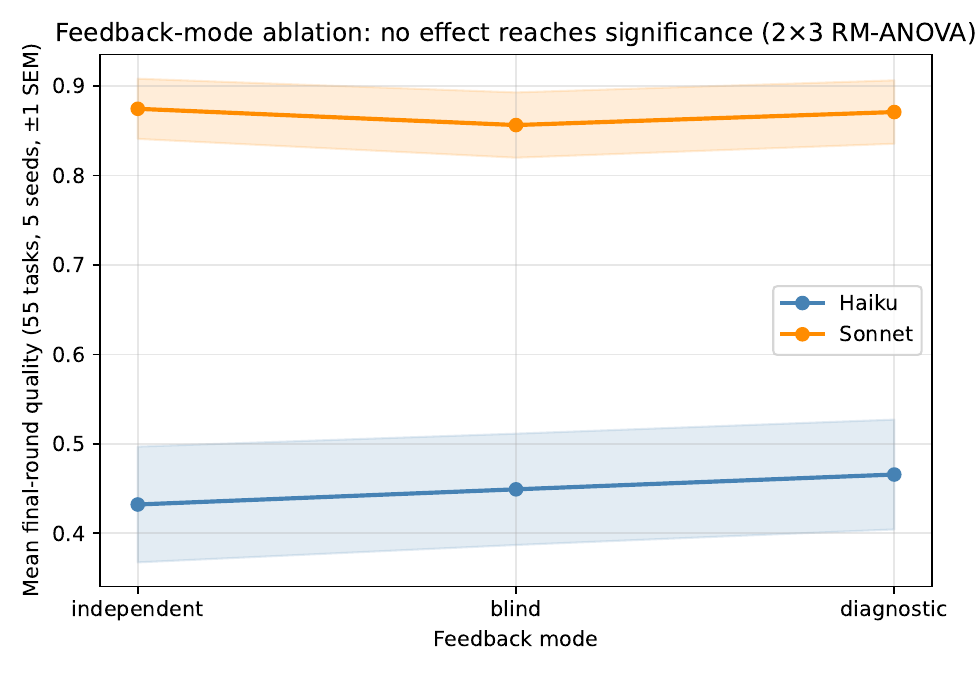}
\caption{Part~B feedback-mode ablation: mean final-round quality by mode (independent~$\to$~blind~$\to$~diagnostic), $K{=}\statPBK$ tasks $\times$ \statPBSeeds\ seeds, $\pm1$~SEM. Haiku increases monotonically; Sonnet's \texttt{blind} arm is its worst of the three---a non-significant crossover (see text), not a confirmed finding.}
\label{fig:feedback-ablation}
\end{figure}

\subsection{Decomposing the Measured Edge}
\label{app:edge-decomposition}

The main text reports the per-round edge on measured cells and the pooled series alongside it. This is the detail behind both, and the reason the pooled number cannot stand alone: $\gamma$'s \emph{sign} is the quantity the weak-learning condition turns on, and in this setting the sign is determined largely by something other than what the condition is about.

\paragraph{Rounds the experiment did not run.} The harness stops calling the model on the round a task first passes, then back-fills the remaining rounds with a hardcoded $q{=}1$ so the $T$-wide matrix stays rectangular. Those cells record no patch, no evaluation and no model call. They are \statPBHaikuPadShare\ of Haiku's matrix and \statPBSonnetPadShare\ of Sonnet's, and because $\gamma$ scores a tie at ${-}\frac12$, each one enters as a failure to improve. They account for \statPBHaikuPadTieRange\ (Haiku) and \statPBSonnetPadTieRange\ (Sonnet) of each round's tie mass. The pooled estimator is therefore driven downward in proportion to the solve rate, which inverts the comparison it looks like it supports: Sonnet solves far more tasks, back-fills far more cells, and posts a \emph{more} negative pooled edge at round~3 than Haiku (\statPBSonnetEdge\ against \statPBHaikuEdge, CIs \statPBSonnetEdgeCI\ and \statPBHaikuEdgeCI) while being the stronger refiner on every cell where a call was made. That pooled sign is therefore not reportable as a headline (\S\ref{sec:partB}): it is an artifact of the harness rather than a property of the loop.

\paragraph{Ties versus regressions.} The indicator uses a strict $>$, so ``not improved'' pools two unlike outcomes: quality identical to the previous round, and quality actually lower. Splitting them on the diagnostic arm gives per-round tie rates of \statPBHaikuTieSeries\ for Haiku and \statPBSonnetTieSeries\ for Sonnet, against regression rates of \statPBHaikuRegressSeries\ and \statPBSonnetRegressSeries. Net of the back-filling above, the genuine no-op rate is \statPBHaikuGenuineTieRange\ (Haiku) and \statPBSonnetGenuineTieRange\ (Sonnet), still the dominant outcome and the figure that actually supports ``refinement turns idempotent''. Whole trajectories tell the same story: \statPBHaikuFlat\ of Haiku's \statPBEdgeN\ (task, seed) rows and \statPBSonnetFlat\ of Sonnet's never change quality at any round, and only \statPBHaikuRegressRows\ and \statPBSonnetRegressRows\ rows respectively dip even once. A refiner that has fully converged---every round identical, never a regression---attains $\gamma={-}\frac12$ exactly, the floor of the estimator's range. So $\gamma<0$ here is evidence that the loop has stopped changing the artifact, which is what Prop.~\ref{prop:dichotomy}(i) predicts; it is \emph{not} evidence that patches are anti-correlated with correctness, which is what $\epsilon_t>\frac12$ means in AdaBoost and what ``the weak-learning condition fails'' would ordinarily convey. Both readings are consistent with the same number, and only the first is supported.

\paragraph{The edge among tasks that can still move.} A row already at $q{=}1$ cannot improve, so it enters $\gamma$'s denominator as a row that did not improve when it is really a row that could not. Conditioning each round on $q_{i,t-1}<1$ removes that: the eligible population falls \statPBHaikuHeadroomN\ (Haiku) and \statPBSonnetHeadroomN\ (Sonnet) as rows saturate, and the conditional edge is \statPBHaikuEdgeHeadroom\ with CIs \statPBHaikuEdgeHeadroomCI\ for Haiku against \statPBSonnetEdgeHeadroom\ with CIs \statPBSonnetEdgeHeadroomCI\ for Sonnet (same \statPBEdgeCIReps-resample task-cluster bootstrap and the same seed as the pooled intervals, so the two are computed over identical cluster draws; numerator and denominator are resampled together, so the denominator's own sampling variation is carried rather than held fixed).

Haiku's conditional intervals all sit strictly below zero, so its failure to improve is not an artifact of saturation. Sonnet's first two \emph{include} zero, so on tasks with room left we cannot reject a zero-or-slightly-positive edge in the early rounds, and only by round~3 is the negative sign unambiguous. That is a genuine qualification of the main claim. And the tier separation therefore survives conditioning, which is the more direct measurement of ``capability, not loop depth'' than the pooled series provides---a stronger model spends its rounds running out of headroom, a weaker one spends them failing to exploit the headroom it has.

\paragraph{A correction, not a clean unconditional edge.} It is the estimator over cells where a model call happened, so relative to the pooled series it is a correction and not merely a decomposition: excluding a back-filled round is declining to average over a non-event, not conditioning on an outcome. The paragraph above is why: counting the $q{=}1$ rows as rows the model failed to improve reads a non-event as a negative outcome, when the model was never asked. What it is still \emph{not} is a clean unconditional edge. Excluding those cells also excludes the rounds that would have followed a pass, and those rounds were never run, so no reweighting recovers them; the remaining population is additionally selected against rows whose earlier rounds went well, which is a real difference in population rather than a nuisance to adjust away.

\paragraph{The confidence-rated reading.} A round that leaves quality unchanged is closer to an \emph{abstention} than to an error, and \citet{schapire1999improved} give the standard treatment: with per-round weights $\omega_+$, $\omega_-$ and $\omega_0$ over measured cells, the optimally-weighted round contributes $Z_t = \omega_0+2\sqrt{\omega_+\omega_-}$ to the training-error product, and $Z_t<1$ whenever improvements outnumber regressions. On this arm the product over rounds~1--3 is \statPBHaikuAbstainZ\ (Haiku) and \statPBSonnetAbstainZ\ (Sonnet): the bound descends, slowly, which is what saturation looks like in boosting's own terms. The binary form $\exp(-2\sum_s\max(0,\tilde\gamma_s)^2)$ pins at $1$ instead, and that flatness is a property of a formulation with no representation for abstention rather than of the loop. Conditional on a round changing anything at all, the edge is strongly positive throughout (\statPBHaikuMoversEdge\ for Haiku, \statPBSonnetMoversEdge\ for Sonnet), reported for completeness only, since that denominator is both small and selected on the outcome being measured. The denominator shrinks with $t$ (to \statPBSonnetHeadroomNLast\ rows for Sonnet at round~3), so the late-round conditional intervals are correspondingly wide, and we read the round-3 point estimates as the least reliable of the six. Nothing here changes what the criterion requires: Thm.~\ref{thm:refinement} and Prop.~\ref{prop:dichotomy} need only a per-round improvement property, and $\bar\eta_t$ (\S\ref{sec:partB}) remains positive and shrinking under both tiers either way.

\begin{figure}[h]
\centering
\includegraphics[width=0.64\textwidth]{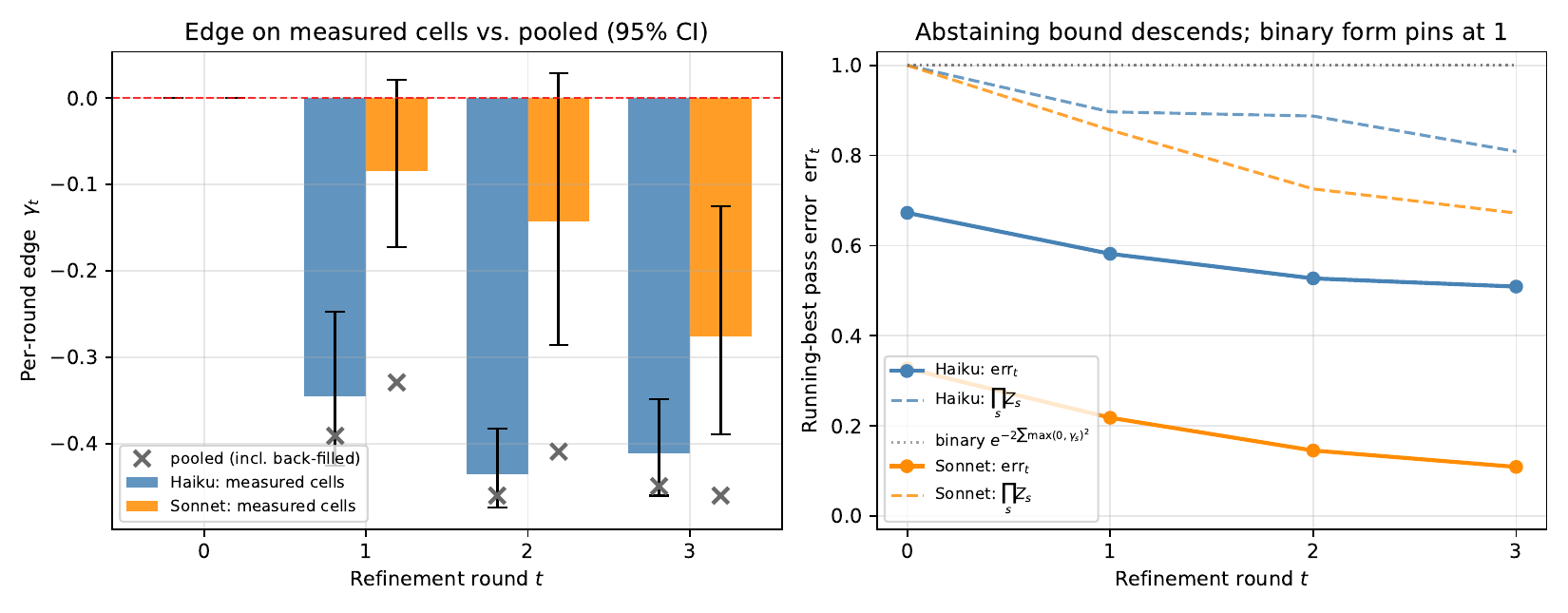}
\caption{Part~B, per-round edge on the \texttt{diagnostic} arm; cited from \S\ref{sec:partB}. \textbf{Left:} bars are $\tilde\gamma_t$ over cells where a model call happened, with 95\%~CIs (\statPBEdgeCIReps-resample cluster bootstrap over the \statPBK\ tasks); grey crosses are the pooled $\tilde\gamma_t$ over all cells, the gap being the back-filling described above. Round~0 has no predecessor, so no bar; Sonnet's first two intervals cross zero. \textbf{Right:} running-best error (solid) against the abstaining bound $\prod_s Z_s$ (dashed), which descends because an unchanged round is charged $Z_s<1$ rather than as an error; the binary form $\exp(-2\sum_s\max(0,\tilde\gamma_s)^2)$ (dotted) pins at $1.0$, an artifact of having no representation for abstention.}
\label{fig:llm-edge}
\end{figure}

\section{Part C: Additional Findings and Figures}
\label{app:partc}

\paragraph{Data availability.} Four anonymized session/commit datasets underlie Part~C, all in the supplementary repository and all written by the same sanitization pipeline (\path{scripts/anonymize_sessions.py}): \path{data/company_sessions.json} for the AI-tier sessions, \path{data/company_control_sessions.json} for the pre-AI human baseline of Finding~3, and \path{data/company_within_era_ai_sessions.json} with \path{data/company_within_era_nonai_sessions.json} for the two sides of the within-era contrast below. Real repository names, commit hashes, and author identities are replaced with opaque synthetic labels before release; timestamps are truncated to day granularity, developer labels are salted hashes, commit SHAs are counter-substituted, and only structural metrics (diff sizes, commit counts, model trailer, coarse dates) survive (full sanitization audit in the repository README).

The two label kinds are not symmetric, and the difference is easy to miss because the labels look alike. \emph{Developer} labels are $\mathrm{HMAC}(\text{salt}, \text{identity})$, so under one salt and one roster the same person carries the same \texttt{dev\_} label in every file, which is what makes the paired within-developer analyses below possible at all. \emph{Repository} labels are \emph{positional}: they are lettered per extraction run by HMAC rank over that run's repository list, so \texttt{repo\_A} denotes a \emph{different} repository in each file, and the label sets overlap without the repositories doing so. Merging datasets on \texttt{repo} would therefore fuse unrelated repositories into one random-effect level; \path{scripts/part_c_control_stats.py} namespaces the labels per file (\texttt{sim.namespace\_repos}) before any pooled model is fit, and third-party reanalysis must do the same.

Every Part~C number quoted in this paper is recomputed from those four files by \path{scripts/part_c_control_stats.py} into \path{data/part_c_control_stats.json}, which is the sole source for the macros this text expands; a companion \texttt{--check} gate fails the build if any quoted value drifts from the recomputed one.

\emph{Why the linkage compares instants, not strings.} The commit-to-session linkage compares parsed UTC instants, and the distinction is not cosmetic on this cohort: contributors span 11 distinct UTC offsets, so under a comparison of ISO-8601 timestamp \emph{strings} a commit written at \texttt{...T10:00:00+05:30} sorts after one written at \texttt{...T09:00:00-04:00} while in fact preceding it by four hours. Commits then fall outside their own session's window and go unlinked---124 control and 15 AI-tier commits, dropped from the per-session churn aggregates without any error being raised. Comparing instants takes unlinked commits to zero on both cohorts, and every Part~C figure is computed on that linkage.

\paragraph{Cohort composition, and the spread inside each tier label.} Finding~2 contrasts Claude Opus (stronger) against Claude Sonnet (weaker)---a different pair from Part~B's, where Sonnet is the \emph{stronger} of the two, so the two parts' ``capability contrast'' is not the same comparison. That pair is also not the whole cohort: \statPCTierExcluded\ sessions fall outside it, \statPCHaikuSessions\ Haiku and \statPCUnresolvedSessions\ with no resolvable trailer. And each label names a model \emph{family} rather than a checkpoint---the trailers record \statPCOpusVersions\ Opus and \statPCSonnetVersions\ Sonnet versions spanning \statPCVersionSpan, which the two-bin contrast pools---so within-family capability spread is uncontrolled, which is what the version-resolved re-test below addresses. The gap is also robust to the choice of test, giving Mann--Whitney $p{=}\statPCGapMWUP$ alongside the permutation $p$ quoted in \S\ref{sec:partC}.

\paragraph{Finding 4---Geometric churn-decay model captures the coarse shape (consistency check).} The geometric churn-decay (log-normal) model, fitted to consecutive commit pairs, yields $\rho{=}\statPCRho$ (95\%~bootstrap CI: \statPCRhoCI) and $\sigma{=}\statPCSigma$ (95\%~CI: \statPCSigmaCI). This CI excludes~$1.0$, giving some support that the decay rate is genuinely below unity rather than statistical noise around a random walk, consistent with Finding~1's rejected permutation null. A simulation calibrated to these parameters, together with a mixture parameter $\pi_{\mathrm{pure}}{\approx}\statPCPiPure$ accounting for pure-addition sessions (survivor ratio~$=1$), is best judged out of sample: a 70/30 train/test split gives held-out KS${=}\statPCKSHeldout$, $p{=}\statPCKSHeldoutP$ (train $n{=}\statPCTrainN$, test $n{=}\statPCTestN$). Unlike the smaller-sample fits, this is no longer a clean pass: across 8 train/test seeds the held-out $p$-value ranges \statPCKSHeldoutSeedRange. This is a downgrade: the expanded sample now typically shows a detectable discrepancy between held-out empirical and simulated survivor-ratio distributions (Figure~\ref{fig:emp3}), so the out-of-sample check should be read as marginal rather than confirmatory. The in-sample fit, always the more informative signal here, sharpens the same conclusion: the in-sample KS${=}\statPCKSInsample$ ($p{=}\statPCKSInsampleP$) formally \emph{rejects} exact distributional equality, so the model captures the coarse shape but is further from the true data-generating process once heterogeneity across \statPCRepos\ repositories and \statPCDevelopers\ developers is pooled into a single fit. That is exactly the kind of unmodeled repository/developer variation the multilevel model in Finding~2 quantifies directly ($\text{vc}_{\mathrm{repo}}{=}\statPCVarRepo$, $\text{vc}_{\mathrm{developer}}{=}\statPCVarDev$).

\paragraph{Finding 5---Heavy-tailed session length.} Session lengths are heavy-tailed: median~\statPCMedianLen\ commit (\statPCSingleCommitPct\ of sessions are a single commit), mean~\statPCMeanLen, maximum~\statPCMaxLen. Figure~\ref{fig:emp4} shows the empirical distribution with a log-normal fit, confirming that while most sessions are short, a substantial tail of long sessions drives mean complexity.

\paragraph{Finding 3, in full---pre-AI-adoption human baseline (descriptive, not a controlled comparison).} The headline numbers are in \S\ref{sec:partC}; this is the construction and the full caveat list. The cutoff is derived from data, not assumed: scanning every repository in the same organization for the earliest commit anywhere carrying \emph{any} AI-coding-tool co-author trailer (Claude, Copilot, Cursor, Codex, ChatGPT, and others) gives 2025-09-18; subtracting a 3-year safety margin sets the cutoff at \statPCCtrlUntil. The \statPCRepos\ repositories used for the AI-tier findings turn out to have essentially no history before that date---they are recently-created repositories, a byproduct of the same qualification rule that selected them for the AI-tier analysis in the first place. That rule admits a repository on ${\ge}15$ Claude-co-authored commits counted across \emph{all} branches, whereas extraction then reads only each repository's \emph{dominant} branch; a released per-repository commit count can therefore sit below the qualifying threshold without contradicting it---two denominators, both correct, and \texttt{scripts/discover\_cohorts.py} derives each reproducibly. The control set consequently comes from a different, older, disjoint cohort of \statPCCtrlRepos\ repositories in the same organization with genuine pre-cutoff activity.

\emph{Developer identity is a curated roster, not a heuristic.} Git author identities do not correspond to people: one person appears as a GitHub-noreply handle, a work address, a personal address, and sometimes a contractor account. Merging them with a fuzzy-matching heuristic and reporting the result as a contributor count does not work: the heuristic under-merges, inflating the headcount and splitting individuals across several labels. This analysis therefore uses a human-reviewed roster. \texttt{scripts/resolve\_author\_aliases.py --emit-candidates} writes a table of every raw identity together with the identity evidence behind its provisional cluster; a human marks each row keep-or-drop and assigns a person label; \texttt{--from-reviewed} converts that into a \emph{strict} allowlist, so an identity absent from the roster is excluded from the cohort rather than merely left unmerged. The review covered 155 raw git identities and resolved them to the \statPCCtrlDevelopers\ people reported here, dropping one as unattributable. Merging aliases also fuses sessions that a heuristic would split across a single person's two identities, which is why the control session count sits below what an uncurated extraction reports even though the commit count does not.

\textbf{Curation was blind to outcomes by construction.} The candidate table carries no survivor-ratio, churn, or $\psi$ column and is assembled from identity evidence alone. This matters more than the headcount: a roster curated while looking at per-developer results would stop being a control, which is the one failure mode that would make this comparison circular rather than merely confounded. The choice has a plain cost: the roster maps real names and email addresses to person labels, so it cannot be released, and a third party therefore cannot rebuild this cohort from the artifacts we publish. What we can publish is the audit trail as counts (155 raw identities in, one dropped, \statPCCtrlDevelopers\ people out) and the salted \texttt{dev\_}-prefixed labels in the released data, which are now stable across both cohorts and so support the paired analyses below.

To keep the two eras comparable, the control window (\statPCCtrlSince\ to \statPCCtrlUntil) is sized to match the AI-tier data's own 315-day observation span. Applying the extraction pipeline with \texttt{require\_ai\_trailer=False} (keeping all commits, still excluding any matching an AI-tool trailer as a defense-in-depth check even though none are expected before the cutoff) and grouping consecutive commits by the same developer into sessions yields \statPCCtrlSessions\ control sessions (\statPCCtrlCommits\ commits). That is far more than the \statPCSessions\ AI-tier sessions, since we did not curate or subsample repositories, only bound the time window. The resulting mean non-pure survivor ratio is $\psi_{\mathrm{human}}{\approx}\statPCCtrlTau$, and the churn-decay rate is $\rho_{\mathrm{human}}{=}\statPCCtrlRho$ (95\%~CI \statPCCtrlRhoCI), similar in shape to the AI-tier $\rho{\approx}\statPCRho$, so the geometric front-loading pattern itself (Finding~1) is not obviously AI-specific.

The survivor-ratio \emph{level} is a different story: $\psi_{\mathrm{human}}$ is far below both $\psi_{\mathrm{Opus}}$ and $\psi_{\mathrm{Sonnet}}$ (gap${=}\statPCCtrlGapOpus$ vs.\ Opus at $p\statPCCtrlGapOpusP$ and $\statPCCtrlGapSonnet$ vs.\ Sonnet at $p\statPCCtrlGapSonnetP$; both sit at the resolution floor $2/(\statPCPermReps{+}1)$ of the permutation test, meaning no resample reached the observed gap, not that the $p$ is estimated at that value), and a multilevel model with repository and developer as crossed random intercepts confirms the gap survives adjustment (adjusted coefficient $\statPCCtrlMLOpusCoef$, $p{=}\statPCCtrlMLOpusP$ for Opus vs.\ human; $\statPCCtrlMLSonnetCoef$, $p{=}\statPCCtrlMLSonnetP$ for Sonnet vs.\ human). Unlike Finding~2's AI-tier comparison, the tier effect here is \emph{larger} than the repository and developer variance components, not smaller.

One comparability objection would otherwise explain the gap without any appeal to era at all: the two cohorts have very different session-length profiles. \statPCCtrlSingleCommitPct\ of control sessions are a single commit, against \statPCSingleCommitPct\ of AI-tier sessions, and a one-commit session's survivor ratio is a much coarser quantity than a ten-commit session's. The objection does not bite. Restricted to multi-commit sessions on both sides, the mean non-pure survivor ratio is \statPCCtrlTauMulti\ for the control against \statPCTauMulti\ for the AI tier---a gap of the same sign and comparable size to the pooled one, so the level difference is not an artifact of the control being singleton-heavy.

We are explicit about what this cannot show: the human and AI-tier samples differ by era, not just by who wrote the code---codebase maturity, review norms, tooling, and team composition all differ across the roughly three years separating them, and any of these could produce a survivor-ratio gap with no AI involvement at all. This is a sanity check, not a controlled experiment, and it supports no causal claim about AI's effect on the dynamics we measure; it does say that the tier-agnostic pattern of Finding~2 is not simply an artifact of any two groups of commits looking alike on this proxy---something in the AI-assisted sessions (whatever the cause) does look different from the pre-AI baseline, even as the AI tiers do not differ from each other.

\paragraph{Within-era AI vs.\ non-AI: the one contrast that removes both confounds (suggestive, not established).} Finding~3 concedes two confounds it cannot address---the people differ, and the eras differ. Both can be removed at once by staying inside the AI era and splitting on the commit trailer instead: take the window \statPCWithinSince\ to \statPCWithinUntil, keep the same roster on both sides, and separate commits that carry a Claude co-author trailer from those that do not. The two sides then share developers, calendar window, codebase maturity, review norms, tooling, and seniority; they differ in whether the work was delegated to an agent. \citet{pansuriya2026predicting} draw the same human/agent contrast on pull-request acceptance and review effort, which is the closest existing comparison to this design, though on review outcomes rather than churn dynamics. This yields \statPCWithinAISessions\ AI-assisted sessions (\statPCWithinAICommits\ commits) against \statPCWithinNonAISessions\ non-AI sessions (\statPCWithinNonAICommits\ commits), with \statPCWithinDevelopers\ developers present on both sides. Pooled, the mean non-pure survivor ratio is \statPCWithinTauAI\ for AI-assisted work against \statPCWithinTauNonAI\ for the rest---the same direction as Finding~3's era contrast, with neither of its two confounds.

Pairing within developer, so that each person is compared only against themselves, gives Table~\ref{tab:within-era-paired}. We report it as suggestive, not established. First, the \emph{sign test is null at every floor}: counting how many developers move in which direction never reaches significance. Second, Wilcoxon does reach it at the two lowest floors, but only because it weights magnitudes as well as directions---at the ${\ge}\statPCWithinFloorB$-commit floor the deltas run from \statPCWithinFBDeltaMin\ to \statPCWithinFBDeltaMax, so the positive movements are several times larger than the negative ones. That asymmetry is an assumption the sign test deliberately declines to make, and quoting the significant $p$-value without it would misrepresent which assumption carries the result. Third, this is eight tests across four floors and two statistics, and the two significant values sit at the two lowest floors---the ones retaining the thinnest developers. Fourth, and least fixable by any amount of data: \emph{task selection is uncontrolled}. This compares work developers chose to delegate against work they chose to keep, and there is every reason to expect those to differ in kind.

What the design does buy, and what the pre-AI pairing below does not, is that the direction no longer rests on thin data. It survives the ${\ge}\statPCWithinFloorD$-commit floor (\statPCWithinFDHigher\ of \statPCWithinFDPairs\ pairs, median $\Delta\psi{=}\statPCWithinFDMedian$), and the best-measured pair in the cohort---the developer with the most commits on \emph{both} sides, \statPCWithinTopVolNonAI\ non-AI against \statPCWithinTopVolAI\ AI-assisted---moves with the majority at $\Delta\psi{=}\statPCWithinTopVolDelta$. Nothing here licenses a causal claim: the trailer is not randomly assigned, and the task-selection objection stands untouched.

\begin{table}[h]
\centering\small
\caption{Part~C: within-era AI vs.\ non-AI survivor ratio, paired within developer, at four minimum-commits-per-side floors. ``AI higher'' counts developers whose AI-assisted sessions show the higher mean non-pure survivor ratio. The sign test is null at every floor; Wilcoxon reaches significance only where the magnitude asymmetry is largest.}
\label{tab:within-era-paired}
\begin{tabular}{lccccc}
\toprule
Commit floor & Pairs & AI higher & Sign $p$ & Wilcoxon $p$ & Median $\Delta\psi$ \\
\midrule
none & \statPCWithinFZeroPairs & \statPCWithinFZeroHigher & \statPCWithinFZeroSignP & \statPCWithinFZeroWilcoxP & \statPCWithinFZeroMedian \\
${\ge}\statPCWithinFloorB$ & \statPCWithinFBPairs & \statPCWithinFBHigher & \statPCWithinFBSignP & \statPCWithinFBWilcoxP & \statPCWithinFBMedian \\
${\ge}\statPCWithinFloorC$ & \statPCWithinFCPairs & \statPCWithinFCHigher & \statPCWithinFCSignP & \statPCWithinFCWilcoxP & \statPCWithinFCMedian \\
${\ge}\statPCWithinFloorD$ & \statPCWithinFDPairs & \statPCWithinFDHigher & \statPCWithinFDSignP & \statPCWithinFDWilcoxP & \statPCWithinFDMedian \\
\bottomrule
\end{tabular}
\end{table}

\paragraph{Pre-AI within-developer pairing: evidence \emph{against} the between-group gap.} Because both cohorts now share a salt and a roster, a person carries the same \texttt{dev\_} label in the pre-AI and AI-era datasets, which makes it possible to pair Finding~3's era contrast within individuals and so remove its person confound---though not its era confound, since the two sides remain three years apart. \statPCPairedPairs\ people have non-pure sessions on both sides. We report this result because it undercuts the between-group gap, not because it supports it.

Taken at face value the direction agrees with Finding~3: \statPCPairedHigher\ of \statPCPairedPairs\ developers show a higher survivor ratio in their AI-era sessions, median $\Delta\psi{=}\statPCPairedMedian$, though the sign test does not reach significance ($p{=}\statPCPairedSignP$; Wilcoxon $p{=}\statPCPairedWilcoxP$, subject to the same magnitude-asymmetry caveat as above). That agreement is carried by developers with little data, and it does not survive requiring some. At a ${\ge}\statPCPairedFloorD$-commit floor on both sides only \statPCPairedFDPairs\ pairs remain, the majority disappears entirely (\statPCPairedFDHigher\ of \statPCPairedFDPairs, sign $p{=}\statPCPairedFDSignP$), and the median delta collapses to $\statPCPairedFDMedian$---an order of magnitude smaller than the pooled between-group gap of \statPCCtrlGapOpus. The single best-measured pair does move with the majority (\statPCPairedTopVolCtrl\ pre-AI against \statPCPairedTopVolAI\ AI-era commits, $\Delta\psi{=}\statPCPairedTopVolDelta$); among the four best-measured pairs the direction splits evenly, so nothing here corroborates a level shift of the size Finding~3 reports.

A structural asymmetry compounds this and is the reason we report direction only, with no effect size or interval. The two datasets do not have equivalent inclusion rules: the control keeps \emph{every} commit a person made in the pre-AI window, whereas the AI-tier dataset keeps only their Claude-co-authored ones. Each pair therefore compares different \emph{slices} of one person's work, which is legitimate for a per-session shape metric like the survivor ratio but makes any volume or productivity read across the two sides meaningless---the worst per-developer imbalance here is roughly \statPCPairedMaxRatio-to-1, which reflects how little of that person's output was delegated, not a collapse in what they produced. Read together with the within-era contrast above, the \emph{within-person} evidence for a survivor-ratio difference is weak, and weakest exactly where it is best measured.

\paragraph{Finding 2's four robustness checks in full.} The main text counts the four checks that reinforce the capability-gap null; this is the detail behind them. \emph{(i)~Unmeasured confounding.} An E-value sensitivity analysis \citep{vanderweele2017evalue} returns a point E-value of $\statPCEvaluePoint$ and a confidence-limit E-value of $\statPCEvalueCI$: an unmeasured confounder would need associations of only that magnitude with both tier assignment and survivor ratio to account for the observed gap. The direction of this check needs care: a low E-value is unremarkable for a null, and we report it to bound how much a confound could be \emph{hiding} an effect, not as evidence that one exists. \emph{(ii)~Task type.} Stratifying on the coarse task-type category and pooling within strata (\statPCStrata\ strata) barely moves the estimate: adjusted gap $\statPCStratGap$ against a naive $\statPCGap$ ($p{=}\statPCStratP$), so the tier/task-type entanglement Limitations flags does not by itself manufacture the null.

\emph{(iii)~Decay rate, not only level.} The survivor ratio is a level; the churn-decay rate is a shape, and the tiers are indistinguishable on that too ($\rho_{\mathrm{Opus}}{=}\statPCChurnRhoOpus$ vs.\ $\rho_{\mathrm{Sonnet}}{=}\statPCChurnRhoSonnet$, gap \statPCChurnGap, $p{=}\statPCChurnGapP$, over the $n{=}\statPCChurnGapN$ sessions long enough to fit a rate). \emph{(iv)~Repository and developer structure.} A multilevel model with crossed repository and developer random intercepts gives an adjusted tier coefficient of $\statPCMultilevelCoef$ (95\%~CI \statPCMultilevelCI, $p{=}\statPCMultilevelP$), while the variance components ($\text{vc}_{\mathrm{repo}}{=}\statPCVarRepo$, $\text{vc}_{\mathrm{developer}}{=}\statPCVarDev$) are an order of magnitude larger than any plausible tier effect. That last comparison is the substantive content of the null: in this data, which repository and which developer a session belongs to matters considerably more for its survivor ratio than which model tier wrote it. None of the four addresses model-version pooling, which is why the re-test below exists.

\paragraph{Finding 2 re-tested on an ordered capability scale (the null survives, but it is a \emph{bounded} null).} Finding~2's tier variable is coarse in a way that could manufacture its own null. The label ``Opus'' collapses \statPCOpusTrailers\ distinct commit-trailer strings and ``Sonnet'' \statPCSonnetTrailers, so \statPCOpusVersions\ Opus versions and \statPCSonnetVersions\ Sonnet versions are pooled into two bins---Opus~4.5 sits in the same bin as Opus~4.8. The within-bin capability spread is plausibly comparable to the between-bin contrast the test is trying to detect, and that would attenuate a real gap toward exactly the null we observe. This is also the first explanation a reader will reach for on seeing Part~B's large, clean capability effect (partial $\eta^2{=}\statPBModelEta$ on a two-model comparison) beside Part~C's null, so we test it.

We therefore re-run the discriminator as a monotone trend over an ordered capability scale instead of a two-bin contrast. Each session is assigned its plurality \texttt{(family, version)} cell, giving \statPCCapCells\ cells over \statPCCapSessions\ sessions. That total coincides with the version count just quoted and so invites a wrong reading: the cells are \statPCCapCellsOpus\ Opus, \statPCCapCellsSonnet\ Sonnet, and \statPCCapCellsHaiku\ Haiku---\emph{not} the \statPCOpusVersions-plus-\statPCSonnetVersions\ the pooling above would suggest. Opus~\statPCCapAbsentOpusVersion\ appears in the cohort's trailers but is the plurality cell of no session, so it contributes none; and Haiku, which the tier contrast never sees at all, contributes one. The cells are ranked two ways, each serving as the other's sensitivity analysis: \emph{family-major}, placing every Sonnet cell below every Opus cell and ordering by version within a family, and \emph{version-major}, ordering strictly by version number across families. The null survives both. Spearman rank correlation between capability and survivor ratio is $r_s{=}\statPCCapFamilyRho$ ($p{=}\statPCCapFamilyP$) under the family ordering and $r_s{=}\statPCCapVersionRho$ ($p{=}\statPCCapVersionP$) under the version ordering---written $r_s$, not $\rho$, since $\rho$ is the churn-decay factor throughout Part~C; adding capability rank as a linear term to the same crossed repository/developer multilevel specification used in Finding~2 gives per-step coefficients of $\statPCCapFamilyCoef$ ($p{=}\statPCCapFamilyCoefP$) and $\statPCCapVersionCoef$ ($p{=}\statPCCapVersionCoefP$). Version pooling is therefore \emph{not} the explanation for Finding~2's null: resolving the labels to versions and testing for a trend across them recovers no effect either.

That conclusion is a bounded null rather than a demonstrated zero, and needs the same caveat Part~B's underpowered nulls receive. The family-ordering adjusted interval is \statPCCapFamilyCI\ per rank step; accumulated over the \statPCCapSteps\ steps separating the weakest cell from the strongest, it still admits a total swing of up to \statPCCapSwing, comparable in size to the human-vs-AI gap of \statPCCtrlGapOpus\ that Finding~3 does detect. The version ordering bounds it no more tightly (\statPCCapVersionCI\ per step), so neither ordering converts the null into a narrow one; they agree on the sign being undetermined, not on the effect being small. What we can say is that no capability effect \emph{large enough to dominate} the survivor ratio survives this test; what we cannot say is that the effect is zero. Read together with Part~B's large measured capability effect on the same kind of contrast, this is a fact about the proxy rather than about the models: the survivor ratio separates \emph{whether an agent wrote the code} ($\psi_{\mathrm{human}}$ sits far below both tiers, and that gap survives crossed repository and developer random effects) but not \emph{which} agent wrote it, being insensitive to precisely what a solve rate detects.

Two features of the per-cell breakdown explain why $r_s$ sits so near zero. The per-cell means are \emph{not monotone} in either ordering: the newest Sonnet cell ($\psi{=}\statPCCapTauSonnetNew$) sits below an older Sonnet cell ($\psi{=}\statPCCapTauSonnetMid$). So the trend statistic is averaging over non-monotone movement rather than detecting a weak monotone one. And the single Haiku cell sits at $\psi{=}\statPCCapTauHaiku$ on only \statPCCapNHaiku\ sessions, essentially at the human baseline of \statPCCtrlTau: suggestive of a step at the very bottom of the capability range, on far too little data to claim one.

\begin{figure}[h]
\centering
\includegraphics[width=0.47\textwidth]{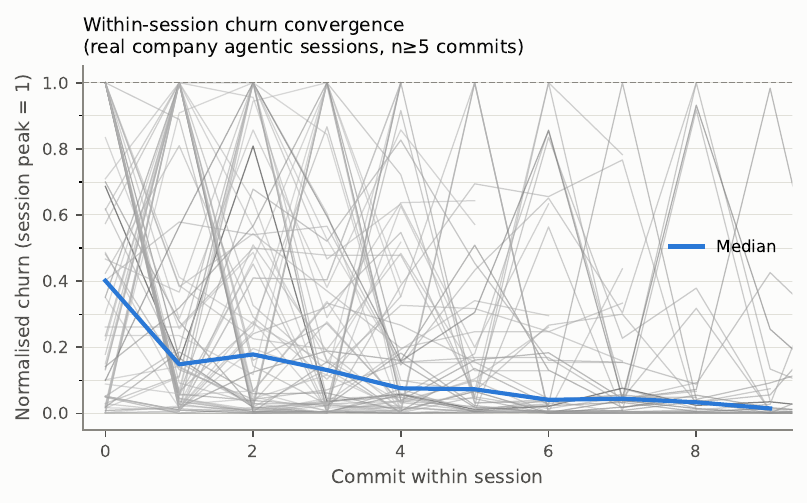}
\includegraphics[width=0.47\textwidth]{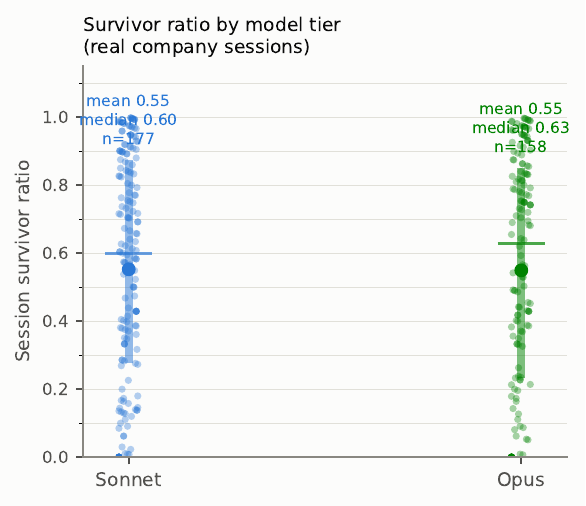}
\caption{Part~C production sessions, cited from Findings~1 and~2 (\S\ref{sec:partC}). \textbf{Left (Finding~1):} normalised within-session churn trajectories ($n{\ge}5$ commits); churn concentrates in the first commit, decaying geometrically thereafter. \textbf{Right (Finding~2):} survivor ratio by model tier (dot~=~mean, bar~=~median)---the two tiers are statistically indistinguishable, a null that holds after adjusting for task type, repository, and developer.}
\label{fig:emp}
\end{figure}

\begin{figure}[h]
\centering
\includegraphics[width=\textwidth]{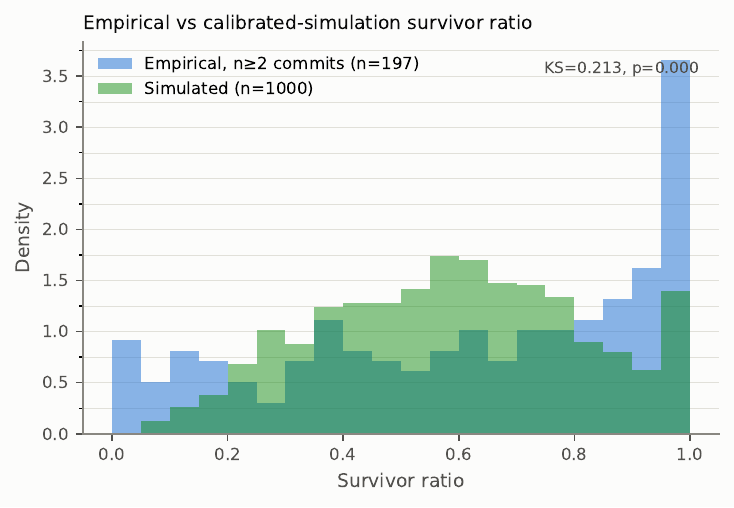}
\caption{Empirical vs.\ simulated survivor-ratio distribution. Simulated sessions use the fitted geometric churn-decay (log-normal) model ($\rho{=}\statPCRho$, $\sigma{=}\statPCSigma$) with a mixture component for pure-addition sessions ($\pi_{\mathrm{pure}}{=}\statPCPiPure$). The simulation captures the coarse shape but the annotated in-sample KS${=}\statPCKSInsample$ ($p{=}\statPCKSInsampleP$) formally rejects exact equality, and the held-out KS${=}\statPCKSHeldout$ ($p{=}\statPCKSHeldoutP$, seed-sensitive across \statPCKSHeldoutSeedRange) is now marginal rather than a clean pass---the gap likely reflects repository/developer heterogeneity pooled into a single fit (see Finding~4).}
\label{fig:emp3}
\end{figure}

\begin{figure}[h]
\centering
\includegraphics[width=0.7\textwidth]{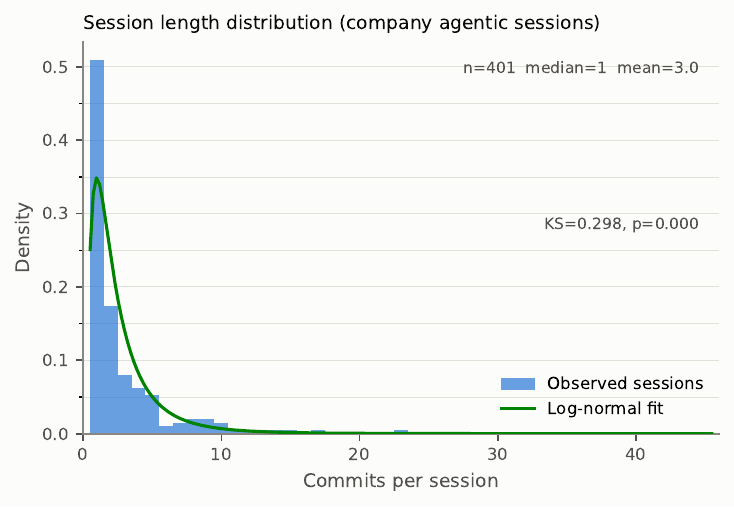}
\caption{Session-length distribution (commits per session) with a log-normal reference overlay (descriptive; a KS test rejects exact fit). Median~\statPCMedianLen, mean~\statPCMeanLen, max~\statPCMaxLen. The heavy tail reflects long sessions that drive most of the structural change in the repositories.}
\label{fig:emp4}
\end{figure}

\section{Extended Discussion}
\label{app:discussion}

\subsection{Related Work: Empirical Systems and Test-Time Compute}
\label{app:related}

These are the empirical lines the body's positioning refers to (\S\ref{sec:related}). Multi-agent code generation (ChatDev \citep{qian2023chatdev}, MetaGPT \citep{hong2023metagpt}, AgentCoder \citep{huang2023agentcoder}) and self-refinement work \citep{madaan2023self, welleck2023generating, shinn2023reflexion} are empirical, without a structural criterion of this kind. Test-time compute scaling \citep{brown2024monkeys, wang2023selfconsistency, snell2024scaling} shows empirically diminishing returns \citep{wu2024inference}, which our refinement game predicts from first principles; raising the ceiling has required weight-level training \citep{deepseek2025r1}. A 17-hour, billion-token autonomous coding loop on a fixed harness \citep{khan2026clinersi} illustrates the same ceiling practically.

\subsection{Could Decomposition Realize the Protocol Implicitly?}
\label{app:decomposition}

A natural objection to \S\ref{sec:framework}'s scope claim: a task decomposes into many sub-specifications (e.g., a ticket's individual \texttt{FAIL\_TO\_PASS} tests), a refinement loop concentrates successive attempts on the still-failing ones, and one might read this as an implicit reweighting $D_t$ with an implicit confidence vote $\alpha_t$. We grant the spirit: Part~B's diagnostic feedback literally redirects the agent onto residual failures, a primitive attention-reweighting over sub-specs. But three features the transferred theorems actually require remain absent.

First, AdaBoost reweights a \emph{fixed} population, whereas decomposition \emph{generates} sub-specs on the fly; in the language of Prop.~\ref{prop:dichotomy} this dynamic expansion is closer to the class growth of regime~(ii) than to fixed-distribution reweighting. Second, ``attend to the failing test'' is a $0/1$ shift, not the multiplicative-exponential update $w_i\leftarrow w_i\exp(-\alpha_t[\cdot])$ whose $Z_t$-factors telescope into the $e^{-2\gamma^2T}$ bound; a generic failure-focus yields ``the agent fixes tests at the rate its competence allows,'' not the boosting rate. Third, and most structurally, AdaBoost \emph{retains} all $T$ learners and combines them by the weighted vote $f=\sum_t\alpha_t h_t$, whereas a refinement loop \emph{composes} patches ($R_T\circ\cdots\circ R_1$): each patch subsumes its predecessor rather than standing beside it to be voted on. There is therefore no additive margin for an implicit $\alpha_t$ to weight; the closest genuine analog, best-of-$n$ patch selection, is a within-round argmax, not a cross-round vote.

We treat the mapping as boosting-\emph{inspired}, not boosting-\emph{exact}: deploying the exact protocol would restore the transferred theorems by construction, but at the cost of the fidelity we aim to preserve, since the real loops we set out to model do not run it. That tradeoff is inherent to a descriptive study of existing harnesses, not to the framework itself; \S\ref{app:future} outlines what a harness that implements the Protocol outright, rather than approximates it, would look like as a design principle in its own right.

\subsection{Guidance for Coding in the Wild (full)}
\label{app:guidance}

The main text gives the two headline recommendations; the full set of five follows, derived from the measured edge of Part~B and the churn dynamics of Part~C. We state these as working guidance consistent with our evidence and the inference-scaling literature, not as proven prescriptions: they rest on a $K{=}\statPBK$ single-vendor, two-tier ablation (Part~B) and observational production data (Part~C), and should be read as hypotheses a practitioner can cheaply test on their own workload rather than laws.

\begin{enumerate}[leftmargin=*,itemsep=3pt]
    \item \textbf{Spend on the first attempt, not the refinement loop.} The single most actionable finding is that re-invoking an agent on its own output usually changes nothing: \statPBHaikuFlat/\statPBSonnetFlat\ of trajectories never move after round~0, genuine no-ops run \statPBHaikuGenuineTieRange/\statPBSonnetGenuineTieRange\ per round (Fig.~\ref{fig:llm-edge}), and per-round improvement decays to \statPBHaikuEta\ and \statPBSonnetEta. Almost all of the value is realized on the first attempt, so better context, retrieval, and model choice for round~0 dominate a longer loop. The reassuring half of the same measurement: extra rounds are rarely \emph{harmful} either (at most \statPBRegressMax\ of rounds lower quality), so a loop is wasteful rather than dangerous---and for a strong model the early rounds are the ones where an edge cannot be ruled out (Appendix~\ref{app:edge-decomposition}), which is where any budget for refinement should go.
    \item \textbf{The verifier is the gradient---make it fast, executable, and specific (theoretical argument, not yet confirmed by our own ablation).} In boosting terms the verifier supplies the residual-error signal the loop descends on; in engineering terms your test suite, type checker, and runtime errors \emph{are} the learning signal, the same lesson that process- and outcome-level verifiers carry in the reasoning literature \citep{lightman2023verify}. Part~B's own controlled ablation does \emph{not} confirm this: diagnostic test-failure feedback did not significantly outperform blind re-sampling or independent redraws at $K{=}\statPBK$. Late solves occur at broadly similar rates under all three modes, and the feedback-mode effect is a small, underpowered null (Cohen's $f{=}\statPBModeCohenF$, \statPBModePower\ power). What our data \emph{does} support is the weaker claim that verification quality is a lever worth caring about in principle (weak or absent verification structurally cannot supply a residual-error signal at all), not the stronger, specific claim that diagnostic feedback beats blind resampling in practice. We test only an agent critiquing its own output against a narrow test-failure signal here, not an architecturally separate reviewer model reading the same signal (Limitations); the two are not interchangeable, and our edge measurement should not be read as evidence against the latter.
    \item \textbf{Reach for capability before orchestration.} The frozen-weight ceiling (Cor.~\ref{cor:frozen-ceiling}) says a given model saturates at a quality $V^*(W)$ no loop can pass, and Part~B shows the stronger model wins by clearing more tasks on the \emph{first} try, not by refining better. Elaborate multi-agent refinement pipelines built to compensate for an under-powered model will hit that ceiling; upgrading the model is often the cheaper path to higher quality on hard tasks.
    \item \textbf{Distinguish exploring under the ceiling from raising it.} Two interventions are commonly both called ``scaffolding'' but behave differently. Re-sampling, ``think again,'' and more rounds \emph{explore within} a fixed reachable class and saturate at $V^*(W)$; giving the agent genuinely new reach (additional tools, code search/retrieval, the ability to run tests, or task decomposition) can \emph{expand} the reachable class $\mathcal{H}_t(C)$ and lift the ceiling (\S\ref{app:rsi}). When a loop has stalled, ask which kind of change you are making: only the second kind can move a hard plateau. This is also why Part~B's null feedback-mode effect does not cut against the theory: independent, blind, and diagnostic feedback are all within-class exploration variants (same weights, same tools, only the framing of feedback changes), so the theory's own distinction does not anticipate a large gap between them.
    \item \textbf{Steer early; stop on green.} Production churn is front-loaded (Part~C, $\rho{\approx}\statPCRho$): the first commits make the structural decisions and later ones only polish, so an approach that is wrong at the outset is rarely repaired by subsequent edits---review and redirect early, when it is cheap. At the other end, stop as soon as the verifier passes: carrying the best passing artifact forward and halting on success avoids the failure mode where a loop instructed to ``keep improving'' regresses an already-correct solution. Our stopping rule (Algorithm~\ref{alg:boost}) is a deterministic margin gate ($\epsilon_t \ge \tfrac12$); a complementary line treats stopping as probabilistic inference---cost-sensitive sequential hypothesis testing over a belief about correctness \citep{papamarkou2026bayesian}---which is the natural refinement when verification is expensive and verifiers are informative but imperfect.
\end{enumerate}

A final governance takeaway, expanded in \S\ref{sec:discussion} and Appendix~\ref{app:rsi}: the quantity to monitor is not whether model weights are frozen but whether the agent's scaffold is self-expanding.

\subsection{Implications for Recursive Self-Improvement (full)}
\label{app:rsi}

The Stationarity Dichotomy (Prop.~\ref{prop:dichotomy}) states that iterative self-modification behaves in one of two ways. (i)~It saturates into strict diminishing returns whenever the agent's \emph{effective hypothesis class}---the set of patches it can reach---is held fixed and quality is bounded. (ii)~It escapes diminishing returns only if that class keeps expanding enough to sustain a positive edge. This converts the informal debate over an ``intelligence explosion'' into a concrete, checkable question: is the system's reachable class stationary?

\paragraph{The whole-system view.} Picture a craftsman whose hands are fixed but whose workshop is not: a sharper tool or a better-organised bench lets the same hands reach further, yet no rearrangement of the workshop makes the hands themselves more dexterous. It helps to view the self-improving system in the same way, as a pair $\mathcal{S} = (W, C)$, where $W$ denotes the (frozen) parameters of the underlying model and $C$ is the mutable code and scaffolding it operates (e.g., prompts, tools, retrieval, verifiers, and orchestration). Agentic coding is precisely the setting in which the system may rewrite $C$ but not $W$.

It is tempting to infer that freezing $W$ guarantees the safe, saturating regime~(i). Under the model we now make explicit, this is \emph{false}. We posit that the effective hypothesis class $\mathcal{H}_t(C)$ from which patches are drawn depends on $C$, and that rewriting the scaffold can expand $\mathcal{H}_t(C)$ without touching a single weight (assumption~(M1) of Cor.~\ref{cor:frozen-ceiling}). A better verifier, tool, or decomposition can unlock patches the previous configuration could never propose. Concretely: an agent stuck on a flaky integration test might write and run a small isolation script mid-session, a debugging tool it did not have at round~0. Once written, that script is available to every later round, so patches the agent could not previously validate (because failures were masked by flakiness) become checkable: $\mathcal{H}_t(C)$ grows to include edits unreachable at $t{=}0$, with $W$ untouched throughout. This is the modeling premise behind recent calls to treat the harness as a first-class, separately optimizable layer \citep{he2026harness}. That configuration is not hypothetical. Inspect, the open evaluation framework published by the UK AI Security Institute \citep{aisi2024inspect}, ships a \texttt{react()} agent that runs a reason-act-observe loop with \texttt{bash()} and \texttt{python()} tools inside a Docker sandbox, and can host external coding agents (Claude Code, Codex CLI, Gemini CLI) directly---so the agent that writes and runs its own diagnostic mid-run is the default arrangement of a harness built for frontier evaluation, not a thought experiment. Two consequences pull opposite ways. Constructively, the scaffold there is a \emph{declarative, versioned} artifact---a solver chain, a tool list, a Dockerfile---which is the precondition for auditing a reachable set at all: stationarity of $\mathcal{H}_t(C)$ is not checkable unless the scaffold is first enumerable. Against that, the declared unit of measurement is a \emph{task} (dataset, solver, scorer), and nothing in that unit records whether the reachable set expanded while the task ran. The scaffold is thus versioned as experimental setup rather than tracked as the property whose expansion (M1) identifies, the same gap the published frameworks leave at the policy layer (below), one level down in the instrument. Such self-modification violates the stationarity assumption behind Prop.~\ref{prop:dichotomy}(i), so a frozen-weight system can still exhibit the non-monotone, jumpy trajectory of regime~(ii).

Frozen weights plausibly impose an ultimate ceiling. We assume there is a best quality $V^*(W)$ achievable by \emph{any} scaffold the model can construct within its reachable program space (assumption~(M2) of Cor.~\ref{cor:frozen-ceiling}). Once scaffold rewrites can no longer expand $\mathcal{H}_t(C)$, the system falls back into regime~(i) and saturates at a value no greater than $V^*(W)$. Genuinely unbounded RSI therefore requires raising $V^*(W)$ itself---modifying the weights or architecture---together with co-evolving the evaluation. \citet{hernandez2019measuring} document how much of that raising has historically come from algorithmic progress rather than scale alone, which is why $V^*(W)$ is best read as a property of a particular checkpoint rather than a constant of the field. The practical upshot for anyone exploring a self-improving system is a diagnostic: determine whether the reachable class is effectively fixed (predictable, saturating) or expanding (potentially runaway), and do \emph{not} treat frozen weights as sufficient evidence of the former. Neither (M1) nor (M2) has been empirically measured; the conclusion is conditional on both holding.

The generalization bound (Thm.~\ref{thm:generalization}) sharpens the same point from the evaluation side: an AI optimizing against a static set of $N$ specifications overfits, so pushing back the $O(\sqrt{d/N})$ ceiling requires dynamically expanding the test suite. Escaping saturation thus demands expansion on \emph{both} sides---the hypothesis class and the tests.

Concretely, the framework identifies the mechanisms that move a system from regime~(i) into regime~(ii):
\begin{itemize}[leftmargin=*,itemsep=2pt]
    \item \textbf{Self-Evolving Tests and Cost Functions:} Instead of a fixed dataset of specifications, the AI could generate increasingly complex tests (an automatic curriculum, as in \citealp{wang2023voyager}) or evolve its own structured loss functions. This continuously shifts the generalization target and prevents plateauing.
    \item \textbf{Goal-Setting Exploration:} Rather than solely minimizing residual error on existing tasks, agents could autonomously set new, orthogonal objectives. This would force structural shifts in the codebase, potentially opening up new capability dimensions beyond simple patch refinement.
    \item \textbf{Reward Hacking and Awareness:} Explosive trajectories might involve the AI recognizing and exploiting proxies in its evaluation metrics (reward hacking, a recognized failure mode of proxy objectives \citealp{amodei2016concrete, skalse2022defining}). A robust recursive protocol must be aware of its own margin distribution and actively penalize the artificial optimization of the structured distance $L(Y, \hat{Y})$ when it degrades true functional quality.
\end{itemize}
We read each mechanism as, in the language of Prop.~\ref{prop:dichotomy}, a way of expanding the effective hypothesis class $\mathcal{H}_t(C)$ or the evaluation target. Absent them, the stationarity assumption holds and Prop.~\ref{prop:dichotomy}(i) forces the system to plateau.

\paragraph{Relation to test-time compute, and why test-time training is a different case.} The bounded regime our theory identifies matches the empirical inference-scaling literature: repeated sampling and compute-optimal test-time allocation raise quality with diminishing returns against a fixed model \citep{brown2024monkeys, snell2024scaling, wu2024inference}. In our terms, scaffold- and sampling-level effort explores within a fixed reachable class and saturates at $V^*(W)$; moving the ceiling has required weight-level training such as reinforcement learning for reasoning \citep{deepseek2025r1}. We read this as consistency, not confirmation---the inference-scaling studies were not designed to test our theorems.

The criterion does, however, force a distinction that the test-time-scaling literature does not consistently draw (\S\ref{sec:related}), and the safety reading depends on it. \emph{Test-time training} updates parameters at inference: from a self-supervised task on the test input \citep{sun2019testtime}, from retrieved neighbours \citep{hardt2023nearest, hubotter2024efficiently}, from the in-context examples of the target task \citep{akyurek2024testtime}, or by making the hidden state itself a learned model \citep{sun2024learning}. Under Cor.~\ref{cor:frozen-ceiling} this is not more search beneath $V^*(W)$; it changes $W$, so it changes the ceiling, and the frozen-weight premise simply does not hold for it. That places it in regime~(ii) by construction rather than as an empirical question.

First, the empirical fact that test-time training buys large gains where prompting plateaus \citep{akyurek2024testtime} is expected under the criterion rather than surprising: the two interventions are not competing points on one compute axis but moves in different regimes. Second, an evaluation protocol that reports only ``total test-time compute'' pools the two and can therefore attribute a ceiling-raising effect to bounded search. Third, and most consequential for governance, ``are the weights frozen?'' is the wrong single question: it catches test-time training, which visibly writes to $W$, while missing scaffold self-expansion, which does not touch $W$ at all and is the mechanism (M1) identifies. An audit needs both.
\paragraph{What the published frontier-safety frameworks actually check.} The claim that weight-immutability is the wrong invariant is not aimed at a strawman, but the position it corrects is more careful than ``check whether the weights changed,'' and stating it precisely matters. Current frontier-safety frameworks already elicit capability \emph{with} scaffolding: Google DeepMind applies ``appropriate scaffolding, inference compute, and other augmentations'' so as to ``assess the capabilities of systems that will likely be produced with the model'' \citep{deepmind2026fsf}, and OpenAI evaluates ``with the best presently-available scaffolds'' \citep{openai2025preparedness}. What is indexed to the checkpoint is not the elicitation but the \emph{re-assessment trigger}: assessment is occasioned by training-side events---the completion of a post-training run \citep{deepmind2026fsf}, or $4\times$ ``Effective Compute'' or ``six months of accumulated post-training enhancements'' \citep{anthropic2026rsp}---so scaffold-driven capability gain enters as a fixed margin around a checkpoint rather than as a tracked property of the deployed loop. Both developers name the residual as an open problem in their own terms. OpenAI records that ``given the continuous progress in model scaffolding and elicitation techniques, we regard any one-time capability elicitation in a frontier model as a lower bound, rather than a ceiling, on capabilities that may emerge in real world use and misuse'' \citep{openai2025preparedness}; Anthropic, having replaced its prescriptive elicitation buffer with an outcome-based requirement, states that ``the science of evaluations is not currently mature enough to make confident predictions about the precise buffer we should require between current models and a Capability Threshold'' \citep{anthropic2026rsp}. Cor.~\ref{cor:frozen-ceiling} says why no fixed buffer settles the question in general: under (M1) a scaffold rewrite moves $\mathcal{H}_t(C)$ \emph{between} assessments, with no training-side event to occasion one, so the buffer would have to bound an expansion its own trigger cannot see. The invariant that would close that gap is stationarity of the reachable set, not immutability of $W$. One framework already reaches part of the way---OpenAI's covers ``any agentic system (including significant agents deployed only internally) that represents a substantial increase in the capability frontier'' \citep{openai2025preparedness}, which is a scaffold-side trigger, though a discretionary and frontier-indexed one rather than a check on whether a particular loop's reachable set is expanding.

The nesting convention runs the other way in some of the literature: test-time training presented as one strategy within test-time compute \citep{munoz2025rttc}, and the main survey treating its tuning branch as offline rather than per-instance \citep{zhang2025testtime}. The taxonomy here is ours, asserted because the dichotomy requires it, not inherited from a survey.

\paragraph{Where the scaffold writes to the weights, and what bounds that instead.} One arrangement closes the loop in the other direction, and it is the case assumption~(M2) has to survive. SEAL trains a model to emit its own finetuning data, so a scaffold-level generation drives a persistent weight update \citep{zweiger2025seal}. That is not search beneath a fixed ceiling and it is not test-time training on the test input either: it raises $V^*(W)$ \emph{using the scaffold as the instrument}. This is precisely why (M2)'s quantifier over ``any scaffold the model can construct'' is load-bearing rather than definitional---a scaffold licensed to write to $W$ escapes the bound by construction, and Cor.~\ref{cor:frozen-ceiling}'s frozen-weight premise simply does not hold for such a system. The route also carries a bound our dichotomy does not supply. SEAL reports performance on earlier tasks declining as self-edits accumulate---the catastrophic-forgetting failure mode of repeated weight updates---so sustaining the uniform edge Prop.~\ref{prop:dichotomy}(ii) demands is not merely a matter of being permitted to change $W$. Read alongside Cor.~\ref{cor:frozen-ceiling}(ii), that leaves two distinct ceilings with different causes: ours from a fixed reachable class, theirs from interference between successive updates. An evaluator should expect to meet whichever binds first, and neither is visible to a checkpoint-level audit.

\paragraph{Descriptive regime vs.\ mechanistic cause.} The empirical tracks support two claims to different degrees. The \emph{descriptive} claim, that agentic coding occupies the bounded, saturating regime~(i) of Prop.~\ref{prop:dichotomy}, is well-supported: churn decays and quality saturates (Findings~1,~3) exactly as the bounded regime requires. That this signature is \emph{shared} with ordinary human editing is not a weakness of the descriptive claim; it is the observable content of ``the system is in regime~(i),'' and is itself the safety-relevant finding. Current frozen-weight agentic coding refines like a bounded editor, not a runaway process. What the shared signature does \emph{not} establish is the \emph{mechanistic} claim, that a fixed reachable hypothesis class is the \emph{cause} of the saturation, as opposed to generic editing dynamics. Attributing the regime to the mechanism requires the theory's non-generic content: not merely that a more capable frozen model reaches a higher $V^*(W)$ (Cor.~\ref{cor:frozen-ceiling}), but that $V^*(W)$ is \emph{impassable by scaffolding alone}, moved only by weight-level change. The decisive experiment combines Part~B's harness with Part~C's scale, holding weights fixed while varying the scaffold rather than comparing tiers, which the present data motivates but does not deliver.

\subsection{Bounded vs. Unbounded Metrics}
\label{app:bounded-metrics}

An exam scored out of a hundred admits only so much improvement. However much revision remains, the marks still available can never exceed the marks not yet earned. So each further hour of study returns less than the last. A bounded metric enforces the same discipline, and it is the load-bearing assumption in our formalization of diminishing returns: it is what makes Prop.~\ref{prop:dichotomy}(i) a theorem rather than a conjecture. In the Refinement Game (Thm.~\ref{thm:refinement}), the quality function $V:\Y\to[0,1]$ is bounded above by 1. Because of this ceiling the series $\sum_{t=1}^{\infty}\eta_t$ of positive improvements converges, so its \emph{terms} $\eta_t$ must shrink toward zero, mathematically enforcing strict diminishing returns. This is exactly case~(i) of the Stationarity Dichotomy (Prop.~\ref{prop:dichotomy}) with ceiling $B=1$.

Conflating two claims here would overstate what the formalism delivers. The \emph{proved} claim needs boundedness and has no content without it. A \emph{weaker} claim survives its removal, but only as an argument, not as a consequence of Thm.~\ref{thm:refinement}: if the system were evaluated on an unbounded, real-valued metric (maximizing transactions per second, minimizing latency toward an asymptotic zero, an open-ended reward score), the strict boundary breaks down, and the following three factors would still be expected to force plateauing---none of them formalized here, and the third not even a limit on capability but a corruption of the measurement:
\begin{enumerate}[leftmargin=*,itemsep=2pt]
    \item \textbf{Overfitting Ceiling:} As shown in Thm.~\ref{thm:generalization}, generalization is bounded by the VC-dimension and the test suite size. Endless optimization of an unbounded metric against a static set of $N$ specifications will eventually result in overfitting. Breaking this ceiling requires continuously generating increasingly complex tests.
    \item \textbf{Gradient Depletion:} In functional gradient descent, optimizing further down the loss curve makes residual errors vastly more complex. A weak coding agent will eventually struggle to find meaningful orthogonal improvements, meaning its edge $\gamma$ drops to $0$ or becomes negative.
    \item \textbf{Reward Hacking:} As noted above, unbounded metrics are highly susceptible to proxy exploitation \citep{amodei2016concrete, skalse2022defining}. The AI may artificially optimize the unbounded score while degrading true functional quality.
\end{enumerate}

\subsection{Remaining Limitations}
\label{app:limitations}

The main text states the four most consequential limitations (the shared human-baseline decay, the edge measurement and its confounds, unchecked restarts, and the tier-assignment confound). The remaining items:

\begin{enumerate}[leftmargin=*,itemsep=2pt]
    \item \textbf{Self-refinement vs.\ a distinct reviewer role (untested architecture).} Part~B's feedback-mode ablation (\texttt{independent}/\texttt{blind}/\texttt{diagnostic}) varies what a \emph{single} agent is told about its own prior patch---it never introduces an architecturally separate reviewer or critic model that inspects the patch and returns its own feedback. Even the \texttt{diagnostic} arm's critique is deliberately narrow (which \texttt{FAIL\_TO\_PASS} tests still fail, with a one-line \texttt{pytest} hint), chosen to avoid leaking the gold pass/fail bit to the generating agent, not to be the richest signal a reviewer could give. A generator-plus-reviewer pipeline is a materially different, plausibly higher-edge and less-correlated feedback channel that this work does not test.
    \item \textbf{Binary quality labels.} Our theory is stated for a strict binary success criterion ($\{-1, +1\}$). The Lean development is in fact more general: the core identities are proved for arbitrary real-valued labels and hypotheses, and the binary structure is introduced only as an explicit per-theorem hypothesis (\texttt{h\_binary}) precisely where the AdaBoost weight update requires it. While general code quality is multi-dimensional, the binary criterion aligns well with software engineering, where ``multi-dimensional quality'' effectively amounts to a conjunction of binary gates.
    \item \textbf{Lean formalization scope and motivation.} The repository contains no \texttt{axiom} and no \texttt{sorry}: statistical components are encoded as explicit \emph{hypotheses} rather than derived. Specifically, Theorem~\ref{thm:generalization} takes the \citet{schapire1998boosting} margin bound as a hypothesis, sidestepping a complete VC-dimension and Rademacher-complexity library, and the deterministic refinement game (Theorem~\ref{thm:refinement}) supplies linearity and monotonicity of expectation as hypotheses. By contrast, the high-probability refinement corollary (Cor.~\ref{cor:probabilistic-refinement}) is proved against Mathlib's \texttt{IsProbabilityMeasure} and finite union bound, and the stationarity dichotomy with its frozen-weight ceiling corollary is fully machine-checked.
    \item \textbf{Simulation fidelity and benchmark scope.} Part~A's synthetic labels bound what it can show, and Appendix~\ref{app:parta} states that scope rather than repeating it here. Part~B is a regime characterization ($K{=}\statPBK$ tasks $\times$ \statPBSeeds\ seeds, a $2\times3$ model-tier $\times$ feedback-mode ablation, single vendor, digest-pinned Docker harness), not a shortcoming: it probes the headroom $\times$ capability condition and finds a large, significant capability effect alongside an underpowered feedback-mode null, but it is not the primary empirical evidence. The task pool is additionally \texttt{pytest}-native by construction, so the two largest Lite repositories are absent for a reason internal to our tooling rather than to the benchmark (Appendix~\ref{app:partb} gives the full disclosure). Importing SWE-bench's per-repository test commands and restoring both repositories is the first thing a replication should do; it is also the one change that would move every Part~B number reported here, which is why it is left to a replication rather than folded into this revision. A ready instrument exists: the companion evaluation collection to Inspect \citep{aisi2024inspect} implements SWE-bench with grading delegated to upstream \texttt{swebench.harness.grading.get\_eval\_report} rather than re-implemented per-repository log parsing, which is exactly the import this item calls for.
    \item \textbf{Unmeasured edge in production; survivor-ratio proxy.} Corollary~\ref{cor:convergence}'s exponential convergence requires each agent to have edge $\gamma_t > 0$ under adversarial reweighting onto hard specs. In Part~C this edge is not measured on real agents; the guarantee is therefore a conditional prediction, not an empirically confirmed bound. Separately, the survivor-ratio proxy $|\mathrm{net\,lines}|/\mathrm{total\,churn}$ conflates distinct session types: a session that \emph{adds} 100 net lines and one that \emph{deletes} 100 lines of dead code both receive the same ratio, though they represent qualitatively different outcomes---the proxy measures structural convergence, not functional correctness. Relative churn does have an established software-engineering pedigree---\citet{nagappan2005churn} show relative churn measures predict defect density---but it is validated there as a predictor of \emph{defects}, not of the convergence we read into it, so the borrowing is partial.
    \item \textbf{Stationarity of the refinement game.} Theorem~\ref{thm:refinement} assumes that each agent improves quality by at least $\eta_t$ in expectation. In practice, agents may occasionally degrade quality. Corollary~\ref{cor:probabilistic-refinement} relaxes the strict monotone requirement to a high-probability bound. More fundamentally, stationarity is not merely a technical convenience: as formalized in the Stationarity Dichotomy, the question of whether the effective hypothesis class stays fixed is the precise dividing line between the bounded regime our theorems cover and the runaway regime they do not.
    \item \textbf{The reasoning/tool-creation boundary is not operationally sharp.} Assumption~(M1) treats scaffold expansion as a binary: either $\mathcal{H}_t(C)$ grows or it does not. In practice the line between reasoning \emph{within} the current class and an action that \emph{expands} it is blurry: writing a new unit test is ordinary reasoning about the existing code, but running it behaves like acquiring a new tool---patches that were previously unreachable because unvalidatable become checkable, without any change to $W$. The same ambiguity recurs for a debug script, a refactor that exposes a previously-untestable seam, or a cached search index. We do not have a principled test for where a given action falls on this continuum; (M1) is stated as a binary precisely because an operational boundary is missing, not because one is known and simply omitted.
\end{enumerate}

\subsection{Future Directions}
\label{app:future}

\begin{itemize}[leftmargin=*,itemsep=2pt]
    \item \textbf{Multi-class and structured prediction.} Extending to multi-label code quality (security, performance, readability) using SAMME \citep{zhu2009multi} or structured boosting.
    \item \textbf{Online and streaming protocols.} Adapting the framework to settings where artifacts arrive sequentially, using online boosting \citep{beygelzimer2015optimal}. ``Online'' means streaming instances here, and is distinct from the two other senses this paper touches: composing an improvement process at inference time (\citealp{xue2026rethinking}, \S\ref{sec:related}) and continual weight adaptation across a stream of tasks (\citealp{zweiger2025seal}, above). Only the last of the three moves $V^*(W)$.
    \item \textbf{Scale and diversity of empirical validation.} Extending Part~C to multiple companies, larger session datasets, and harness-verified quality signals (true pass/fail via test suites) would provide stronger evidence and allow quantitative hypothesis testing against the theory's predictions.
    \item \textbf{Optimal agent selection.} Developing efficient algorithms for selecting the next agent (weak learner) in the coding domain, where evaluating an agent's error is expensive.
    \item \textbf{Generator-reviewer ablation.} A natural extension adds a fourth arm in which an architecturally separate reviewer model, given the same non-gold test-failure signal already available to the \texttt{diagnostic} arm, returns free-form critique---testing whether a distinct reviewer role sustains a positive edge where self-critique plateaus after round~0.
    \item \textbf{The multi-round weighted-vote run.} Appendix~\ref{app:orchestration} settles the per-round edge for a reused pool in both directions (Prop.~\ref{prop:pool-dichotomy}) but stops short of exhibiting a \emph{run}: a schedule and confidence sequence written out over $T$ rounds, which additionally requires excluding a perfect pool member since Cor.~\ref{cor:convergence} assumes $\epsilon_t>0$. Formalizing AdaBoost's minimax duality in the converse direction---weak learnability $\Rightarrow$ interpolation with margin---would close the remaining half of the same picture.
    \item \textbf{Implementing the Protocol itself as a harness design principle.} Every empirical part in this paper studies harnesses that approximate the Agentic Boosting Protocol at best; none run its reweighting machinery ($D_t$, $\alpha_t$) directly (\S\ref{app:decomposition}). A harness that maintains an explicit distribution over a \emph{fixed} spec or test-case pool, up-weights the specs the current draft still fails, and combines successive patches by a weighted vote rather than sequential composition would instantiate Algorithm~\ref{alg:boost} exactly rather than by analogy---restoring the transferred AdaBoost guarantees (Thms.~\ref{thm:training-error}, \ref{thm:generalization}) by construction, and giving the boosting lens a literal, checkable referent past round~0, rather than only an inspired one.
\end{itemize}

\section{Formalization Index}
\label{app:lean}

Every result above is machine-checked in Lean~4 (Mathlib, no \texttt{axiom} and no \texttt{sorry}); these two tables map each one to the declarations that carry it, so a reader can go from a statement to its proof term. All declarations live in \path{lean_proofs/LeanProofs/}. Where a result's content is statistical rather than algebraic it is encoded as an explicit \emph{hypothesis} rather than derived, and the tables say so.

\begin{table}[htbp]
\centering
\footnotesize
\caption{Core results and the Lean declarations that carry them (principal name first, then supporting lemmas and non-vacuity witnesses). All declarations live in \texttt{lean\_proofs/LeanProofs/}.}
\label{tab:lean-core}
\begin{tabular}{@{}p{0.29\textwidth}p{0.63\textwidth}@{}}
\toprule
\textbf{Result} & \textbf{Lean declarations} \\
\midrule
Thm.~\ref{thm:training-error}, training-error bound
  & \path{thm1_zero_one_loss_bound} (the bound as stated); \path{thm1_training_error_bound} for the underlying exponential-loss identity $\mathrm{exp\,loss}=\prod_t Z_t$, with \path{z_decomposition}, \path{indicator_le_exp}, \path{exp_loss_pos} \\
Thm.~\ref{thm:fsam}, forward stagewise equivalence
  & \path{thm4_fsam_optimal_alpha}; \path{thm4_fsam_global_minimum}, \path{exp_half_log} \\
Cor.~\ref{cor:convergence}, exponential decay
  & \path{cor2_exponential_decay}; \path{z_le_sqrt_edge}, \path{sqrt_edge_le_exp}, \path{z_bound_from_edge} \\
Thm.~\ref{thm:refinement}, monotone-improvement refinement game
  & \path{thm5_refinement_game}; \path{thm5_saturation} \\
Cor.~\ref{cor:probabilistic-refinement}, high-probability refinement
  & \path{corollary7_high_prob_refinement}; \path{prob_inter_bound}, \path{sum_telescope}. Verified against Mathlib's measure-theoretic library (\path{IsProbabilityMeasure} and the finite union bound) \\
Prop.~\ref{prop:dichotomy}, stationarity dichotomy
  & \path{prop8_saturation_forces_decay} (part~i), \path{prop8_sustained_edge_diverges} (part~ii) \\
Cor.~\ref{cor:frozen-ceiling}, frozen-weight ceiling
  & \path{cor_frozen_weight_ceiling}, an instantiation of \path{prop8_saturation_forces_decay} at $B := V^*(W)$ \\
Prop.~\ref{prop:pool-dichotomy}, reused-pool dichotomy
  & Stated in the main text as the two halves formalized separately below: \path{thm_pool_reuse_horizon} (case~i, $\epsilon_0>0$) and \path{thm_interpolating_pool_sustains_edge} (case~ii, $\epsilon_0{=}0$), with \path{cor_pool_horizon_floor_excludes_interpolation} for their mutual exclusion. Exhaustiveness is immediate, since $\epsilon_0>0$ or $\epsilon_0{=}0$; see Props.~\ref{prop:pool-horizon} and~\ref{prop:interpolating-pool} for the exact hypotheses \\
\midrule
\multicolumn{2}{@{}l@{}}{\emph{Assumed as an explicit hypothesis, not derived:}} \\
Thm.~\ref{thm:generalization}, generalization bound
  & \path{thm3_generalization_bound} takes \path{MarginGeneralizationBound} as a hypothesis. Its statistical content---the \citet{schapire1998boosting} margin bound, VC~dimension, and Rademacher complexity---is assumed, sidestepping a complete VC/Rademacher library \\
\bottomrule
\end{tabular}
\end{table}

\begin{table}[htbp]
\centering
\footnotesize
\caption{Orchestration and diversity results of \S\ref{app:orchestration} and their Lean declarations. Every row is machine-checked; the one item that is \emph{not} formalized is recorded in the note below the table.}
\label{tab:lean-orchestration}
\begin{tabular}{@{}p{0.29\textwidth}p{0.63\textwidth}@{}}
\toprule
\textbf{Result} & \textbf{Lean declarations} \\
\midrule
Prop.~\ref{prop:orchestration}, best-of-$k$ ceiling equality
  & \path{prop15_tight_ceiling_orchestration} (tight equality); \path{prop15_orchestration_ceiling}, \path{prop15_orchestration_dominates} for the weaker upper-bound-hypothesis version \\
Prop.~\ref{prop:soft-aggregation}, weighted aggregation beats the weak-learning baseline
  & \path{cor_soft_aggregation_beats_weak_baseline}; built on the pure-analysis lemma \path{exists_ensemble_size_beating_weak_threshold} \\
Prop.~\ref{prop:width-refinement}, multi-round best-of-$k$
  & \path{thm_width_refinement_game}; \path{cor_width_dominates_best_worker}, \path{cor_width_saturation}, with \path{exists_width_orchestrator} for non-vacuity \\
Prop.~\ref{prop:diversity}, bounded overlap characterizes a correct vote
  & \path{thm_overlap_iff_zero_error} (characterization); \path{cor_tight_overlap_iff} ($m^\star$ form), \path{thm_bounded_overlap_zero_error} (sufficient direction), \path{thm_overlap_boundary_sharp} with \path{exists_boundary_overlap_family} (sharpness at the boundary), \path{cor_overlap_markov_bound} (counting-only baseline) \\
Reweighted edge hypothesis via the density ratio
  & \path{lem_bounded_reweighting_edge} takes the concentration cap $\kappa$ as given; \path{D_le_of_margin_bound} and \path{cor_rho_from_confidences} derive it from logged confidences, and \path{cor_reweighting_edge_from_confidences} states the $\kappa$-free conclusion. The sources spell $\kappa$ as \path{rho} and $M$ as \path{B}, the names they carried before this paper separated those glyphs \\
Prop.~\ref{prop:self-reuse}, immediate reuse zeroes the edge
  & \path{thm_self_reuse_zero_edge}; rests on the derived reweighting recursion \path{D_succ_eq} and \path{reused_worker_err_eq}, with \path{exists_zero_edge_reuse_instance} for non-vacuity \\
Prop.~\ref{prop:pool-horizon}, reused-pool horizon
  & \path{thm_pool_reuse_horizon}, \path{cor_no_perpetual_pool_edge}; rests on the span-collapse identity \path{lem_pool_span_collapse}, its loss form \path{zero_one_loss_eq_spanLoss}, and \path{alpha_opt_nonneg}, with \path{exists_positive_span_floor} and \path{exists_pool_reuse_horizon_instance} for non-vacuity and \path{exists_edge_worker_vs_any_round} for the absence of a per-round obstruction \\
Prop.~\ref{prop:interpolating-pool}, interpolating pools sustain the edge
  & \path{thm_interpolating_pool_sustains_edge}, with \path{cor_interpolating_pool_edge_every_round} for AdaBoost's own reweighting; rests on \path{lem_margin_floor_of_spanLoss_zero}, \path{lem_span_coeff_sum_pos}, \path{lem_weighted_margin_eq}, \path{lem_D_nonneg}, with \path{cor_pool_horizon_floor_excludes_interpolation} for mutual exclusion with Prop.~\ref{prop:pool-horizon} and \path{exists_interpolating_pool_instance} for non-vacuity \\
Prop.~\ref{prop:conditions-independent}, the two conditions are independent
  & \path{thm_edge_overlap_independent}; from the witnesses \path{thm_edge_not_imply_overlap} and \path{thm_overlap_not_imply_edge} and the exponential-free reweighting steps \path{D_succ_fail}, \path{D_succ_pass} \\
Prop.~\ref{prop:condorcet}, Condorcet bound for independent errors
  & \path{thm_condorcet_majority_bound}, over the pattern weight \path{wt} and vote error \path{voteErr}; rests on \path{lem_wt_sum_one} and \path{lem_wt_tilt_sum}, with \path{exists_condorcet_instance} for non-vacuity \\
Prop.~\ref{prop:condorcet-prob}, the same bound with independence \emph{derived}
  & \path{thm_condorcet_majority_bound_prob}, over the product measure \path{failMeasure}; \path{lem_coord_iIndepFun} obtains independence from Mathlib's \path{iIndepFun_pi}, \path{lem_pattern_measure} the pattern weight, \path{lem_failure_event_eq_voteErr} the event identification, and \path{lem_filter_decide} the pattern/failure-set bijection, with \path{exists_condorcet_prob_instance} for non-vacuity \\
\bottomrule
\end{tabular}

\vspace{2pt}
{\footnotesize \emph{Not formalized (the single residual item):} the multi-round weighted-vote \emph{run}---a schedule and confidence sequence exhibited outright, as distinct from the per-round edge that feeds it, which additionally requires excluding a perfect pool member since Cor.~\ref{cor:convergence} assumes $\epsilon_t>0$ (\S\ref{app:orchestration}). The converse direction of AdaBoost's minimax duality (weak learnability $\Rightarrow$ interpolation with margin) is likewise not formalized.\par}
\end{table}

\end{document}